\documentclass[nohyperref]{article}

\usepackage{microtype}
\usepackage{graphicx}
\usepackage{subfigure}
\usepackage{booktabs} %

\usepackage{hyperref}

\usepackage[accepted]{icml2022}

\usepackage{amsmath}
\usepackage{amssymb}
\usepackage{mathtools}
\usepackage{amsthm}

\usepackage{bm}
\usepackage{tikz}
\usepackage{booktabs}
\usepackage{multirow}
\usetikzlibrary{fit,arrows.meta,external}
\usepackage{xcolor}
\newcommand{\x}{\mathbf{x}}
\newcommand{\s}{\mathbf{s}}

\newcommand{\z}{\mathbf{z}}

\newcommand{\E}{\mathbb{E}}
\def\indep{\perp\!\!\!\perp}
\usepackage{pdfpages}

\usepackage{amsthm}
\usepackage{thmtools}
\usepackage{thm-restate}
\theoremstyle{plain}
\newtheorem{theorem}{Theorem}
\newtheorem{proposition}[theorem]{Proposition}

\newtheorem{assumption}[theorem]{Assumption}
\theoremstyle{remark}

\DeclareMathOperator*{\KL}{\mathrm{KL}}
\newcommand{\kbet}{\text{kBET}}
\newcommand{\mkbet}{\text{m-kBET}}

\newcommand{\tcgaccle}{Tumour / Cell Line}
\mathchardef\mhyphen="2D

\hypersetup{pdfauthor={Adam Foster, Árpi Vezér, Craig A Glastonbury, Páidí Creed, Sam Abujudeh, Aaron Sim},
pdftitle={Contrastive Mixture of Posteriors for Counterfactual Inference, Data Integration and Fairness}}

\icmltitlerunning{Contrastive Mixture of Posteriors for Counterfactual Inference, Data Integration and Fairness}

\begin{document}

\twocolumn[
\icmltitle{Contrastive Mixture of Posteriors for Counterfactual Inference, Data Integration and Fairness}

\icmlsetsymbol{equal}{*}

\begin{icmlauthorlist}
\icmlauthor{Adam Foster}{oxford}
\icmlauthor{Árpi Vezér}{bai}
\icmlauthor{Craig A. Glastonbury}{bai,techno}
\icmlauthor{Páidí Creed}{bai}
\icmlauthor{Sam Abujudeh}{bai}
\icmlauthor{Aaron Sim}{bai}
\end{icmlauthorlist}

\icmlaffiliation{oxford}{Microsoft Research, Cambridge, UK. This work was completed when AF was an intern at BenevolentAI and a DPhil
student at the Department of Statistics, University of Oxford.
}
\icmlaffiliation{bai}{BenevolentAI, 4--8 Maple Street, London, UK}
\icmlaffiliation{techno}{Human Technopole, V.le Rita Levi-Montalcini, 1, 20157 Milano, Italy}

\icmlcorrespondingauthor{Adam Foster}{adam.e.foster@microsoft.com}

\icmlkeywords{Machine Learning, ICML, Counterfactual, Data integration, Fairness, VAE, CVAE, Contrastive, Cell alignment}

\vskip 0.3in
]

\printAffiliationsAndNotice{} %

\begin{abstract}
Learning meaningful representations of data that can address challenges such as batch effect correction and counterfactual inference is a central problem in many domains including computational biology. Adopting a Conditional VAE framework, we show that marginal independence between the representation and a condition variable plays a key role in both of these challenges. We propose the Contrastive Mixture of Posteriors (CoMP) method that uses a novel misalignment penalty defined in terms of mixtures of the variational posteriors to enforce this independence in latent space. We show that CoMP has attractive theoretical properties compared to previous approaches, and we prove counterfactual identifiability of CoMP under additional assumptions. We demonstrate state-of-the-art performance on a set of challenging tasks including aligning human tumour samples with cancer cell-lines, predicting transcriptome-level perturbation responses, and batch correction on single-cell RNA sequencing data. We also find parallels to fair representation learning and demonstrate that CoMP is competitive on a common task in the field.
\end{abstract}
\section{Introduction}
Large scale datasets describing the molecular properties of cells, tissues and organs in a state of health and disease are commonplace in computational biology. Referred to collectively as `omics data, thousands of features are measured per sample and, as single-cell methodologies have developed, it is now typical to measure such features across $10^5$--$10^6$ observations \cite{svensson2018exponential, regev2017science}.
Given these two properties of `omics data, the need for scalable algorithms to learn meaningful low-dimensional representations that capture the variability of the data has grown.
As such, Variational Autoencoders (VAEs)~\cite{kingma2013auto,Rezende2014StochasticBA} have become an important tool for solving a range of modelling problems in the biological sciences \cite{lopez2018deep, way2018extracting, wang2018vasc, gronbech2020scvae, lotfollahi2019conditional, lotfollahi2019scgen}.
One such problem is utilising representations for counterfactual inference, e.g.~predicting how a certain cell or cell-type, observed in the control group, would have behaved when exposed to a drug \citep{lotfollahi2019conditional, lotfollahi2019scgen, Amodio2018OutofSampleEW}.
Another key problem is removing batch effects---spurious shifts in observations due to differing experimental conditions---from data in order to integrate or compare multiple datasets \cite{lopez2018deep, Johnson2007AdjustingBE, Leek2007CapturingHI, Haghverdi2018BatchEI, Celligner}.

Our approach to these problems is to learn a VAE representation that is marginally independent of a condition variable (e.g.~experimental batch, stimulated vs. control).
Figure~\ref{fig:kang_umap_intro} [CoMP] illustrates what this looks like in practice: the complete overlap of the cell populations from different conditions in the latent space.
For data integration, the resulting representation perfectly integrates distinct batches, assuming there are no population-level differences between them.
To predict the effects of interventions, following \citet{lotfollahi2019conditional}, we encode control data to representation space, and decode it back to the original space under the stimulated condition.
Alignment of control and stimulated cells in representation space isolates the effects of interventions to the decoder network, and is a necessary condition for the encode--swap--decode algorithm to provide correct predictions.
This same independence constraint also occurs in fair representation learning, where we seek a representation that cannot be used to recover a sensitive attribute \citep{zemel2013learning,louizos2015variational}.

Neither the standard VAE nor the conditional VAE (CVAE)~\cite{sohn2015learning} are typically successful at learning representations that achieve this desired independence, as shown in Figure~\ref{fig:kang_umap_intro}.
Existing methods use a penalty to encourage the CVAE to learn representations that overlap correctly in latent space, with Maximum Mean Discrepancy (MMD)~\cite{gretton2012kernel} being the most common, applied in the VFAE~\cite{louizos2015variational} and trVAE \cite{lotfollahi2019conditional}.
These methods, however, suffer from a number of drawbacks: conceptually, they introduce an extraneous discrepancy measure that is not a part of the variational inference framework; practically, they require the choice of, and  hyperparameter tuning for, an MMD kernel; empirically, whilst trVAE is a significant improvement over an unconstrained CVAE, Figure~\ref{fig:kang_umap_intro} [trVAE] shows that it may fail to exactly align different conditions.

To overcome these difficulties, we introduce \emph{Contrastive Mixture of Posteriors (CoMP)}, a new method for learning aligned representations in a CVAE framework.
Our method features the novel CoMP misalignment penalty that compels the CVAE to remove batch effects.
Inspired by contrastive learning \cite{oord2018representation,chen2020simple}, the penalty encourages representations from different conditions to be close, whilst representations from the same condition are spread out.
To achieve this, we approximate the requisite marginal distributions using mixtures of the variational posteriors themselves, leading to a penalty that does not require an extraneous discrepancy measure or a separately tuned kernel.
We prove that the CoMP penalty is a stochastic upper bound on a weighted sum of KL divergences, so minimising the penalty minimises a well-established statistical divergence measure.
As shown in Figure~\ref{fig:kang_umap_intro} [CoMP], our method can achieve visually perfect alignment on a number of real-world biological datasets.

Theoretically, counterfactual inference provides the formal framework to discuss data integration \citep{bareinboim2016causal}, perturbation response prediction, and fairness \citep{kusner2017counterfactual}.
We demonstrate that the constrained CVAE approach is \emph{not} always able to compute counterfactuals, even with infinite data. However, introducing additional assumptions, including non-Gaussianity of the latent distribution, we prove counterfactual identifiability and model consistency in our framework.
This begins to provide theoretical grounding, not only for CoMP, but for related methods \citep{louizos2015variational,lotfollahi2019conditional}.

We apply CoMP to three challenging biological problems\footnote{Source code for the experiments is provided at \url{https://github.com/BenevolentAI/CoMP}.}: 1) aligning gene expression profiles between tumours and their corresponding cell-lines \cite{Celligner}, 2) estimating the gene expression profile of an unperturbed cell as if it \textit{had} been treated with a chemical perturbation \cite{lotfollahi2019scgen}, 3) data integration with single-cell RNA-seq \citep{Harmony}. We show that CoMP outperforms existing methods, achieving state-of-the-art performance on these tasks.
We also show that CoMP can learn a representation that is fully independent of a protected attribute (gender) whilst maintaining useful information for other prediction tasks on the UCI Adult Income dataset \citep{Dua2019UCIIncome}. CoMP represents a conceptually simple and empirically powerful method for learning aligned representation, opening the door to answering high-value questions in biology and beyond.

\begin{figure}[t]
  \centering
  \includegraphics[width=0.99\columnwidth]{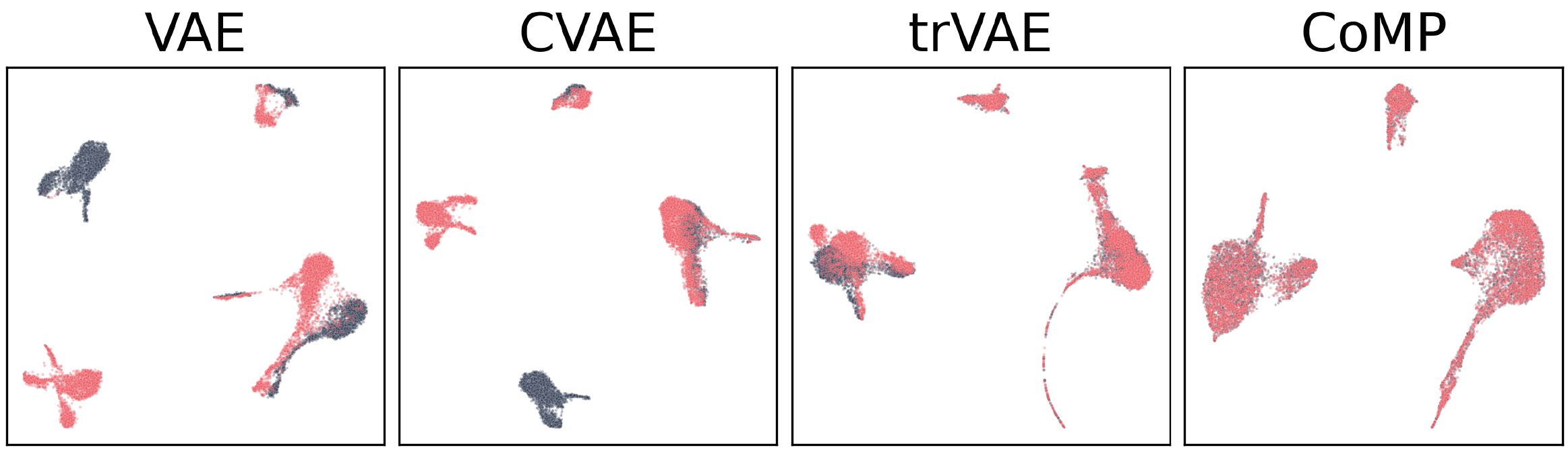}
  \caption{Latent representations of a single-cell gene expression dataset under two conditions: stimulated ({\color{red} red}) and not stimulated (black). Full details in Section \ref{sec:kang}. From fully disjointed (VAE) to a well-mixed pair of distributions (CoMP).}
  \label{fig:kang_umap_intro}
\end{figure}
\section{Background}
\subsection{Variational Autoencoders and extensions}
We begin by assuming that we have $n$ observations $\x_1,\dotsc,\x_n$ of an underlying data distribution.
For example, $\x_i$ may represent the gene expression profiles of $n$ cells.
Variational autoencoders (VAEs) \cite{kingma2013auto,Rezende2014StochasticBA} explain the high-dimensional observations $\x_i$ using low dimensional representations $\z_i$. The standard VAE places a Gaussian prior $\z\sim p(\z)$ on the latent variable, and learns a generative model $p_\theta(\x|\z)$ that reconstructs $\x$ using $\z$, alongside an inference network $q_\phi(\z|\x)$ that encodes $\x$ to $\z$. Both $\theta$ and $\phi$ are trained jointly by maximising the ELBO, a lower bound on marginal likelihood given by $\log p_{\theta}(\x) \ge \E_{q_\phi(\z|\x)}\left[\log p_\theta(\x|\z)\right] - \text{KL}\left[q_\phi(\z|\x) \| p(\z)\right]$. This can be maximised using stochastic optimisers \cite{robbins1951stochastic,kingma2014adam}.

So far, we have assumed that the only data available are the observations $\x_1,\dotsc,\x_n$, but in many practical applications we may have additional information such as a condition label for each observation.
For example, in gene knock-out studies, we have information about which gene was targeted for deletion in each cell; in multi-batch experiments we have information about which experimental batch each samples was collected in.
Thus, we augment our data by considering data pairs $(\x_1,c_1),\dotsc,(\x_n,c_n)$ where $\x$ is a high-dimensional observation, and $c$ is a label indicating the condition or experimental batch that $\x$ was collected under.

Whilst VAEs are theoretically able to model the pairs $(\x_i,c_i)$, it makes sense to build a model that explicitly distinguishes between the $\x$ and $c$. 
The simplest model is the Conditional VAE (CVAE) \cite{sohn2015learning}. In this model, a conditional generative model $p_\theta(\x|\z,c)$ and a conditional inference network $q_\phi(\z|\x,c)$ are trained using a modified ELBO.
A key observation for our work is that the CVAE has many different ways to model the data. For example, it can completely ignore the condition $c$ in $p_\theta$ and $q_\phi$, reducing to the original VAE.
Assuming that $\x$ is not independent of $c$, this failure mode of the CVAE would be apparent on a visualisation of the representations. For example, different values of $c$ might be visible as separate latent clusters, as shown in Figure~\ref{fig:kang_umap_intro} [CVAE].

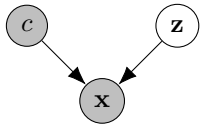
\begin{figure}[t]
	\begin{center}
		\begin{tikzpicture}
		\node (x) [draw, circle, fill=lightgray] at (0, 0) {$\x$};
		\node (y) [draw, circle, fill=lightgray] at (-1, 1) {$c$};
		\node (z) [draw, circle] at (1, 1) {$\z$};
		\draw[-{Latex[length=2.5mm]}] (z) -> (x);
		\draw[-{Latex[length=2.5mm]}] (y) -> (x);
		\end{tikzpicture}
	\end{center}
	\caption{Structural Equation Model for observation $\x$ under known condition $c$ with unobserved latent variable $\z$. In this model, $\z$ and $c$ are independent in the prior.}
	\label{fig:pgm}
\end{figure}

\subsection{Counterfactual inference} \label{sec: cf}
If $(\x_i,c_i)$ represents an RNA transcript and the gene knock-out applied to the cell, a natural question to ask is ``How would the transcript have differed if a different knock-out $c'$ had been applied?''
In general, \emph{counterfactual inference}  is necessary to answer questions of the form ``How would the data have changed if $c_i$ had been replaced by $c'$?''
In this paper, we assume access to unpaired data, meaning that each cell $\x_i$ is observed only in one condition $c_i$.
Answering counterfactual questions with such data is a notoriously difficult task, because they naturally refer to unobservable data \cite{pearl2009causality}.
A principled approach to such questions is to adopt the framework of Structural Equation Models \cite{bollen2005structural,pearl2009causality}. In this paper, we assume that the data generating process is given as in Figure~\ref{fig:pgm}.
If this model is correct, counterfactual inference in the Pearl framework \cite{pearl2009causality} can then be performed by: 1) \emph{abduction}: inferring the latent $\z$ from $\x$ and $c$ using $p(\z|\x,c)$, 2) \emph{action}: swap $c$ for $c'$, 3) \emph{prediction}: use $p(\x|\z,c')$ to obtain a predictive distribution for the counterfactual.
Thus, the counterfactual distribution of $\x_i$ observed with condition $c_i$ but predicted for condition $c'$ is given by
\begin{equation}
\label{eq:counter}
    p\left(\x_{c=c'}|\x_i,c_i\right) = \int p(\z|\x_i,c_i)p(\x|\z,c')\ d\z.
\end{equation}

\section{The constraint $\z \indep c$}
\label{sec:unify}

Our high-level approach is to learn a CVAE model with the constraint that $\z \indep c$ under the encoder distribution $q_\phi$. Visually, this means that latent representations from different conditions are aligned.
To achieve this, we introduce a novel training penalty that penalises misalignment between different conditions (Section~\ref{sec:contrastive}). We first discuss how this general approach applies to data integration, fair representation learning and counterfactual inference.

For data integration where $c_i$ indicates experimental batch, spurious shifts may be present in the distribution of $\x$ due to differing experimental conditions, as opposed to true changes in the underlying biology.
Using our approach, the latent $\z$ can be used in place of $\x$ for downstream tasks, thereby integrating data from different batches. Intuitively, by enforcing $\z \indep c$ we `subtract' batch effects, leaving a representation that has the same marginal distribution between batches.
Alternatively, we can reconstruct the $\x_i$ as if they arose under a single experimental batch, performing batch correction in the original space. By enforcing $\z \indep c$ in our encoder, we are assuming that there are no population-level differences between batches. This assumption bears a close resemblance to the assumptions used (sometimes implicitly) in latent factor \citep{Leek2007CapturingHI,stegle2012using} and CVAE \citep{zuo2021deep,lotfollahi2019conditional} models for data integration. We discuss our assumptions and ways to relax them in Section~\ref{sec:summary}.

In fair representation learning, the notion of building a representation that cannot be used to recover $c$ has been studied widely in recent literature \citep{zemel2013learning,louizos2015variational,kusner2017counterfactual, glastonbury2018adjusting}. 
In particular, if we wish to make a predictive rule based on $\x$ that does not discriminate between individuals in different conditions $c$, we can use a fair representation $\z$, one which contains information from $\x$ but cannot be used to recover $c$, as an intermediate feature and train our model using $\z$. 
Being unable to recover $c$ from $\z$ is equivalent to our constraint $\z \indep c$ (see Appendix~\ref{sec:characterising}). 

For counterfactual inference, we can estimate equation~\eqref{eq:counter} by replacing the true data generating distributions with model-based estimates $q_\phi(\z|\x_i,c_i)$ and $p_\theta(\x|\z,c')$, giving
\begin{equation}
    \label{eq:approx_counter}
    \hat{p}(\x_{c=c'}|\x_i,c_i) = \int q_\phi(\z|\x_i,c_i)p_\theta(\x|\z,c')\ d\z.
\end{equation}
The failure mode in which different values of $c$ form separate latent clusters,  as in Figure~\ref{fig:kang_umap_intro} [CVAE], can be catastrophic for this application, because it violates assumptions of Figure~\ref{fig:pgm}.
However, it is not true that the constraint $\z \indep c$ alone is sufficient to guarantee the correct estimation of counterfactuals using \eqref{eq:approx_counter}. We discuss this is Section~\ref{sec:theory}, and prove that, under additional assumptions, it becomes possible to identify counterfactuals using our CVAE approach with $\z \indep c$.
Counterfactual inference also provides a more rigorous foundation to discuss both fairness \citep{chiappa2019path,kusner2017counterfactual,zhang2016causal,NIPS2017_f5f8590c,NEURIPS2018_ff1418e8} and data integration \citep{bareinboim2016causal}.
As such, these three problems have a deep underlying connection.

\section{Contrastive Mixture of Posteriors}
\label{sec:contrastive}
Our approach to counterfactual inference, data integration and fair representation learning centres on learning a representation such that the latent variable $\z$ is independent of the condition $c$ under the distribution\footnote{We drop the $\theta,\phi$ subscripts on $p_\theta$ and $q_\phi$ in this section for conciseness and legibility.} $q$, so that the latent clusters with different values of $c$ are perfectly aligned.
Building off the CVAE, which rarely achieves this in practice, a number of authors have attempted to use a penalty term to reduce the dependence of $\z$ upon $c$ during training.
The most successful methods, such as trVAE \cite{lotfollahi2019conditional}, are based on MMD \cite{gretton2012kernel}. 
Whilst trVAE and related methods can work well, they require an MMD kernel, not a part of the original model, to be specified and its parameters to be carefully tuned. 
Experimentally, we observe that MMD-based methods can often struggle when there is complex global structure in the latent space.
We also analyse the gradients of MMD penalties, showing that they have some undesirable properties.

We propose a novel method to enforce $\z \indep c$ in a CVAE model.
Our penalty is based on posterior distributions obtained from the model encoder itself. That is, we do not introduce any external discrepancy measure, rather we propose a penalty term that arises naturally from the model itself.
Our penalty enforces the equality of the marginal distribution $q(\z|c)$ and $q(\z|\neg\,c)$ for each $c\in\mathcal{C}$, where $q(\z|c) = \E_{p(\x|c)}\left[ q(\z|\x,c) \right]$ represents the marginal distribution of $\z$ over all points within condition $c$ and
\begin{equation}
    q(\z|\neg\,c) = \frac{\sum_{c'\in\mathcal{C}, c'\ne c}p(c')q(\z|c')}{\sum_{c'\in\mathcal{C},c'\ne c} p(c')}
\end{equation}
represents the marginal distribution of $\z$ over all conditions not equal to $c$.
In Appendix~\ref{sec:characterising}, we show that the statement `$q(\z|c) = q(\z|\neg\,c)$ for each $c$' is equivalent to the statement `$\z \indep c$ under the distribution $p(\x,c)q(\z|\x,c)$'.
Therefore, enforcing $\z \indep c$ is the same as enforcing $q(\z|c)$ and $q(\z | \neg,c)$ to be equal for every $c$.

To encourage greater overlap between $q(\z|c)$ and $q(\z|\neg\,c)$, we can encourage points with the condition $c$ to be in areas of high density under the representation distribution for \emph{other} conditions, i.e.~areas in which $q(\z|\neg\,c)$ is also high. 
To encourage this, we can add the penalty term $\mathcal{P}_0(\z_i,c_i) = -\log q(\z_i|\neg\,c_i)$ to the objective for the data pair $(\x_i,c_i)$.
When we minimise $\mathcal{P}_0$, this brings the representations of samples under condition $c_i$ towards regions of high  density under $q(\z|\neg\,c_i)$.

Since the density $q(\z|\neg\,c)$ is not known in closed form, we approximate $q(\z|\neg\,c)$ using other points in the same training batch as $(\x_i,c_i)$. 
Indeed, suppose we have a batch $(\x_1,c_1),...,(\x_B,c_B)$. We let $I_c$ denote the subset of indices for which $c_j=c$ and $I_{\neg c}$ denote its complement. We use the approximation
\begin{equation}
    \label{eq:batch_mix}
    \log q(\z_i|\neg\,c_i) \approx \log\biggl(\frac{1}{|I_{\neg c_i}|} \sum_{j\in I_{\neg c_i}} q(\z_i|\x_j, c_j) \biggr)
\end{equation}
and we will show in Theorem~\ref{thm:kl} that this approximation in fact leads to a valid stochastic bound.

It may happen that the penalty $\mathcal{P}_0$ causes points to become too tightly clustered. 
Inspired by contrastive learning \citep{oord2018representation}, we include a second term which promotes higher entropy of the marginal, thereby avoiding tight clusters of points. Combined with $\mathcal{P}_0$, this leads us to a second penalty $\mathcal{P}_1(\z_i,c_i) = \log{q(\z_i|c_i)}-\log {q(\z_i|\neg\, c_i)}$.
Again, the density $q(\z|c)$ is not known in closed form, but we can approximate it using points within the same training batch in a similar fashion to \eqref{eq:batch_mix}.
Combining both approximations to estimate $\mathcal{P}_1$ and then taking the mean of the penalty over the batch gives our \emph{Contrastive Mixture of Posteriors (CoMP) misalignment penalty}
\begin{equation}
\begin{split}
    \begin{matrix}
    \text{CoMP} \\
    \text{penalty}
    \end{matrix} =  
    \frac{1}{B}\sum_{i=1}^B \log\biggl(\frac{\frac{1}{|I_{c_i}|}\sum_{j\in I_{c_i}} q(\z_i|\x_j,c_i)}{\frac{1}{|I_{\neg c_i}|} \sum_{j\in I_{\neg c_i}} q(\z_i|\x_j, c_j)} \biggr) %
\end{split}
\end{equation}
where $\x_{1:B},c_{1:B},\z_{1:B} \sim \prod_{i=1}^B p(\x_i,c_i)q(\z_i|\x_i,c_i)$ is a random training batch of size $B$, $I_c$ denotes the subset of $\{1,\dots,B\}$ with condition $c$ and $I_{\neg c}=\{1,\dots,B\} \setminus I_c$.
Our method therefore utilises a training penalty for CVAE-type models that encourages the constraint $\z \indep c$ to hold by using mixtures of the variational posteriors themselves to approximate $q(\z|c)$ and $q(\z|\neg\, c)$.

As hinted at by the definition of $\mathcal{P}_1$, CoMP can be seen as approximating a symmetrised KL-divergence between the distributions $q(\z|c)$ and $q(\z|\neg\,c)$.
In fact, the following theorem shows that the CoMP misalignment penalty is a \emph{stochastic upper bound on a weighted sum of KL-divergences}.
\begin{restatable}{theorem}{kl}
\label{thm:kl}
The CoMP misalignment penalty satisfies
\begin{align*}
\begin{split}
    &\E\Biggl[\frac{1}{B}\sum_{i=1}^B \log\biggl(\frac{\frac{1}{|I_{c_i}|}\sum_{j\in I_{c_i}} q(\z_i|\x_j,c_i)}{\frac{1}{|I_{\neg c_i}|} \sum_{j\in I_{\neg c_i}} q(\z_i|\x_j, c_j)} \biggr) \Biggr] \\ & \qquad \qquad \ge \sum_{c\in\mathcal{C}} p(c) \KL\left[q(\z|c) || q(\z|\neg\,c) \right]
\end{split}
\end{align*}
where the expectation is over $\prod_{i=1}^B p(\x_i,c_i)q(\z_i|\x_i,c_i)$. The bound becomes tight as $B\to\infty$.
\end{restatable}
The proof is given in Appendix~\ref{sec:proofs}.
Our result shows that our new penalty directly reduces the KL divergence between each pair $q(\z|c)$, $q(\z|\neg\,c)$ weighted by $p(c)$.
As with standard contrastive learning, our method benefits from larger batch sizes.
We add the CoMP misalignment penalty to the familiar CVAE objective to give our \emph{complete training objective} for a batch of size $B$ as
\begin{align}\begin{split}
    \label{eq:total_training_obj}
    \mathcal{L}^\text{CoMP}_B&(\theta,\phi) =  \frac{1}{B}\sum_{i=1}^B \Biggl[ \log\left(\frac{p_\theta(\x_i|\z_i,c_i)p(\z_i)}{q_\phi(\z_i|\x_i,c_i)}\right) \\ &-\gamma \log\left(\frac{\frac{1}{|I_{c_i}|}\sum_{j\in I_{c_i}} q_\phi(\z_i|\x_j,c_i)}{\frac{1}{|I_{\neg c_i}|} \sum_{j\in I_{\neg c_i}} q_\phi(\z_i|\x_j, c_j)} \right)\Biggr]
\end{split}\end{align}
with hyperparameter $\gamma$ that controls the strength of the regularisation we apply to enforce the constraint $\z \indep c$.

In Appendix~\ref{sec:app:gradients}, we analyse the training gradient of the CoMP penalty, contrasting it with MMD. We show that, unlike MMD, CoMP gradients have a self-normalising property, allowing one to obtain strong gradients for distant points in a latent space with complex global structure.

\section{Theory}
\label{sec:theory}
\subsection{Counterfactual identifiability in CVAE framework}
When estimating the counterfactual predictions of equation \eqref{eq:counter}, we replace the true data generating distributions with model-based estimates $q_\phi(\z|\x_i,c_i)$ and $p_\theta(\x|\z,c')$, giving equation~\eqref{eq:approx_counter}.
Unfortunately, it is possible for a model to fit the training data arbitrarily well, achieving large $p_\theta(\x|c)$, and yet give incorrect counterfactual predictions (see \citet{pearl2000models}, \citet{bareinboim2020pearl}).

In Proposition~\ref{prop:negative} in Appendix~\ref{sec:app:theory}, we show that this issue is present in the CVAE set-up when the true data generating distribution has $\z \sim N(0, I)$. In this example, the \emph{non-identifiability} arises because we can apply a rotation in the latent space for condition $c=1$, but not for $c=0$, leading to different counterfactual predictions.
At a more fundamental level, for a correctly specified model, symmetries of the true latent space distribution (the existence of a transformation $R$ such that $R\z \overset{d}{=} \z$) make counterfactuals non-identifiable. 

Empirically though, we find that the latent space distribution $q(\z)$ is not, in fact, a Gaussian (Sec.~\ref{sec:experiments}) and has no apparent global symmetries. We also find in these experiments that counterfactual inference is stable between different training seeds, and that it accords extremely well with counterfactuals that are estimated using held-out cell type information.
To explain this phenomenon theoretically, we prove that the non-Gaussianity of the true distribution for $\z$ leads, with additional assumptions, to counterfactual identifiability. 
Our assumptions include explicitly disallowing linear symmetries of the true latent distribution.
See Appendix~\ref{sec:app:theory} for formal specification of our assumptions and the proof.
\begin{restatable}{theorem}{counterfactualidentifiability}
\label{thm:counterfactualidentifiability}
Suppose the true data generating distribution has $\z \sim r(\z)$ and linear decoders for each condition. Assume $r(\z)$ is non-Gaussian and that Assumption~\ref{assume:s} holds. Then counterfactuals are identifiable from unpaired data.
\end{restatable}
The theorem shows that when the true data generating distribution for $\z$ is non-Gaussian but the constraint $\z \indep c$ still holds, then a model which best fits the training data also makes correct counterfactual predictions. In other words, estimating~\eqref{eq:counter} by~\eqref{eq:approx_counter} is valid under these conditions.

\subsection{Consistency of CoMP under prior misspecification}
In the preceding section, we showed that the CVAE framework can identify counterfactuals when the true latent distribution $r(\z)$ is non-Gaussian.
However, our training objective still contains a Gaussian prior $p(\z)$.
Our empirical results (Fig.~\ref{composite}) indicate that this is not problematic, as CoMP is well able to learn non-Gaussian latent distributions.
To ground this in theory, we show that training with the CoMP objective can recover the true model \emph{provided that} the KL-divergence between the true prior and $N(0, I)$ is controlled.
\begin{restatable}{theorem}{consistency}
\label{thm:consistency}
Define $p_{r,\theta}(\x|c) = \E_{r(\z)}[p_\theta(\x|\z,c)]$.
There exists a constant $K_1$ such that, if $\KL[r(\z)\|p(\z)]\le K_1$ and if the encoder network is sufficiently flexible, then maximising the CoMP objective with infinite data generated under the misspecified model with $\z \sim r(\z)$ leads to a $\theta_\infty$ that is a maximum point of $\E_{\x,c}[\log p_{r,\theta}(\x|c)]$.
\end{restatable}

\subsection{Evaluating theoretical assumptions in practice}
\label{sec:theory:practical}
Inspecting the latent distribution $q_\phi(\z)$ provides insights on whether the conditions of the theorems hold in practice---if the latent distribution is different from $N(0,I)$, this hints that $r(\z)$ has the correct form to make counterfactuals identifiable (Theorem~\ref{thm:counterfactualidentifiability}) and that $\KL[r(\z)\|p(\z)]$ is sufficiently small for CoMP to find this non-Gaussian $\z$ through training (Theorem~\ref{thm:consistency}).
More formally, Normality testing \citep{razali2011power} could be employed to check non-Gaussianity.
Secondly, upon retraining the model with different random seeds, the latent distribution $q_\phi(\z)$ should remain the same up to a linear transformation, and counterfactual predictions should remain (approximately) the same. 

\section{Experiments}
\label{sec:experiments}
We perform experiments on four datasets: 1) \tcgaccle: bulk gene expression profiles of tumours and cancer cell-lines across 39 different cancer types \citep{Celligner}; 2) Stimulated / untreated single-cell PBMCs: single-cell gene expression (scRNA-seq) profiles of interferon (IFN)-$\beta$ stimulated and untreated peripheral blood mononuclear cells (PBMCs) \cite{kang2018multiplexed}; 3) Single-cell RNA-seq data integration: scRNA-seq profiles of PBMCs that were processed using different library preparation protocols \cite{Harmony}; 4) UCI Adult Income: personal information of census participants and a binary high / low income label \cite{Dua2019UCIIncome}.

The two broad objectives across our experiments are 1) to demonstrate the extent to which the two random variables $\mathbf{z}$ and $c$ are independent, and 2) to quantify useful information retained in $\mathbf{z}$. To benchmark CoMP on the first objective, we use the following pair of $k$ nearest-neighbour metrics: $\kbet_{k, \alpha}$ \cite{kbet}, the metric used to evaluate batch correction methods in biology, and a local Silhouette Coefficient  \cite{rousseeuw1987silhouettes} $s_{k, c}$. In both cases a low value close to zero indicates good local mixing of sample representations. As for the second objective, if we have access to a held-out discrete label $d_i$ that represents information one wishes to preserve---in the \tcgaccle\   case, $d_i$ is the cancer type, while for the scRNA-seq experiments, it refers to cell type---then we calculate $\kbet$ and $s$ separately for every fixed-$d_i$ subpopulation and take the mean. We refer to these as the \textit{mean Silhouette Coefficient} $\tilde{s}_{k, c}$ and the \textit{mean kBET} metric $\mkbet$ respectively.
These two novel metrics are designed to penalise algorithms that achieve global alignment but mix up different cell types from different conditions, e.g.~by aligning stimulated natural killer cells with unstimulated CD14 cells.
Low values of both metrics indicate correct alignment of each subpopulation.
Full details of the datasets and metrics, along with confidence intervals from multiple experiment runs, are presented in Appendix~\ref{sec:experiment_details}.

\subsection{Alignment of tumour and cell-line samples}
\label{sec:tcga_ccle}

Despite their widespread use in pre-clinical cancer studies, cancer cell-lines are known to have significantly different gene expression profiles compared to their corresponding tumour samples. Here we evaluate the ability of CoMP to subtract out the tumour / cell line condition. This can be seen as both a dataset integration and batch effect correction task. In addition to the set of $k$ nearest neighbour-based mixing evaluations, we train a Random Forest model on the representations of the tumour samples and their cancer-type labels and assess the prediction accuracy on held-out cell lines. To match the results from \citet{Celligner}, the evaluations are performed on the 2D UMAP projections. The results are presented in Table \ref{cancertable}. 

\begin{table}[t]
\small
\centering
\caption{\tcgaccle\ experiment results, with $k=100$, $c=\text{Cell Line}$, and parameter $\alpha=0.01$ for the kBET and m-kBET metrics. $s_{k,c}$ and $\tilde{s}_{k,c}$ are the two Silhouette Coefficient variants (see Section~\ref{sec:experiments}). The top scores are in \textbf{bold}.}
\label{cancertable}
\begin{tabular}{cccccc}
& Accuracy & $s$ & ${\kbet}$ & $\tilde{s}$ & $\mkbet$
\\ \toprule
VAE & 0.209 & 0.658 &  0.974 & 0.803 & 0.581\\ 
CVAE & 0.328 & 0.554 & 0.931 & 0.684 & 0.571 \\ 
VFAE & \textbf{0.585}  & 0.168 & 0.258 & 0.198 & 0.188 \\
trVAE  & \textbf{0.585} & 0.096 &  0.163 & 0.138 & 0.123 \\
Celligner & 0.578 & 0.082 & 0.525 & 0.568 & 0.226 \\ 
\textit{CoMP} & 0.579  & \textbf{0.023} & \textbf{0.160} & \textbf{0.094} & \textbf{0.101} \\
\bottomrule
\end{tabular}
\end{table}

As expected, both the VAE and CVAE baselines fail at the mixing task; the three explicitly penalised CVAE models and, to a lesser extent, the \textit{Cellinger} method have good mixing performances, with CoMP outperforming the benchmark models by a significant margin on the silhouette coefficient and $\kbet$ metric, while successfully maintaining a high accuracy in the cancer-type prediction task. We also see from Figure \ref{composite}A that CoMP representations have the fewest instances of isolated tumour-only clusters. Finally, from our evaluation on the $\tilde{s}$ and $\mkbet$ metrics, we can deduce that the occurrence of cell lines of one cancer type erroneously clustering around tumours of a different type is less frequent for CoMP compared to the other models.
In Appendix \ref{sec:experiment_details} we qualitatively validate this for several example clusters.
Overall, we see that CoMP learns to correctly align matching cell type clusters under different conditions \emph{without} any cell type labels being available during training.

\begin{figure}[t]
  \centering
  \includegraphics[width=1\columnwidth]{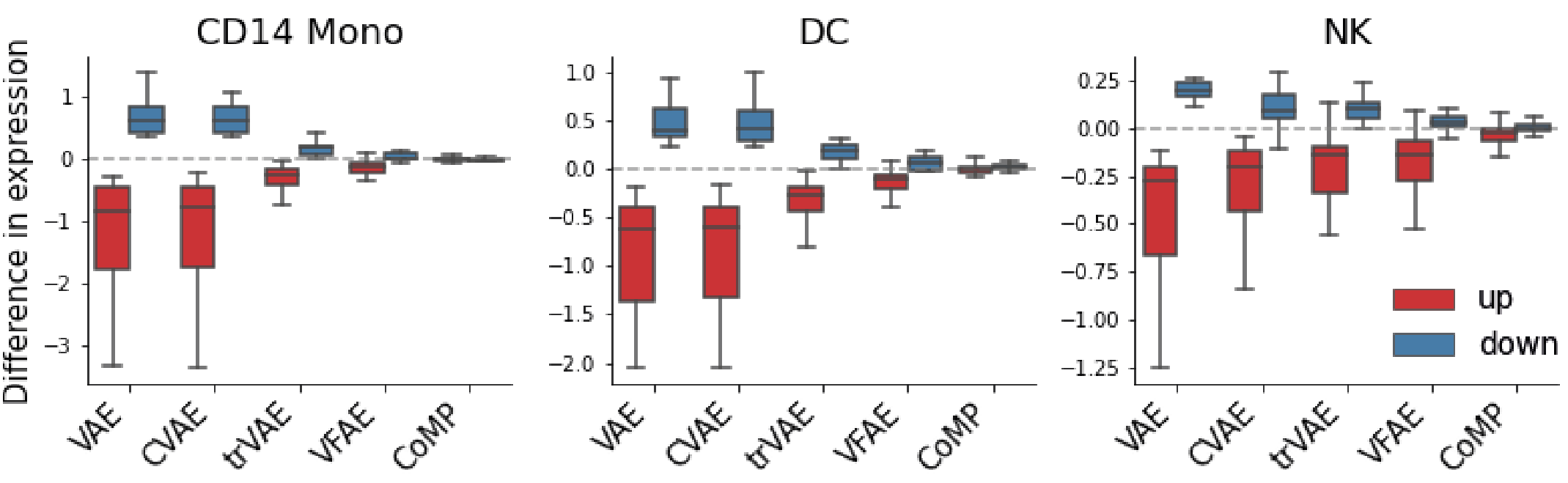}
  \caption{The difference in gene expression values for the top 50 differentially expressed genes (up-regulated: {\color{red}red}, down-regulated: {\color{blue}blue}) between IFN-$\beta$ stimulated cells and counterfactually stimulated control cells for CD14 monocytes, dendritic cells (DC) and natural killer (NK) cells. See Appendix~\ref{sec:experiment_details} for further details.}
  \label{fig:kang_cf}
\end{figure}

\definecolor{midblue}{RGB}{44, 127, 184} 
\definecolor{paleyellow}{RGB}{204, 190, 128} 
\begin{figure*}[t]
  \centering
  \includegraphics[width=2\columnwidth]{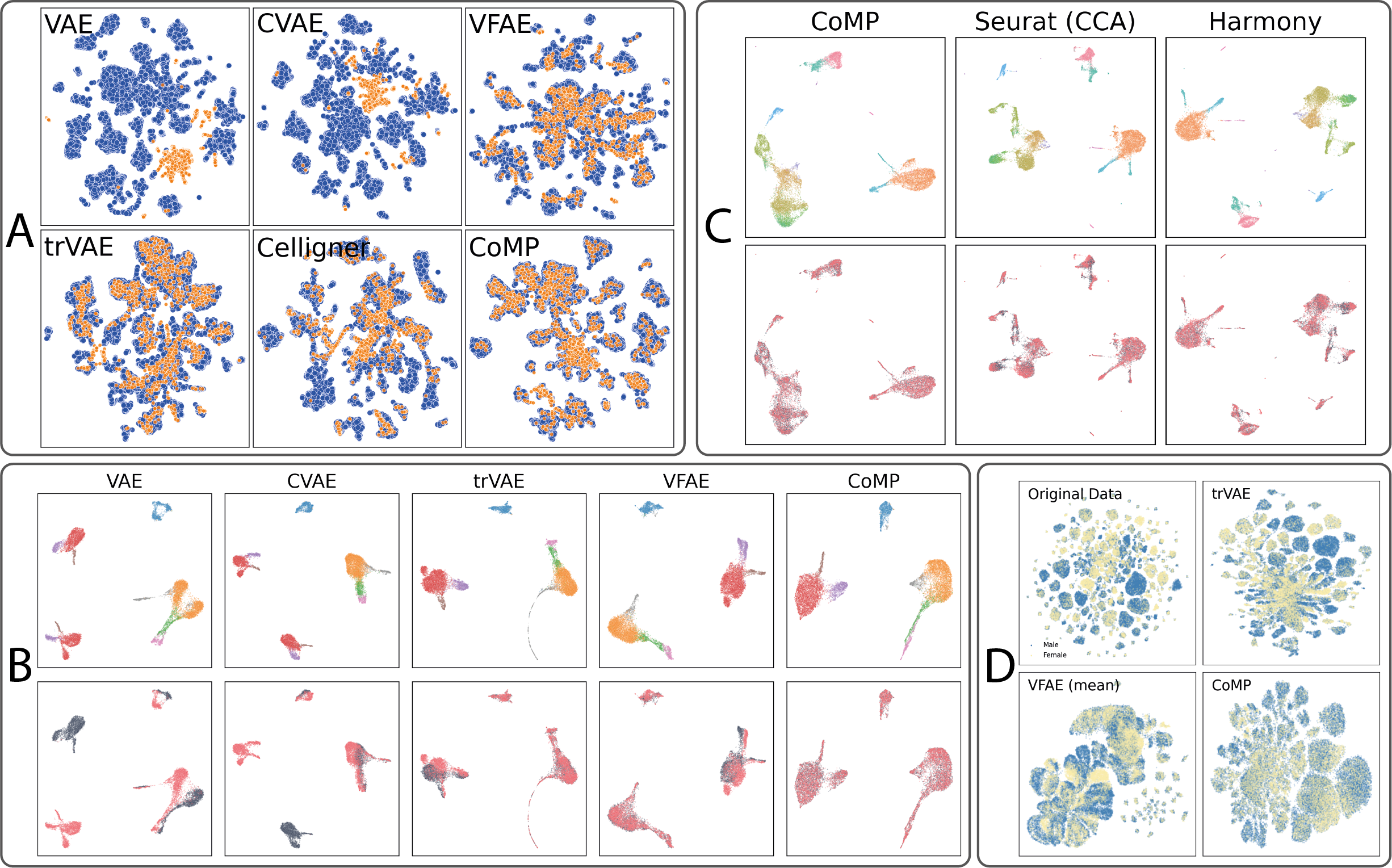}
  \caption{2D UMAP projection of posterior means of $\mathbf{z}_i$. \textbf{A}: \tcgaccle\ data. Tumours ({\color{blue} blue}) and cell lines ({\color{orange} orange}). \textbf{B}: Stimulated and control PBMC scRNA-seq data with colours highlighting immune cell types (top) and the IFN-$\beta$ condition (bottom). \textbf{C}: PBMC scRNA-seq data processed using different protocols, with colours indicating immune cell type (top) and 5-prime or 3-prime V2 library construction protocol (bottom); \textbf{D}: UCI Adult Income dataset, coloured by gender with male ({\color{midblue} blue}) and female ({\color{paleyellow} yellow}).}
  \label{composite}
\end{figure*}

\subsection{Interventions}
\label{sec:kang}

Obtaining molecular measurements from biological tissues typically requires destructive sampling,
meaning that we are unable to study the gene expression profile of the same cell under multiple experimental conditions. Counterfactual inference (Sec.~\ref{sec: cf}) can be used to predict how the molecular status of a destroyed biological sample would have differed if it were measured under different experimental conditions, such as applications of different drugs.

To assess CoMP’s utility in counterfactual inference, we trained it on scRNA-seq data from PBMCs that were either stimulated with IFN-$\beta$ or left untreated (control) \cite{kang2018multiplexed}. It is clear from Figure~\ref{composite}B that IFN-$\beta$ stimulation causes clear shifts in the latent space of a standard VAE between stimulated and control cells from the same cell type. Noticeably, the CD14 and CD16 monocyte and dendritic cell (DC) populations see greater shifts in their gene expression after stimulation. The CVAE fails to align these particular cell types in the latent space, while trVAE, VFAE and CoMP perform better. However, stimulated and control cells are best aligned in the latent space derived from CoMP (see metrics presented in Appendix~\ref{sec:experiment_details}).

CoMP suppresses latent space shifts caused by perturbations, but captures perturbation information in the decoder network. This can then be used
to predict gene expression levels of a cell type with and without the perturbation (counterfactual inference).
To validate this empirically, we perform a counterfactual prediction task for a IFN-$\beta$ control-to-stimulation variable swap, i.e.~the gene expression profiles for control cells were encoded to the latent space, then reconstructed through the decoder with the condition $c \mapsto \textrm{`stimulated'}$. This is a direct application of equation~\eqref{eq:approx_counter}.
We use held-out cell type labels to evaluate our predictions.
Figure~\ref{fig:kang_cf} shows how the profiles of (actual) stimulated cells differ from the counterfactual predictions for a selection of cell types (see Appendix~\ref{sec:experiment_details} for the complete set of results). We see that baseline models tend to systematically underestimate the expression of genes up-regulated by stimulation and overestimate those down-regulated. CoMP outperforms all other models by accurately predicting the expression alterations brought about by stimulation.

\subsection{Data integration of scRNA-seq data}
\label{sec:scrnaseq}

\begin{table}[t]
\small
\centering
\caption{scRNA-seq data integration experiment results, with $k=100$, $c=\text{Protocol}$, and $\alpha=0.05$ for kBET and m-kBET. $s_{k,c}$ and $\tilde{s}_{k,c}$ are the two Silhouette Coefficient variants (see Section~\ref{sec:experiments}). The top scores are in \textbf{bold}.}
\label{tab:scrna_batch_correction}
\begin{tabular}{cccccc}
& $s$ & ${\kbet}$ & $\tilde{s}$ & $\mkbet$
\\ \toprule
Seurat CCA & 0.0176 & 0.436 & 0.022 & 0.356 \\ 
Harmony & 0.0158 & 0.318 & 0.013 & 0.245 \\ 
\textit{CoMP} & $\bm{0.0004}$ & $\bm{0.164}$ & $\bm{0.0011}$ & $\bm{0.120}$ \\
\bottomrule
\end{tabular}
\end{table}

The scale and complexity of single-cell `omics datasets has increased rapidly in recent years \cite{mereu2020benchmarking, luecken2021benchmarking}. 
Efforts such as the Human Cell Atlas (HCA) \citep{regev2017science} require the collaboration of scientists from all around the world, each performing their own experiments and contributing their datasets to meet this goal. Processing these cells in different laboratories, with different protocols and technologies gives rise to distinct batch effects---unwanted technical variation observed in the data--that can obscure the biological variation that scientists seek to characterise and interpret. Being able to integrate diverse single-cell datasets, while accounting for unwanted technical variation, has been described as a major challenge in single-cell data science \cite{lahnemann2020eleven, eisenstein2020single, luecken2021benchmarking}.

To assess CoMP’s performance in integrating single-cell data with batch effects, we trained it on scRNA-seq data of PBMCs processed with different library preparation protocols \cite{Harmony}. Here, we compare CoMP to two widely used single-cell integration approaches: Seurat (CCA) \cite{stuart2019comprehensive} and Harmony \cite{Harmony}. As can be seen qualitatively in Figure~\ref{composite}C and quantitatively in Table~\ref{tab:scrna_batch_correction} (with further metrics presented in Appendix~\ref{sec:experiment_details}), CoMP provides improved mixing in the latent space for cells processed under different protocols when compared to Seurat and Harmony, while maintaining latent embeddings that represent distinct cell types.

\subsection{Fair classification}
\label{sec:uci_income}

\begin{table}[t]
\small
\centering
\caption{UCI Adult Income experiment results with $k = 1000$, $c = \text{Male}$ for silhouette score $s$, and $k=100$, $\alpha=0.01$ for $\kbet$. A lower gender prediction accuracy is better; $0.675$ is the lowest achievable. VFAE-s is VFAE version taken directly from \cite{louizos2015variational} with sampled latents, and VAE-m is our implementation where we take the posterior means.
}
\label{tab:income-results}
\begin{tabular}{ccccc}
& Gender Acc & Income Acc & $s$ & $\kbet$ \\ \toprule
Original data & 0.796 & \textbf{0.849} & 0.067 & 0.786\\ \cmidrule(r){1-5}
VAE & 0.764 & 0.812 & 0.054 & 0.748 \\
CVAE & 0.778 & 0.819 & 0.054 & 0.724 \\
VFAE-s & 0.680 & 0.815 & - & -\\
VFAE-m & 0.789 & 0.805 & 0.046 & 0.571 \\
trVAE & 0.698 & 0.808 & 0.066 & 0.731 \\
\textit{CoMP} & \textbf{0.679} & 0.805 & \textbf{0.011} & \textbf{0.451} \\
\bottomrule
\end{tabular}
\vspace{-5pt}
\end{table}

The goal for this fair classification task is to learn a representation on the Adult Income dataset that is not predictive of an individual's gender whilst still being predictive of their income. We compute a baseline by predicting gender and income labels directly from the input data and compare our method to the published results for the VFAE \cite{louizos2015variational} and the trVAE. We also include results for a standard VAE and CVAE. 

CoMP achieves a gender accuracy that is close to random (67.5\%), tying with the VFAE results from \cite{louizos2015variational} whilst also remaining competitive with the other methods on income accuracy (Table \ref{tab:income-results}). CoMP also outperforms all methods on the nearest neighbour and silhouette metrics. Latent space mixing between males and females can be seen qualitatively in the 2D UMAP projection (Figure \ref{composite}D).

\section{Discussion \& Related Work}
\label{sec:summary}

\paragraph{Data integration}
We have proposed CoMP as a new method for data integration and shown excellent performance on integrating bulk (Sec.~\ref{sec:tcga_ccle}) and single-cell (Sec.~\ref{sec:scrnaseq}) RNA-seq data.
CoMP uses the constraint $\z \indep c$ to perform data integration, assuming that different batches should be aligned in latent space.
Without additional prior knowledge, \emph{some} assumption is needed to conduct data integration.
Models such as PEER \citep{stegle2012using} or SVA \citep{Leek2007CapturingHI} estimate batch effects as latent factors. These models all assume that latent batch effects are independent of any biological variation as they are subsequently regressed out of the data.
Most similar to CoMP are generative models such as scVI \citep{lopez2018deep}, scMVAE \citep{zuo2021deep} and trVAE \citep{lotfollahi2019conditional}, the latter two use a CVAE framework and trVAE directly encourages $\z \indep c$ using an MMD penalty.
Methods that incorporate linear batch correction, such as ComBat \citep{Johnson2007AdjustingBE}, CCA \citep{correa2010canonical} and Harmony \citep{Harmony}, instead assume that batch effects can be isolated to linear components of the data, which can then be removed.
Other approaches \citep{Haghverdi2018BatchEI,Celligner} use matching of mutual nearest neighbours to reduce statistical dependency between different conditions, improving alignment.

If there are phenotypical differences between batches, it is possible that CoMP `over-corrects'.
One approach to alleviate this is to focus on the weight $\gamma$ in equation~\eqref{eq:total_training_obj} that scales the CoMP penalty. When this is small, the model will focus more on providing accurate representations, and the force to perfectly align representations will be smaller. More rigorously, $\gamma$ can be treated as a Lagrange multiplier, imposing a constraint $\sum_c p(c) \text{KL}\left[ q(\z|c) \| q(\z|\neg\, c) \right] \le L_\gamma$ for some constant $L_\gamma$. Thus, CoMP ensures that there is some level of alignment between different conditions, but may converge to a solution in which the constraint is not exactly zero.
Empirically, our mean-kBET and mean-silhouette scores are designed to check that CoMP correctly aligns matching cell types across batches; we find it does so in our experiments. More generally, these scores could be helpful to diagnose over-correction.

\paragraph{Counterfactual inference}
CoMP proved to be extremely successful at predicting the response of cells to IFN-$\beta$ stimulation. Predicting cell-level response to intervention has been previously tackled using VAEs by scGen \citep{lotfollahi2019scgen} and trVAE \citep{lotfollahi2019conditional}.
Although not always referred to as such, this problem is an example of counterfactual inference.
Both scGen and trVAE can be interpreted as approaching counterfactual inference by enforcing alignment between different conditions in latent space with a linear translation and MMD penalty respectively. \citet{johansson2016learning} estimated counterfactuals using representation learning, and used the discrepancy measure of \citet{mansour2009domain} to encourage latent alignment between conditions.

We showed that latent alignment, $\z \indep c$, is \emph{not} always sufficient to perform valid counterfactual inference with a CVAE. 
Our identifiability and consistency results (Sec.~\ref{sec:theory}) showed conditions under which counterfactual inference within the CVAE framework is valid. 
These results impact, not just CoMP, but other methods that use the CVAE framework for counterfactual inference, providing a blueprint to prove their counterfactual correctness.
Theorems~\ref{thm:counterfactualidentifiability} and~\ref{thm:consistency} do rest on assumptions; in Sec.~\ref{sec:theory:practical} we discussed practical methods to try and check them.
One particular limitation is that Theorem~\ref{thm:counterfactualidentifiability} assumes a linear decoder, whilst our experiments show that CoMP works well with non-linear decoders (shallow MLPs).
Further work may look to generalise our results and weaken the assumptions that they make.

\paragraph{Fair representations}
CoMP proved effective at creating representations that cannot be used to infer the protected condition (Sec.~\ref{sec:uci_income}).
The VFAE \citep{louizos2015variational} adopts a closely related approach to ours, using a CVAE and an MMD penalty to enforce $\z \indep c$.
Understanding fairness through counterfactuals \citep{kusner2017counterfactual} highlights the relevance of our theoretical analysis to fairness, showing that, under some conditions, CVAE approaches can go beyond alignment in latent space and directly target counterfactual notions of fairness.

\paragraph{Conclusion}
Marginal independence between the representation $\z$ and condition $c$ is a mathematical thread linking data integration, counterfactual inference and fairness.
We proposed CoMP, a novel method to enforce this independence requirement in practice.
We saw that CoMP has several attractive properties.
First, CoMP only uses the variational posteriors, requiring no additional discrepancy measures such as MMD.
Second, the CoMP penalty is an upper bound on a weighted sum of KL divergences, making the connection with a well-known divergence measure.
Third, we analysed CoMP theoretically as a means of performing counterfactual inference, showing identifiabilty and consistency under certain conditions.
Empirically, we demonstrated CoMP's performance when applied to three biological and one fair representation learning dataset. These biological datasets are of critical importance in drug discovery, for example matching cell-lines to tumours for effective pre-clinical assay development of anti-cancer compounds.
Overall, CoMP has the best in class performance on all tasks across a range of metrics and therefore broad utility across multiple challenging domains.

\section*{Acknowledgements}
AF gratefully acknowledges funding from EPSRC grant no.~EP/N509711/1. AF would like to thank Jake Fawkes for helpful comments on the counterfactual inference aspects of this paper. We would like to thank reviewers who reviewed this paper for NeurIPS 2021 and ICML 2022 and provided constructive feedback.

\bibliography{references}

\begin{thebibliography}{72}
\providecommand{\natexlab}[1]{#1}
\providecommand{\url}[1]{\texttt{#1}}
\expandafter\ifx\csname urlstyle\endcsname\relax
  \providecommand{\doi}[1]{doi: #1}\else
  \providecommand{\doi}{doi: \begingroup \urlstyle{rm}\Url}\fi

\bibitem[Aitkin(1991)]{aitkin_posterior_1991}
Aitkin, M.
\newblock Posterior {Bayes} {Factors}.
\newblock \emph{Journal of the Royal Statistical Society. Series B
  (Methodological)}, 53\penalty0 (1):\penalty0 111--142, 1991.
\newblock ISSN 0035-9246.
\newblock URL \url{https://www.jstor.org/stable/2345730}.

\bibitem[Amodio et~al.(2018)Amodio, Dijk, Montgomery, Wolf, and
  Krishnaswamy]{Amodio2018OutofSampleEW}
Amodio, M., Dijk, D.~V., Montgomery, R., Wolf, G., and Krishnaswamy, S.
\newblock Out-of-sample extrapolation with neuron editing.
\newblock \emph{arXiv: Quantitative Methods}, 2018.

\bibitem[Bareinboim \& Pearl(2016)Bareinboim and Pearl]{bareinboim2016causal}
Bareinboim, E. and Pearl, J.
\newblock Causal inference and the data-fusion problem.
\newblock \emph{Proceedings of the National Academy of Sciences}, 113\penalty0
  (27):\penalty0 7345--7352, 2016.

\bibitem[Bareinboim et~al.(2020)Bareinboim, Correa, Ibeling, and
  Icard]{bareinboim2020pearl}
Bareinboim, E., Correa, J.~D., Ibeling, D., and Icard, T.
\newblock {On Pearl’s hierarchy and the foundations of causal inference}.
\newblock \emph{ACM Special Volume in Honor of Judea Pearl}, 2\penalty0
  (3):\penalty0 4, 2020.

\bibitem[Bollen(2005)]{bollen2005structural}
Bollen, K.~A.
\newblock \emph{Structural equation models}.
\newblock Wiley, 2005.

\bibitem[B{\"u}ttner et~al.(2019)B{\"u}ttner, Miao, Wolf, Teichmann, and
  Theis]{kbet}
B{\"u}ttner, M., Miao, Z., Wolf, F.~A., Teichmann, S.~A., and Theis, F.~J.
\newblock A test metric for assessing single-cell rna-seq batch correction.
\newblock \emph{Nature methods}, 16\penalty0 (1):\penalty0 43--49, 2019.

\bibitem[Chen et~al.(2020)Chen, Kornblith, Norouzi, and Hinton]{chen2020simple}
Chen, T., Kornblith, S., Norouzi, M., and Hinton, G.
\newblock A simple framework for contrastive learning of visual
  representations.
\newblock \emph{arXiv preprint arXiv:2002.05709}, 2020.

\bibitem[Chiappa(2019)]{chiappa2019path}
Chiappa, S.
\newblock Path-specific counterfactual fairness.
\newblock In \emph{Proceedings of the AAAI Conference on Artificial
  Intelligence}, volume~33, pp.\  7801--7808, 2019.

\bibitem[Comon(1994)]{comon1994independent}
Comon, P.
\newblock Independent component analysis, a new concept?
\newblock \emph{Signal processing}, 36\penalty0 (3):\penalty0 287--314, 1994.

\bibitem[Correa et~al.(2010)Correa, Adali, Li, and
  Calhoun]{correa2010canonical}
Correa, N.~M., Adali, T., Li, Y.-O., and Calhoun, V.~D.
\newblock Canonical correlation analysis for data fusion and group inferences.
\newblock \emph{IEEE signal processing magazine}, 27\penalty0 (4):\penalty0
  39--50, 2010.

\bibitem[Dua \& Graff(2017)Dua and Graff]{Dua2019UCIIncome}
Dua, D. and Graff, C.
\newblock {UCI} machine learning repository, 2017.
\newblock URL \url{http://archive.ics.uci.edu/ml}.

\bibitem[Eisenstein(2020)]{eisenstein2020single}
Eisenstein, M.
\newblock Single-cell rna-seq analysis software providers scramble to offer
  solutions.
\newblock \emph{Nature biotechnology}, 38\penalty0 (3):\penalty0 254--257,
  2020.

\bibitem[Fearnhead \& Prangle(2012)Fearnhead and
  Prangle]{fearnhead_constructing_2012}
Fearnhead, P. and Prangle, D.
\newblock Constructing summary statistics for approximate {Bayesian}
  computation: semi-automatic approximate {Bayesian} computation.
\newblock \emph{Journal of the Royal Statistical Society: Series B (Statistical
  Methodology)}, 74\penalty0 (3):\penalty0 419--474, 2012.
\newblock ISSN 1467-9868.
\newblock \doi{https://doi.org/10.1111/j.1467-9868.2011.01010.x}.
\newblock URL
  \url{https://rss.onlinelibrary.wiley.com/doi/abs/10.1111/j.1467-9868.2011.01010.x}.

\bibitem[Foster et~al.(2020)Foster, Jankowiak, O’Meara, Teh, and
  Rainforth]{foster2020unified}
Foster, A., Jankowiak, M., O’Meara, M., Teh, Y.~W., and Rainforth, T.
\newblock A unified stochastic gradient approach to designing bayesian-optimal
  experiments.
\newblock In \emph{International Conference on Artificial Intelligence and
  Statistics}, pp.\  2959--2969. PMLR, 2020.

\bibitem[Gerhard et~al.(2018)Gerhard, Hunger, Lau, Maris, Meltzer, Meshinchi,
  Perlman, Zhang, Guidry-Auvil, and Smith]{target}
Gerhard, D., Hunger, S., Lau, C., Maris, J., Meltzer, P., Meshinchi, S.,
  Perlman, E., Zhang, J., Guidry-Auvil, J., and Smith, M.
\newblock Therapeutically applicable research to generate effective treatments
  (target) project: Half of pediatric cancers have their own" driver" genes.
\newblock In \emph{PEDIATRIC BLOOD \& CANCER}, volume~65, pp.\  S45--S45. WILEY
  111 RIVER ST, HOBOKEN 07030-5774, NJ USA, 2018.

\bibitem[Ghandi et~al.(2019)Ghandi, Huang, Jan{\'e}-Valbuena, Kryukov, Lo,
  McDonald, Barretina, Gelfand, Bielski, Li, Hu, Andreev-Drakhlin, Kim, Hess,
  Haas, Aguet, Weir, Rothberg, Paolella, Lawrence, Akbani, Lu, Tiv, Gokhale,
  de~Weck, Mansour, Oh, Shih, Hadi, Rosen, Bistline, Venkatesan, Reddy, Sonkin,
  Liu, Leh{\'a}r, Korn, Porter, Jones, Golji, Caponigro, Taylor, Dunning,
  Creech, Warren, McFarland, Zamanighomi, Kauffmann, Stransky, Imieliński,
  Maruvka, Cherniack, Tsherniak, Vazquez, Jaffe, Lane, Weinstock, Johannessen,
  Morrissey, Stegmeier, Schlegel, Hahn, Getz, Mills, Boehm, Golub, Garraway,
  and Sellers]{ccle}
Ghandi, M., Huang, F., Jan{\'e}-Valbuena, J., Kryukov, G., Lo, C., McDonald,
  E., Barretina, J., Gelfand, E., Bielski, C., Li, H., Hu, K.,
  Andreev-Drakhlin, A.~Y., Kim, J., Hess, J., Haas, B., Aguet, F., Weir, B.,
  Rothberg, M., Paolella, B., Lawrence, M., Akbani, R., Lu, Y., Tiv, H.~L.,
  Gokhale, P., de~Weck, A., Mansour, A.~A., Oh, C., Shih, J., Hadi, K., Rosen,
  Y., Bistline, J., Venkatesan, K., Reddy, A., Sonkin, D., Liu, M., Leh{\'a}r,
  J., Korn, J., Porter, D., Jones, M., Golji, J., Caponigro, G., Taylor, J.~E.,
  Dunning, C., Creech, A.~L., Warren, A., McFarland, J.~M., Zamanighomi, M.,
  Kauffmann, A., Stransky, N., Imieliński, M., Maruvka, Y., Cherniack, A.,
  Tsherniak, A., Vazquez, F., Jaffe, J., Lane, A.~A., Weinstock, D.,
  Johannessen, C., Morrissey, M.~P., Stegmeier, F., Schlegel, R., Hahn, W.,
  Getz, G., Mills, G., Boehm, J., Golub, T., Garraway, L., and Sellers, W.
\newblock {Next-generation characterization of the Cancer Cell Line
  Encyclopedia}.
\newblock \emph{Nature}, 569:\penalty0 503--508, 2019.

\bibitem[Glastonbury et~al.(2018)Glastonbury, Ferlaino, Nell{\aa}ker, and
  Lindgren]{glastonbury2018adjusting}
Glastonbury, C.~A., Ferlaino, M., Nell{\aa}ker, C., and Lindgren, C.~M.
\newblock Adjusting for confounding in unsupervised latent representations of
  images.
\newblock \emph{arXiv preprint arXiv:1811.06498}, 2018.

\bibitem[Goldman et~al.(2020)Goldman, Craft, Hastie, Repe{\v{c}}ka, McDade,
  Kamath, Banerjee, Luo, Rogers, Brooks, et~al.]{goldman2020visualizing}
Goldman, M.~J., Craft, B., Hastie, M., Repe{\v{c}}ka, K., McDade, F., Kamath,
  A., Banerjee, A., Luo, Y., Rogers, D., Brooks, A.~N., et~al.
\newblock Visualizing and interpreting cancer genomics data via the xena
  platform.
\newblock \emph{Nature biotechnology}, 38\penalty0 (6):\penalty0 675--678,
  2020.

\bibitem[Gretton et~al.(2012)Gretton, Borgwardt, Rasch, Sch{\"o}lkopf, and
  Smola]{gretton2012kernel}
Gretton, A., Borgwardt, K.~M., Rasch, M.~J., Sch{\"o}lkopf, B., and Smola, A.
\newblock A kernel two-sample test.
\newblock \emph{The Journal of Machine Learning Research}, 13\penalty0
  (1):\penalty0 723--773, 2012.

\bibitem[Gr{\o}nbech et~al.(2020)Gr{\o}nbech, Vording, Timshel, S{\o}nderby,
  Pers, and Winther]{gronbech2020scvae}
Gr{\o}nbech, C.~H., Vording, M.~F., Timshel, P.~N., S{\o}nderby, C.~K., Pers,
  T.~H., and Winther, O.
\newblock scvae: Variational auto-encoders for single-cell gene expression
  data.
\newblock \emph{Bioinformatics}, 36\penalty0 (16):\penalty0 4415--4422, 2020.

\bibitem[Haghverdi et~al.(2018)Haghverdi, Lun, Morgan, and
  Marioni]{Haghverdi2018BatchEI}
Haghverdi, L., Lun, A., Morgan, M.~D., and Marioni, J.
\newblock Batch effects in single-cell rna-sequencing data are corrected by
  matching mutual nearest neighbors.
\newblock \emph{Nature Biotechnology}, 36:\penalty0 421--427, 2018.

\bibitem[Higgins et~al.(2017)Higgins, Matthey, Pal, Burgess, Glorot, Botvinick,
  Mohamed, and Lerchner]{higgins2016beta}
Higgins, I., Matthey, L., Pal, A., Burgess, C., Glorot, X., Botvinick, M.,
  Mohamed, S., and Lerchner, A.
\newblock beta-vae: Learning basic visual concepts with a constrained
  variational framework.
\newblock \emph{International Conference on Learning Representations}, 2017.

\bibitem[Hyvarinen(1999)]{hyvarinen1999fast}
Hyvarinen, A.
\newblock Fast ica for noisy data using gaussian moments.
\newblock In \emph{1999 IEEE international symposium on circuits and systems
  (ISCAS)}, volume~5, pp.\  57--61. IEEE, 1999.

\bibitem[Hyvarinen et~al.(2001)Hyvarinen, Karhunen, and
  Oja]{hyvarinen2001independent}
Hyvarinen, A., Karhunen, J., and Oja, E.
\newblock Independent component analysis and blind source separation, 2001.

\bibitem[Johansson et~al.(2016)Johansson, Shalit, and
  Sontag]{johansson2016learning}
Johansson, F., Shalit, U., and Sontag, D.
\newblock Learning representations for counterfactual inference.
\newblock In \emph{International conference on machine learning}, pp.\
  3020--3029, 2016.

\bibitem[Johnson et~al.(2007)Johnson, Li, and
  Rabinovic]{Johnson2007AdjustingBE}
Johnson, W., Li, C., and Rabinovic, A.
\newblock Adjusting batch effects in microarray expression data using empirical
  {B}ayes methods.
\newblock \emph{Biostatistics}, 8 1:\penalty0 118--27, 2007.

\bibitem[Kang et~al.(2018)Kang, Subramaniam, Targ, Nguyen, Maliskova, McCarthy,
  Wan, Wong, Byrnes, Lanata, et~al.]{kang2018multiplexed}
Kang, H.~M., Subramaniam, M., Targ, S., Nguyen, M., Maliskova, L., McCarthy,
  E., Wan, E., Wong, S., Byrnes, L., Lanata, C.~M., et~al.
\newblock Multiplexed droplet single-cell rna-sequencing using natural genetic
  variation.
\newblock \emph{Nature biotechnology}, 36\penalty0 (1):\penalty0 89, 2018.

\bibitem[Khemakhem et~al.(2020)Khemakhem, Kingma, Monti, and
  Hyvarinen]{khemakhem2020variational}
Khemakhem, I., Kingma, D., Monti, R., and Hyvarinen, A.
\newblock Variational autoencoders and nonlinear ica: A unifying framework.
\newblock In \emph{International Conference on Artificial Intelligence and
  Statistics}, pp.\  2207--2217. PMLR, 2020.

\bibitem[Kilbertus et~al.(2017)Kilbertus, Rojas~Carulla, Parascandolo, Hardt,
  Janzing, and Sch\"{o}lkopf]{NIPS2017_f5f8590c}
Kilbertus, N., Rojas~Carulla, M., Parascandolo, G., Hardt, M., Janzing, D., and
  Sch\"{o}lkopf, B.
\newblock Avoiding discrimination through causal reasoning.
\newblock In Guyon, I., Luxburg, U.~V., Bengio, S., Wallach, H., Fergus, R.,
  Vishwanathan, S., and Garnett, R. (eds.), \emph{Advances in Neural
  Information Processing Systems}, volume~30. Curran Associates, Inc., 2017.
\newblock URL
  \url{https://proceedings.neurips.cc/paper/2017/file/f5f8590cd58a54e94377e6ae2eded4d9-Paper.pdf}.

\bibitem[Kingma \& Ba(2014)Kingma and Ba]{kingma2014adam}
Kingma, D.~P. and Ba, J.
\newblock Adam: A method for stochastic optimization.
\newblock \emph{arXiv preprint arXiv:1412.6980}, 2014.

\bibitem[Kingma \& Welling(2013)Kingma and Welling]{kingma2013auto}
Kingma, D.~P. and Welling, M.
\newblock Auto-encoding variational bayes.
\newblock \emph{arXiv preprint arXiv:1312.6114}, 2013.

\bibitem[Korsunsky et~al.(2019)Korsunsky, Millard, Fan, Slowikowski, Zhang,
  Wei, Baglaenko, Brenner, Loh, and Raychaudhuri]{Harmony}
Korsunsky, I., Millard, N., Fan, J., Slowikowski, K., Zhang, F., Wei, K.,
  Baglaenko, Y., Brenner, M., Loh, P.-r., and Raychaudhuri, S.
\newblock Fast, sensitive and accurate integration of single-cell data with
  harmony.
\newblock \emph{Nature methods}, 16\penalty0 (12):\penalty0 1289--1296, 2019.

\bibitem[Kusner et~al.(2017)Kusner, Loftus, Russell, and
  Silva]{kusner2017counterfactual}
Kusner, M.~J., Loftus, J.~R., Russell, C., and Silva, R.
\newblock Counterfactual fairness.
\newblock \emph{arXiv preprint arXiv:1703.06856}, 2017.

\bibitem[Lacoste et~al.(2019)Lacoste, Luccioni, Schmidt, and
  Dandres]{lacoste2019quantifying}
Lacoste, A., Luccioni, A., Schmidt, V., and Dandres, T.
\newblock Quantifying the carbon emissions of machine learning.
\newblock \emph{arXiv preprint arXiv:1910.09700}, 2019.

\bibitem[L{\"a}hnemann et~al.(2020)L{\"a}hnemann, K{\"o}ster, Szczurek,
  McCarthy, Hicks, Robinson, Vallejos, Campbell, Beerenwinkel, Mahfouz,
  et~al.]{lahnemann2020eleven}
L{\"a}hnemann, D., K{\"o}ster, J., Szczurek, E., McCarthy, D.~J., Hicks, S.~C.,
  Robinson, M.~D., Vallejos, C.~A., Campbell, K.~R., Beerenwinkel, N., Mahfouz,
  A., et~al.
\newblock Eleven grand challenges in single-cell data science.
\newblock \emph{Genome biology}, 21\penalty0 (1):\penalty0 1--35, 2020.

\bibitem[Leek \& Storey(2007)Leek and Storey]{Leek2007CapturingHI}
Leek, J. and Storey, J.~D.
\newblock Capturing heterogeneity in gene expression studies by surrogate
  variable analysis.
\newblock \emph{PLoS Genetics}, 3, 2007.

\bibitem[Li \& Malik(2018)Li and Malik]{li2018implicit}
Li, K. and Malik, J.
\newblock Implicit maximum likelihood estimation.
\newblock \emph{arXiv preprint arXiv:1809.09087}, 2018.

\bibitem[Lopez et~al.(2018)Lopez, Regier, Cole, Jordan, and
  Yosef]{lopez2018deep}
Lopez, R., Regier, J., Cole, M.~B., Jordan, M.~I., and Yosef, N.
\newblock Deep generative modeling for single-cell transcriptomics.
\newblock \emph{Nature methods}, 15\penalty0 (12):\penalty0 1053--1058, 2018.

\bibitem[Lotfollahi et~al.(2019{\natexlab{a}})Lotfollahi, Naghipourfar, Theis,
  and Wolf]{lotfollahi2019conditional}
Lotfollahi, M., Naghipourfar, M., Theis, F.~J., and Wolf, F.~A.
\newblock {Conditional out-of-sample generation for unpaired data using trVAE}.
\newblock \emph{arXiv preprint arXiv:1910.01791}, 2019{\natexlab{a}}.

\bibitem[Lotfollahi et~al.(2019{\natexlab{b}})Lotfollahi, Wolf, and
  Theis]{lotfollahi2019scgen}
Lotfollahi, M., Wolf, F.~A., and Theis, F.~J.
\newblock {scGen predicts single-cell perturbation responses}.
\newblock \emph{Nature methods}, 16\penalty0 (8):\penalty0 715,
  2019{\natexlab{b}}.

\bibitem[Louizos et~al.(2015)Louizos, Swersky, Li, Welling, and
  Zemel]{louizos2015variational}
Louizos, C., Swersky, K., Li, Y., Welling, M., and Zemel, R.
\newblock The variational fair autoencoder.
\newblock \emph{arXiv preprint arXiv:1511.00830}, 2015.

\bibitem[Luecken et~al.(2021)Luecken, B{\"u}ttner, Chaichoompu, Danese,
  Interlandi, Mueller, Strobl, Zappia, Dugas, Colom{\'e}-Tatch{\'e},
  et~al.]{luecken2021benchmarking}
Luecken, M.~D., B{\"u}ttner, M., Chaichoompu, K., Danese, A., Interlandi, M.,
  Mueller, M.~F., Strobl, D.~C., Zappia, L., Dugas, M., Colom{\'e}-Tatch{\'e},
  M., et~al.
\newblock Benchmarking atlas-level data integration in single-cell genomics.
\newblock \emph{Nature methods}, pp.\  1--10, 2021.

\bibitem[Mansour et~al.(2009)Mansour, Mohri, and
  Rostamizadeh]{mansour2009domain}
Mansour, Y., Mohri, M., and Rostamizadeh, A.
\newblock Domain adaptation: Learning bounds and algorithms.
\newblock \emph{arXiv preprint arXiv:0902.3430}, 2009.

\bibitem[Mathieu et~al.(2019)Mathieu, Rainforth, Siddharth, and
  Teh]{mathieu2019disentangling}
Mathieu, E., Rainforth, T., Siddharth, N., and Teh, Y.~W.
\newblock Disentangling disentanglement in variational autoencoders.
\newblock In \emph{In International Conference on Machine Learning}, pp.\
  4402--4412. PMLR, 2019.

\bibitem[Mereu et~al.(2020)Mereu, Lafzi, Moutinho, Ziegenhain, McCarthy,
  {\'A}lvarez-Varela, Batlle, Gr{\"u}n, Lau, Boutet,
  et~al.]{mereu2020benchmarking}
Mereu, E., Lafzi, A., Moutinho, C., Ziegenhain, C., McCarthy, D.~J.,
  {\'A}lvarez-Varela, A., Batlle, E., Gr{\"u}n, D., Lau, J.~K., Boutet, S.~C.,
  et~al.
\newblock Benchmarking single-cell rna-sequencing protocols for cell atlas
  projects.
\newblock \emph{Nature Biotechnology}, 38\penalty0 (6):\penalty0 747--755,
  2020.

\bibitem[Moneta et~al.(2013)Moneta, Entner, Hoyer, and Coad]{moneta2013causal}
Moneta, A., Entner, D., Hoyer, P.~O., and Coad, A.
\newblock Causal inference by independent component analysis: Theory and
  applications.
\newblock \emph{Oxford Bulletin of Economics and Statistics}, 75\penalty0
  (5):\penalty0 705--730, 2013.

\bibitem[Pearl(2000)]{pearl2000models}
Pearl, J.
\newblock \emph{Causality: Models, reasoning and inference}.
\newblock Cambridge University Press, 2000.

\bibitem[Pearl(2009)]{pearl2009causality}
Pearl, J.
\newblock \emph{Causality}.
\newblock Cambridge university press, 2009.

\bibitem[Razali et~al.(2011)Razali, Wah, et~al.]{razali2011power}
Razali, N.~M., Wah, Y.~B., et~al.
\newblock {Power comparisons of Shapiro-Wilk, Kolmogorov-Smirnov, Lilliefors
  and Anderson-Darling tests}.
\newblock \emph{Journal of statistical modeling and analytics}, 2\penalty0
  (1):\penalty0 21--33, 2011.

\bibitem[Regev et~al.(2017)Regev, Teichmann, Lander, Amit, Benoist, Birney,
  Bodenmiller, Campbell, Carninci, Clatworthy, et~al.]{regev2017science}
Regev, A., Teichmann, S.~A., Lander, E.~S., Amit, I., Benoist, C., Birney, E.,
  Bodenmiller, B., Campbell, P., Carninci, P., Clatworthy, M., et~al.
\newblock Science forum: the human cell atlas.
\newblock \emph{Elife}, 6:\penalty0 e27041, 2017.

\bibitem[Rezende et~al.(2014)Rezende, Mohamed, and
  Wierstra]{Rezende2014StochasticBA}
Rezende, D.~J., Mohamed, S., and Wierstra, D.
\newblock Stochastic backpropagation and approximate inference in deep
  generative models.
\newblock In \emph{ICML}, 2014.

\bibitem[Robbins \& Monro(1951)Robbins and Monro]{robbins1951stochastic}
Robbins, H. and Monro, S.
\newblock A stochastic approximation method.
\newblock \emph{The annals of mathematical statistics}, pp.\  400--407, 1951.

\bibitem[Robert \& Rousseau(2017)Robert and Rousseau]{robert_how_2017}
Robert, C.~P. and Rousseau, J.
\newblock How {Principled} and {Practical} {Are} {Penalised} {Complexity}
  {Priors}?
\newblock \emph{Statistical Science}, 32\penalty0 (1):\penalty0 36--40,
  February 2017.
\newblock ISSN 0883-4237, 2168-8745.

\bibitem[Rolinek et~al.(2019)Rolinek, Zietlow, and
  Martius]{rolinek2019variational}
Rolinek, M., Zietlow, D., and Martius, G.
\newblock Variational autoencoders pursue pca directions (by accident).
\newblock In \emph{Proceedings of the IEEE/CVF Conference on Computer Vision
  and Pattern Recognition}, pp.\  12406--12415, 2019.

\bibitem[Rousseeuw(1987)]{rousseeuw1987silhouettes}
Rousseeuw, P.~J.
\newblock Silhouettes: a graphical aid to the interpretation and validation of
  cluster analysis.
\newblock \emph{Journal of computational and applied mathematics}, 20:\penalty0
  53--65, 1987.

\bibitem[Shimizu et~al.(2006)Shimizu, Hoyer, Hyv{\"a}rinen, Kerminen, and
  Jordan]{shimizu2006linear}
Shimizu, S., Hoyer, P.~O., Hyv{\"a}rinen, A., Kerminen, A., and Jordan, M.
\newblock A linear non-gaussian acyclic model for causal discovery.
\newblock \emph{Journal of Machine Learning Research}, 7\penalty0 (10), 2006.

\bibitem[Sohn et~al.(2015)Sohn, Lee, and Yan]{sohn2015learning}
Sohn, K., Lee, H., and Yan, X.
\newblock Learning structured output representation using deep conditional
  generative models.
\newblock In \emph{Advances in neural information processing systems}, pp.\
  3483--3491, 2015.

\bibitem[Stegle et~al.(2012)Stegle, Parts, Piipari, Winn, and
  Durbin]{stegle2012using}
Stegle, O., Parts, L., Piipari, M., Winn, J., and Durbin, R.
\newblock Using probabilistic estimation of expression residuals (peer) to
  obtain increased power and interpretability of gene expression analyses.
\newblock \emph{Nature protocols}, 7\penalty0 (3):\penalty0 500--507, 2012.

\bibitem[Stuart et~al.(2019)Stuart, Butler, Hoffman, Hafemeister, Papalexi,
  Mauck~III, Hao, Stoeckius, Smibert, and Satija]{stuart2019comprehensive}
Stuart, T., Butler, A., Hoffman, P., Hafemeister, C., Papalexi, E., Mauck~III,
  W.~M., Hao, Y., Stoeckius, M., Smibert, P., and Satija, R.
\newblock Comprehensive integration of single-cell data.
\newblock \emph{Cell}, 177\penalty0 (7):\penalty0 1888--1902, 2019.

\bibitem[Svensson et~al.(2018)Svensson, Vento-Tormo, and
  Teichmann]{svensson2018exponential}
Svensson, V., Vento-Tormo, R., and Teichmann, S.~A.
\newblock Exponential scaling of single-cell rna-seq in the past decade.
\newblock \emph{Nature protocols}, 13\penalty0 (4):\penalty0 599--604, 2018.

\bibitem[Tomczak \& Welling(2018)Tomczak and Welling]{tomczak2018vae}
Tomczak, J. and Welling, M.
\newblock Vae with a vampprior.
\newblock In \emph{International Conference on Artificial Intelligence and
  Statistics}, pp.\  1214--1223. PMLR, 2018.

\bibitem[van~den Oord et~al.(2018)van~den Oord, Li, and
  Vinyals]{oord2018representation}
van~den Oord, A., Li, Y., and Vinyals, O.
\newblock Representation learning with contrastive predictive coding.
\newblock \emph{arXiv preprint arXiv:1807.03748}, 2018.

\bibitem[Vert et~al.(2004)Vert, Tsuda, and Sch{\"o}lkopf]{vert2004primer}
Vert, J.-P., Tsuda, K., and Sch{\"o}lkopf, B.
\newblock A primer on kernel methods.
\newblock \emph{Kernel methods in computational biology}, 47:\penalty0 35--70,
  2004.

\bibitem[Wang \& Gu(2018)Wang and Gu]{wang2018vasc}
Wang, D. and Gu, J.
\newblock Vasc: dimension reduction and visualization of single-cell rna-seq
  data by deep variational autoencoder.
\newblock \emph{Genomics, proteomics \& bioinformatics}, 16\penalty0
  (5):\penalty0 320--331, 2018.

\bibitem[Warren et~al.(2021)Warren, Chen, Jones, Shibue, Hahn, Boehm, Vazquez,
  Tsherniak, and McFarland]{Celligner}
Warren, A., Chen, Y., Jones, A., Shibue, T., Hahn, W.~C., Boehm, J.~S.,
  Vazquez, F., Tsherniak, A., and McFarland, J.~M.
\newblock Global computational alignment of tumor and cell line transcriptional
  profiles.
\newblock \emph{Nature Communications}, 12\penalty0 (1):\penalty0 1--12, 2021.

\bibitem[Way \& Greene(2018)Way and Greene]{way2018extracting}
Way, G.~P. and Greene, C.~S.
\newblock Extracting a biologically relevant latent space from cancer
  transcriptomes with variational autoencoders.
\newblock In \emph{PACIFIC SYMPOSIUM ON BIOCOMPUTING 2018: Proceedings of the
  Pacific Symposium}, pp.\  80--91. World Scientific, 2018.

\bibitem[Weinstein et~al.(2013)Weinstein, Collisson, Mills, Shaw, Ozenberger,
  Ellrott, Shmulevich, Sander, and Stuart]{tcga}
Weinstein, J., Collisson, E., Mills, G., Shaw, K., Ozenberger, B., Ellrott, K.,
  Shmulevich, I., Sander, C., and Stuart, J.~M.
\newblock {The Cancer Genome Atlas Pan-Cancer analysis project}.
\newblock \emph{Nature Genetics}, 45:\penalty0 1113--1120, 2013.

\bibitem[Wolf et~al.(2018)Wolf, Angerer, and Theis]{scanpy}
Wolf, F.~A., Angerer, P., and Theis, F.~J.
\newblock Scanpy: large-scale single-cell gene expression data analysis.
\newblock \emph{Genome biology}, 19\penalty0 (1):\penalty0 1--5, 2018.

\bibitem[Zemel et~al.(2013)Zemel, Wu, Swersky, Pitassi, and
  Dwork]{zemel2013learning}
Zemel, R., Wu, Y., Swersky, K., Pitassi, T., and Dwork, C.
\newblock Learning fair representations.
\newblock In \emph{International conference on machine learning}, pp.\
  325--333. PMLR, 2013.

\bibitem[Zhang \& Bareinboim(2018)Zhang and Bareinboim]{NEURIPS2018_ff1418e8}
Zhang, J. and Bareinboim, E.
\newblock Equality of opportunity in classification: A causal approach.
\newblock In Bengio, S., Wallach, H., Larochelle, H., Grauman, K.,
  Cesa-Bianchi, N., and Garnett, R. (eds.), \emph{Advances in Neural
  Information Processing Systems}, volume~31. Curran Associates, Inc., 2018.
\newblock URL
  \url{https://proceedings.neurips.cc/paper/2018/file/ff1418e8cc993fe8abcfe3ce2003e5c5-Paper.pdf}.

\bibitem[Zhang et~al.(2016)Zhang, Wu, and Wu]{zhang2016causal}
Zhang, L., Wu, Y., and Wu, X.
\newblock A causal framework for discovering and removing direct and indirect
  discrimination.
\newblock \emph{arXiv preprint arXiv:1611.07509}, 2016.

\bibitem[Zuo \& Chen(2021)Zuo and Chen]{zuo2021deep}
Zuo, C. and Chen, L.
\newblock Deep-joint-learning analysis model of single cell transcriptome and
  open chromatin accessibility data.
\newblock \emph{Briefings in Bioinformatics}, 22\penalty0 (4):\penalty0
  bbaa287, 2021.

\end{thebibliography}
\bibliographystyle{icml2022}

\newpage
\appendix
\onecolumn

\section{Additional background}

\subsection{Priors from posteriors}
Given the long-standing debates around the role, selection and treatment of the prior within Bayesian statistics, it is natural that the choice of $p(\mathbf{z})$ in VAEs has come under scrutiny.
The simplest alteration to dealing with the VAE prior is the $\beta$-VAE \cite{higgins2016beta}, which scales the KL term of the ELBO by a hyperparameter $\beta$.
While many of the traditional arguments concerning the prior revolve around principled points on objectivity, the primary issue for VAEs is the lack of expressiveness of the standard Normal distribution \cite{mathieu2019disentangling}. The shared concern is that the prior is often selected for practical but, ultimately, spurious reasons of technical convenience (e.g. conjugacy, reparametrisation trick). 

One solution is to simply replace the prior with the posterior. The apparent simplicity of this approach obscures the multiple issues that arise from double-dipping the data \cite{robert_how_2017, aitkin_posterior_1991}. Nevertheless the idea has endured: from the earlier proposal of posterior Bayes Factors as a solution to Lindley's paradox \cite{aitkin_posterior_1991}, modern Empirical Bayes methods, to likelihood-free models such as the calibration in approximate Bayesian computation models \cite{fearnhead_constructing_2012}, invoking the posterior `before its time' is increasingly performed to anchor statistical models to a more objective foundation. 

For VAEs, a well-known proposal is to replace the prior with a mixture of variational posteriors, formed using pseudo-observations $\mathbf{u}_1,...,\mathbf{u}_K$ \cite{tomczak2018vae}. This Variational Mixture of Posteriors (VaMP) prior is given by
\begin{equation}
    p^{\mathrm{VaMP}}(\z) = \frac{1}{K}\sum_{k=1}^K q_\phi(\z|\mathbf{u}_k).
\end{equation}
This results in a multi-modal prior, with the pseudo-observations learned by stochastic backpropagation along with the other parameters $\theta,\phi$. As we define in Section \ref{sec:contrastive}, the CoMP method adopts a similar non-parametric approach to defining a misalignment penalty.

\section{Characterising the constraint $\z \indep c$}
\label{sec:characterising}
To connect different notions of `alignment in representation space' we recall the key components of the CVAE---the encoder $q_\phi(\z|\x,c)$ and decoder $p_\theta(\x|\z,c)$---and we now drop the $\theta,\phi$ subscripts for conciseness.
Recall that the marginal distribution of representations within condition $c\in\mathcal{C}$ is $q(\z|c) = \E_{p(\x|c)}\left[ q(\z|\x,c) \right]$, 
and the marginal distribution of $\z$ over all conditions not equal to $c$ is
\begin{equation}
    q(\z|\neg\,c) = \frac{\sum_{c'\in\mathcal{C}, c'\ne c}p(c')q(\z|c')}{\sum_{c'\in\mathcal{C},c'\ne c} p(c')}
\end{equation}
in our notation.
The following proposition brings together several key notions of alignment in the latent space. 
\begin{restatable}{proposition}{counterfactual}
\label{thm:counterfactual}
The following are equivalent:
\begin{enumerate}
    \item $\z \indep c$ under distribution $q$,
    \item for every $c,c'\in\mathcal{C}$, $q(\z|c)=q(\z|c')$,
    \item for every $c\in\mathcal{C}$, $q(\z|c) = q(\z|\neg\,c)$,
    \item the mutual information $I(\z,c)=0$ under distribution $q$,
    \item $\z$ cannot predict $c$ better than random guessing.
\end{enumerate}
\end{restatable}
\begin{proof}
    $1. \implies 2.$
    If $\z\indep c$, then for every $c,c'\in\mathcal{C}$, $q(\z|c) = q(\z) = q(\z|c')$.
    
    $2.\implies 3.$
    For $c\in\mathcal{C}$, by the definition of $q(\z|\neg\,c)$ we have
    \begin{equation}
        q(\z|\neg\,c) = \frac{\sum_{c'\in\mathcal{C},c'\ne c}p(c')q(\z|c')}{\sum_{c'\in\mathcal{C},c'\ne c} p(c')} = \frac{\sum_{c'\in\mathcal{C},c'\ne c}p(c')q(\z|c)}{\sum_{c'\in\mathcal{C},c'\ne c} p(c')} = q(\z|c)
    \end{equation}
    using condition 2. 
    
    $3. \implies 4.$
    We have by definition of the mutual information under distribution $q$
    \begin{align}
        I(\z,c) &= \E_{p(\x,c)q(\z|\x,c)}\left[\log \frac{p(c)q(\z|c)}{p(c)q(\z)}\right] \\
        \intertext{which can be written}
        &= \E_{p(\x,c)q(\z|\x,c)}\left[\log \frac{p(c)q(\z|c)}{p(c)[p(c)q(\z|c) + (1-p(c))q(\z|\neg c)]}\right] \\
        \intertext{applying condition 3. gives}
        &= \E_{p(\x,c)q(\z|\x,c)}\left[\log \frac{p(c)q(\z|c)}{p(c)[p(c)q(\z|c) + (1-p(c))q(\z|c)]}\right] \\
        &= \E_{p(\x,c)q(\z|\x,c)}\left[\log \frac{p(c)q(\z|c)}{p(c)q(\z|c)}\right] \\
        &=0.
    \end{align}
    
    $4. \implies 5.$\footnote{We interpret `better prediction' in condition 5. as achieving a higher expected log-likelihood.}
    Let $Q(c|\z)$ be some prediction rule for predicting $c$ using $\z$. By Gibbs' Inequality, we have
    \begin{equation}
        I(\z,c) \ge \E_{p(\x,c)q(\z|\x,c)}\left[\log \frac{Q(c|\z)}{p(c)}\right].
    \end{equation}
    Since $I(\z,c)=0$, we have
    \begin{equation}
        \E_{p(\x,c)q(\z|\x,c)}\left[\log {Q(c|\z)}\right] \le \E_{p(\x,c)}\left[\log {p(c)}\right].
    \end{equation}
    Observe that the left hand side above is the expected log-likelihood for the prediction rule $Q$, whilst the right hand side is the the log-likelihood for random guessing of $c$ using only its marginal distribution $p(c)$.
    We see that random guessing obtains a log-likelihood which is at least as good as that obtained using the rule $Q$.
    
    $5. \implies 1.$
    Consider the prediction rule
    \begin{equation}
        Q^*(c|\z) := \frac{p(c)q(\z|c)}{q(\z)}.
    \end{equation}
    By condition 5., we have
    \begin{equation}
        \E_{p(\x,c)q(\z|\x,c)}\left[\log {Q^*(c|\z)}\right] \le \E_{p(\x,c)}\left[\log {p(c)}\right].
    \end{equation}
    Hence,
    \begin{equation}
    \label{eq:proof_info}
        \E_{p(\x,c)q(\z|\x,c)}\left[\log \frac{p(c)q(\z|c)}{p(c)q(\z)}\right] \le 0.
    \end{equation}
    By Gibbs' Inequality, 
    \begin{equation}
        \E_{p(\x,c)q(\z|\x,c)}\left[\log \frac{p(c)q(\z|c)}{p(c)q(\z)}\right] \ge 0
    \end{equation}
    with equality if and only if $p(c)q(\z|c) = p(c)q(\z)$. By \eqref{eq:proof_info}, equality does hold, so $p(c)q(\z|c) = p(c)q(\z)$ meaning $\z \indep c$ under distribution $q$.
\end{proof}

\section{CoMP misalignment penalty}
\label{sec:proofs}
We restate and prove Theorem~\ref{thm:kl}.

\kl*

\begin{proof}
    First, by linearity of the expectation we have
\begin{align}
\begin{split}
    &\E_{\prod_{i=1}^B p(\x_i,c_i)q(\z_i|\x_i,c_i)}\left[\frac{1}{B}\sum_{i=1}^B  \log\left(\frac{1}{|I_{c_i}|}\sum_{j\in I_{c_i}} q(\z_i|\x_j,c_i) \right) - \log\left(\frac{1}{|I_{\neg c_i}|} \sum_{j\in I_{\neg c_i}} q(\z_i|\x_j, c_j) \right) \right] \\ &\  = \E_{\prod_{i=1}^B p(\x_i,c_i)q(\z_i|\x_i,c_i)}\left[\log\left(\frac{1}{|I_{c_1}|}\sum_{j\in I_{c_1}} q(\z_1|\x_j,c_i) \right) - \log\left(\frac{1}{|I_{\neg c_1}|} \sum_{j\in I_{\neg c_1}} q(\z_1|\x_j, c_j) \right) \right].
\end{split}
\end{align}
Focusing on the latter term, Jensen's Inequality gives
\begin{align}
&\E_{\prod_{i=1}^B p(\x_i,c_i)q(\z_i|\x_i,c_i)}\left[ - \log\left(\frac{1}{|I_{\neg c_1}|} \sum_{j\in I_{\neg c_1}} q(\z_1|\x_j, c_j) \right) \right] \\
&\ \ge\E_{p(\x_1,c_1)q(\z_1|\x_1,c_1)}\left[ - \log\left(\E_{\prod_{i>1}^B p(\x_i,c_i)q(\z_i|\x_i,c_i)}\left[\frac{1}{|I_{\neg c_1}|} \sum_{j\in I_{\neg c_1}} q(\z_1|\x_j, c_j) \right]\right) \right] \\
&\ =\E_{p(\x_1,c_1)q(\z_1|\x_1,c_1)}\left[ - \log q(\z_1|\neg\, c_1) \right].
\end{align}
For the other term, we take our inspiration from recent work on experimental design \cite{foster2020unified}.
We have
\begin{align}
&\E_{\prod_{i=1}^B p(\x_i,c_i)q(\z_i|\x_i,c_i)}\left[\log\left(\frac{1}{|I_{c_1}|}\sum_{j\in I_{c_1}} q(\z_1|\x_j,c_i) \right)  \right] \\
&\ = \E_{\prod_{i=1}^B p(\x_i,c_i)q(\z_i|\x_i,c_i)}\left[\log q(\z_1|c_1)  +\log\left(\frac{\frac{1}{|I_{c_1}|}\sum_{j\in I_{c_1}} q(\z_1|\x_j,c_i)}{q(\z_1|c_1)} \right)  \right] \\
&\ = \E_{\prod_{i=1}^B p(\x_i,c_i)q(\z_i|\x_i,c_i)}\left[\log q(\z_1|c_1)  \right] + \Delta
\end{align}
Then applying the tower rule with variable $c_1, \left|I_{c_1}\right|$ we have the difference term equal to
\begin{align}
\Delta &=  \E_{c_1,\left|I_{c_1}\right|}\left[\E_{\prod_{i=1}^{\left|I_{c_1}\right|} p(\x_i|c_1)   q(\z_1|\x_1,c_1)}\left[\log\left(\frac{\frac{1}{|I_{c_1}|}\sum_{j\in I_{c_1}} q(\z_1|\x_j,c_i)}{q(\z_1|c_1)} \right)  \right]\right]\\
&=\E_{c_1,\left|I_{c_1}\right|}\left[\E_{\prod_{i=1}^{\left|I_{c_1}\right|} p(\x_i|c_1)   q(\z_1|\x_1,c_1)}\left[\log\left(\frac{\prod_{i=1}^{\left|I_{c_1}\right|} p(\x_i|c_1) \frac{1}{|I_{c_1}|}\sum_{j\in I_{c_1}} q(\z_1|\x_j,c_i)}{\prod_{i=1}^{\left|I_{c_1}\right|} p(\x_i|c_1) q(\z_1|c_1)} \right)  \right]\right]
\end{align}
Now observe that $\z_1,...,\z_{\left|I_{c_1}\right|}$ are equal in distribution, so we can change the sampling distribution to be over 
\begin{equation}
    \prod_{i=1}^{\left|I_{c_1}\right|} p(\x_i|c_1) \frac{1}{|I_{c_1}|}\sum_{j\in I_{c_1}} q(\z_1|\x_j,c_i)
\end{equation}
which amounts to choosing at random which of the $\x_1,...,\x_{\left|I_{c_1}\right|}$ to sample $\z_1$ from. Finally, we observe that 
\begin{equation}
    \prod_{i=1}^{\left|I_{c_1}\right|} p(\x_i|c_1) q(\z_1|c_1)
\end{equation}
is a normalised distribution over $\x_1,...,\x_{\left|I_{c_1}\right|},\z_1$.
Thus we can write $\Delta$ as the following expected KL divergence
\begin{align}
    \Delta = \E_{c_1,\left|I_{c_1}\right|}\left[\KL \left[\prod_{i=1}^{\left|I_{c_1}\right|} p(\x_i|c_1) \frac{1}{|I_{c_1}|}\sum_{j\in I_{c_1}} q(\z_1|\x_j,c_i) \middle \| \prod_{i=1}^{\left|I_{c_1}\right|} p(\x_i|c_1) q(\z_1|c_1) \right] \right].
\end{align}
Since the KL divergence is non-negative, we have shown that $\Delta \ge 0$. Therefore
\begin{align}
\E_{\prod_{i=1}^B p(\x_i,c_i)q(\z_i|\x_i,c_i)}\left[\log\left(\frac{1}{|I_{c_1}|}\sum_{j\in I_{c_1}} q(\z_1|\x_j,c_i) \right)  \right] \ge  \E_{ p(\x_1,c_1)q(\z_1|\x_1,c_1)}\left[\log q(\z_1|c_1)  \right].
\end{align}

Putting these two results together, we have shown that
\begin{align}
    &\E_{\prod_{i=1}^B p(\x_i,c_i)q(\z_i|\x_i,c_i)}\left[\frac{1}{B}\sum_{i=1}^B  \log\left(\frac{1}{|I_{c_i}|}\sum_{j\in I_{c_i}} q(\z_i|\x_j,c_i) \right) - \log\left(\frac{1}{|I_{\neg c_i}|} \sum_{j\in I_{\neg c_i}} q(\z_i|\x_j, c_j) \right) \right] \\ & \qquad  \ge \E_{\prod_{i=1}^B p(\x_i,c_i)q(\z_i|\x_i,c_i)}\left[\log\frac{q(\z_1|c_1)}{q(\z_1|\neg\, c_1)} \right] \\
    &\qquad = \E_{ p(\x_1,c_1)q(\z_1|\x_1,c_1)}\left[\log\frac{q(\z_1|c_1)}{q(\z_1|\neg\, c_1)} \right] \\
    &\qquad = \sum_{c\in\mathcal{C}} p(c_1)\E_{ p(\x_1|c_1)q(\z_1|\x_1,c_1)}\left[\log\frac{q(\z_1|c_1)}{q(\z_1|\neg\, c_1)} \right] \\
    &\qquad = \sum_{c\in\mathcal{C}} p(c_1)\E_{ q(\z_1|c_1)}\left[\log\frac{q(\z_1|c_1)}{q(\z_1|\neg\, c_1)} \right] \\
    &\qquad = \sum_{c\in\mathcal{C}} p(c)\KL[q(\z|c)\|q(\z|\neg\,c)].
\end{align}
Finally, as $B\to\infty$, the Strong Law of Large Numbers implies that 
\begin{align}
    \log\left(\frac{1}{|I_{c_1}|}\sum_{j\in I_{c_1}} q(\z_1|\x_j,c_i) \right) &\to q(\z_i|c_i) \text{ a.s.}, \\
    \log\left(\frac{1}{|I_{\neg c_i}|} \sum_{j\in I_{\neg c_i}} q(\z_i|\x_j, c_j) \right) &\to 
    q(\z_i|\neg\, c_i) \text{ a.s.},
\end{align}
so (under mild technical assumptions) we conclude that the bound becomes tight in this limit. This completes the proof.
\end{proof}

\subsection{Combining with the $\beta$-VAE}
Note that the CoMP penalty may also be combined with the $\beta$-VAE objective \citep{higgins2016beta}, giving the overall objective for a training batch
\begin{align}\begin{split}
    \label{eq:total_training_obj_beta}
    \mathcal{L}^\text{CoMP}_{B,\beta} = \frac{1}{B}\sum_{i=1}^B &\left[\log p_\theta(\x_i|\z_i,c_i) + \beta \log\frac{p(\z_i)}{q_\phi(\z_i|\x_i,c_i)} -\gamma \log\left(\frac{\frac{1}{|I_{c_i}|}\sum_{j\in I_{c_i}} q_\phi(\z_i|\x_j,c_i)}{\frac{1}{|I_{\neg c_i}|} \sum_{j\in I_{\neg c_i}} q_\phi(\z_i|\x_j, c_j)} \right)\right].
\end{split}\end{align}

\section{Analysing CoMP gradients}
\label{sec:app:gradients}
We attempt to understand how the CoMP penalty differs from existing penalties in the literature.
Specifically, we compare CoMP using a Gaussian posterior family with MMD using a Radial Basis Kernel \cite{vert2004primer}. 
To analyse MMD and CoMP gradients, we focus on the two specific cases that highlight the similarities between these methods, revealing the remaining differences.
Specifically, we consider MMD with a simple unnormalised Radial Basis Kernel \cite{vert2004primer}
\begin{equation}
    k(\z,\z') = e^{-\|\z - \z'\|^2},
\end{equation}
and a Gaussian variational posterior family with fixed covariance matrix $\tfrac{1}{2}I$
\begin{equation}
    q(\z|\x,c) \propto e^{-\|\z - \bm{\mu}_\z(\x,c)\|^2}.
\end{equation}
We also assume just two conditions $|\mathcal{C}|=2$. We show that both methods can be interpreted as applying a penalty to each element $\z_i,c_i$ of the training batch.
Full derivations are given in the following section.
For the MMD penalty for $\z_i,c_i$, the gradient takes the form
\begin{equation}
    \nabla_{\z_i} \mathcal{P}_\text{MMD}(\z_i,c_i) = \frac{2}{\left|I_{c_i}\right|^2} \sum_{j\in I_{c_i}} e^{-\|\z_i - \z_j\|^2} (\z_j - \z_i) - \frac{4}{\left|I_{\neg c_i}\right|\left|I_{c_i}\right|} \sum_{j\in I_{\neg c_i}} e^{-\|\z_i - \z_j\|^2} (\z_j - \z_i),
\end{equation}
whilst the CoMP penalty gradient takes the form
\begin{equation}
    \nabla_{\z_i} \mathcal{P}_\text{CoMP}(\z_i,c_i) = \frac{2\sum_{j\in I_{c_i}} e^{-\|\z_i - \bm{\mu}_{\z_j}\|^2} (\bm{\mu}_{\z_j} - \z_i)}{B\sum_{j\in I_{c_i}} e^{-\|\z_i - \bm{\mu}_{\z_j}\|^2}} - \frac{2\sum_{j\in I_{\neg c_i}} e^{-\|\z_i - \bm{\mu}_{\z_j}\|^2} (\bm{\mu}_{\z_j} - \z_i)}{B\sum_{j\in I_{\neg c_i}} e^{-\|\z_i - \bm{\mu}_{\z_j}\|^2}}
\end{equation}
where $\bm{\mu}_{\z_j}$ is the variational mean for $\z_j$.
One important feature of the MMD gradients is that, if $\|\z_i-\z_j\|^2$ is large for all $j\ne i$, for instance when the point $\z_i$ is part of an isolated cluster, then the gradient to update the representation $\z_i$ will be small.
So if $\z_i$ is already very isolated from the distribution $q(\z|\neg\, c_i)$, then the gradients bringing it closer to points with condition $\neg \, c_i$ will be small.
In comparison to the MMD gradient, it can be seen that gradients for CoMP are \emph{self-normalised}.
This means that the gradient through $\z_i$ will be large, even when $\z_i$ is very far away from any points with condition $\neg \, c_i$. This, in turn, suggests that that CoMP is likely to be preferable to MMD when we have a number of isolated clusters or interesting global structure in latent space, something which often occurs with biological data.
The CoMP approach also bears a resemblance to nearest-neighbour approaches \cite{li2018implicit}. Indeed, for a Gaussian posterior as $\sigma\to 0$, the $\neg\, c_i$ term of the gradient places all its weight on the nearest element of the batch under condition $\neg\, c_i$.

\subsection{Derivations}
\label{sec:app:derivations}
For an MMD penalty, the simplest form of the Kernel Two-sample Test statistic \cite{gretton2012kernel} with batch size $B$ can be written as follows
\begin{align}
    \mathcal{P}_\text{MMD} &= \sum_{i=1}^B \mathcal{P}_\text{MMD}(\z_i,c_i)  \\
    &= \sum_{i=1}^B \left( \frac{1}{\left|I_{c_i}\right|^2} \sum_{j\in I_{c_i}} e^{-\|\z_i - \z_j\|^2} - \frac{1}{\left|I_{\neg c_i}\right|\left|I_{c_i}\right|} \sum_{j\in I_{\neg c_i}} e^{-\|\z_i - \z_j\|^2} \right),
\end{align}
taking gradients with respect to $\z_i$ gives us
\begin{equation}
    \nabla_{\z_i} \mathcal{P}_\text{MMD}(\z_i,c_i) = \frac{2}{\left|I_{c_i}\right|^2} \sum_{j\in I_{c_i}} e^{-\|\z_i - \z_j\|^2} (\z_j - \z_i) - \frac{2}{\left|I_{\neg c_i}\right|\left|I_{c_i}\right|} \sum_{j\in I_{\neg c_i}} e^{-\|\z_i - \z_j\|^2} (\z_j - \z_i),
\end{equation}
the gradients of the total penalty are
\begin{equation}
    \nabla_{\z_i} \mathcal{P}_\text{MMD} = \frac{4}{\left|I_{c_i}\right|^2} \sum_{j\in I_{c_i}} e^{-\|\z_i - \z_j\|^2} (\z_j - \z_i) - \frac{4}{\left|I_{\neg c_i}\right|\left|I_{c_i}\right|} \sum_{j\in I_{\neg c_i}} e^{-\|\z_i - \z_j\|^2} (\z_j - \z_i).
\end{equation}
The CoMP penalty (ignoring normalising constants) is
\begin{align}
    \mathcal{P}_\text{CoMP} &= \sum_{i=1}^B \mathcal{P}_\text{CoMP}(\z_i,c_i)  \\
    &= \frac{1}{B} \sum_{i=1}^B \log\left( \frac{1}{\left|I_{c_i}\right|} \sum_{j\in I_{c_i}} e^{-\|\z_i - \bm{\mu}_{\z_j}\|^2} \right) - \log \left(\frac{1}{\left|I_{\neg c_i}\right|} \sum_{j\in I_{\neg c_i}} e^{-\|\z_i - \bm{\mu}_{\z_j}\|^2} \right),
\end{align}
if we take the gradient with respect to $\z_i$ we obtain
\begin{equation}
    \nabla_{\z_i} \mathcal{P}_\text{CoMP}(\z_i,c_i) = \frac{2\sum_{j\in I_{c_i}} e^{-\|\z_i - \bm{\mu}_{\z_j}\|^2} (\bm{\mu}_{\z_j} - \z_i)}{B\sum_{j\in I_{c_i}} e^{-\|\z_i - \bm{\mu}_{\z_j}\|^2}} - \frac{2\sum_{j\in I_{\neg c_i}} e^{-\|\z_i - \bm{\mu}_{\z_j}\|^2} (\bm{\mu}_{\z_j} - \z_i)}{B\sum_{j\in I_{\neg c_i}} e^{-\|\z_i - \bm{\mu}_{\z_j}\|^2}}
\end{equation}
where $\bm{\mu}_{\z_j} = \bm{\mu}(\x_i,c_i)$ is the variational mean for $\z_j$.
The gradient of the full penalty with respect to $\bm{\mu}_{\z_i}$, noting $\z_i = \bm{\mu}_{\z_i} + \bm{\epsilon}_i$, is
\begin{equation}
\begin{split}
    \nabla_{\bm{\mu}_{\z_i}} \mathcal{P}_\text{CoMP} =& \frac{2\sum_{j\in I_{c_i}} e^{-\|\z_i - \bm{\mu}_{\z_j}\|^2} (\bm{\mu}_{\z_j} - \z_i)}{B\sum_{j\in I_{c_i}} e^{-\|\z_i - \bm{\mu}_{\z_j}\|^2}} - \frac{2\sum_{j\in I_{\neg c_i}} e^{-\|\z_i - \bm{\mu}_{\z_j}\|^2} (\bm{\mu}_{\z_j} - \z_i)}{B\sum_{j\in I_{\neg c_i}} e^{-\|\z_i - \bm{\mu}_{\z_j}\|^2}} \\
    &+2\sum_{j\in I_{c_i}}\frac{ e^{-\|\z_j - \bm{\mu}_{\z_i}\|^2} ({\z_j} - \bm{\mu}_{\z_i})}{B\sum_{k\in I_{c_i}} e^{-\|\z_j - \bm{\mu}_{\z_k}\|^2}} - 2\sum_{j\in I_{\neg c_i}}\frac{ e^{-\|\z_j - \bm{\mu}_{\z_i}\|^2} ({\z_j} - \bm{\mu}_{\z_i})}{B\sum_{k\in I_{\neg c_j}} e^{-\|\z_j - \bm{\mu}_{\z_k}\|^2}}.
\end{split}
\end{equation}

Finally, to see the connection with nearest neighbour methods, we repeat this analysis with Gaussian posterior with fixed variance $\sigma^2$. The gradient term is then
\begin{equation}
\begin{split}
    \nabla_{\z_i} \mathcal{P}_\text{CoMP}(\z_i,c_i) &= \frac{\frac{1}{\sigma^2}\sum_{j\in I_{c_i}} e^{-\frac{1}{2\sigma^2}\|\z_i - \bm{\mu}_{\z_j}\|^2} (\bm{\mu}_{\z_j} - \z_i)}{B\sum_{j\in I_{c_i}} e^{-\frac{1}{2\sigma^2}\|\z_i - \bm{\mu}_{\z_j}\|^2}} \\
    &\qquad - \frac{\frac{1}{\sigma^2}\sum_{j\in I_{\neg c_i}} e^{-\frac{1}{2\sigma^2}\|\z_i - \bm{\mu}_{\z_j}\|^2} (\bm{\mu}_{\z_j} - \z_i)}{B\sum_{j\in I_{\neg c_i}} e^{-\frac{1}{2\sigma^2}\|\z_i - \bm{\mu}_{\z_j}\|^2}}.
\end{split}
\end{equation}
As $\sigma \to 0$, we have
\begin{equation}
    \frac{e^{-\frac{1}{2\sigma^2}\|\z_i - \bm{\mu}_{\z_k}\|^2}}{\sum_{j\in I_{\neg c_i}} e^{-\frac{1}{2\sigma^2}\|\z_i - \bm{\mu}_{\z_j}\|^2}} \to \delta_{k \text{nn}_i}
\end{equation}
where $\text{nn}_i$ is the index of the nearest neighbour to $\z_i$ among the set $\{\z_j: j \in I_{\neg c_i}\}$, i.e.
\begin{equation}
    \text{nn}_i =  \text{argmin}_{j\in I_{\neg c_i}} \|\z_i - \z_j\|
\end{equation}
indicating that the gradient between $\z_i$ and $\z_{\text{nn}_i}$ becomes the dominant term in the limit.

\section{Theory}
\label{sec:app:theory}
We have introduced CoMP, a new approach to enforcing the constraint $\z \indep c$ in a CVAE model. 
However, it remains to be shown that training a CVAE with CoMP on unpaired data, i.e. on independent samples $(\x_i,c_i)_{i=1}^\infty$ from the data generating process, leads to valid counterfactual predictions.

In the language of the Causal Hierarchy Theorem \citep{pearl2000models,bareinboim2020pearl}, our unpaired data may be treated as interventional data, at level 2 of the hierarchy. (This is accurate in biological applications in which data under non-control conditions may be produced by actively intervening on samples with a drug, gene knock-out, etc. It is also formally correct in Fig.~\ref{fig:pgm} in which $c$ has no parents.)
Counterfactuals exist on level 3 of the hierarchy, and therefore additional assumptions or constraints are necessary to infer them from our data.
To illustrate the problem of \emph{counterfactual identifiability}, we begin with a negative result for counterfactual inference with the CVAE framework.

\begin{proposition}
\label{prop:negative}
    Consider the following generative CVAE models for data $\x,c$ with two conditions $c\in\{0, 1\}$.
    \begin{enumerate}
        \item $\z \sim N(0, I) \qquad \x = A_c\z +\bm{\varepsilon}$,
        \item $\z \sim N(0, I)\qquad \x = \begin{cases} A_0\z+\bm{\varepsilon} \text{ if }c=0 \\ A_1R\z+\bm{\varepsilon} \text{ if }c=1, \end{cases}$
    \end{enumerate}
    where $A_0,A_1$ are matrices, $R$ is a non-trivial rotation matrix and $\bm{\varepsilon}$ is Gaussian observation noise independent of $\z$. Assume $A_1R \ne A_1$.
    Then models 1. and 2. fit unpaired training data $(\x_i,c_i)_{i=1}^\infty$ equally well, but give different counterfactual predictions.
\end{proposition}
\begin{proof}
    For the two models to fit the training data equally well, we need class-conditional likelihoods $p(\x|c)$ to be equal for each $c$. For $c=0$, the models are equal. For $c=1$, we use the fact that $R\z \overset{d}{=}\z$ since $\z$ is an isotropic Gaussian. This gives
    \begin{equation}
        p_\text{model 2.}(\x|c=1) = \E_{\z \sim N(0,I)}[p_{A_1}(\x|R\z)] = \E_{\z \sim N(0,I)}[p_{A_1}(\x|\z)] = p_\text{model 1.}(\x|c=1)
    \end{equation}
    where $p_A(\x|\z) = p_{\bm{\varepsilon}}(\x-A\z)$. 
    Thus, both models fit the training data equally well.
    To see that the models give different counterfactual predictions, we have
    \begin{align}
        \hat{p}_\text{model 1.}(\x_{c=1}|\x,c=0) &= \int p_\text{model 1.}(\z|\x,c=0) p_{A_1}(\x|\z)\ d\z \\
        \begin{split}
        \hat{p}_\text{model 2.}(\x_{c=1}|\x,c=0) &= \int p_\text{model 2.}(\z|\x,c=0) p_{A_1}(\x|R\z)\ d\z  \\
        &= \int p_\text{model 1.}(\z|\x,c=0) p_{A_1}(\x|R\z)\ d\z .
        \end{split}
    \end{align}
    Since $A_1R \ne A_1$, we see that there exist $\x$ for which $\hat{p}_\text{model 1.}(\x_{c=1}|\x,c=0) \ne \hat{p}_\text{model 2.}(\x_{c=1}|\x,c=0)$. Thus, counterfactual predictions under the two models are different.
\end{proof}

One way to understand Proposition~\ref{prop:negative} is that the $N(0,I)$ prior has a high degree of symmetry, which means there is a high degree of indeterminacy in $\z$. The relevance of the rotational symmetry of the $N(0,I)$ distribution to VAE models has previously been discussed by \citet{mathieu2019disentangling,rolinek2019variational}, here we focus on the CVAE and the consequences for counterfactual inference.

This intuition also indicates why Proposition~\ref{prop:negative} is not the end of the story. In our practical applications in Sec.~\ref{sec:experiments}, we find that the latent space distribution $q(\z)$ is not, in fact, an isotropic Gaussian. Note that, in this case, the example in the Proposition breaks down. If we consider model 2. with decoder $A_1R$ and $\z \sim q(\z)$, then the predictive distribution $p_\text{model 2.}(\x|c=1)$ will be different from $p_\text{model 1.}(\x|c=1)$.
If instead we consider a different latent distribution $R^{-1} \circ q(\z)$ which arises when we apply $R^{-1}$ to $q(\z)$, then $p_\text{model 2.}(\x|c=1)$ will be correct \emph{but} we will no longer have $\z \indep c$.

Experimentally, we also find that counterfactual inference is stable between different training seeds, and that it accords extremely well with counterfactuals that are estimated using held-out cell type information.
This indicates that there is a mismatch between the basic theory and application.
Our approach to theoretically analyse CoMP is therefore inspired by our practical observation that $q(\z)$ is rarely equal to an isotropic Gaussian in practice. 
To make the theory more realistic, we relax the assumption that, in the true data generating distribution, the latent distribution is $N(0, I)$.

\subsection{Counterfactual identifiability for CoMP and other CVAE methods}
We now assume that the true data generating distribution has $\z \sim r(\z)$, a non-Gaussian distribution. The assumption of non-Gaussianity permits us to make the connection to the theory of independent component analysis (ICA) \citep{hyvarinen2001independent}, in which the non-Gaussianity assumption is of critical importance. ICA has successfully been applied to causal discovery \citep{shimizu2006linear} and inference \citep{moneta2013causal}.
To simplify the exposition, we initially focus on a generative model with a noiseless linear decoder as follows
\begin{equation}
    \z \sim r(\z) \qquad \x = A_c\z
\end{equation}
and we further make the ICA assumption that
\begin{equation}
    \z = B\s
\end{equation}
where the components $s_i$ of $\s$ are independent and non-Gaussian of unit variance.

We begin by translating the standard linear ICA theory into our setting. The standard ICA setting applies to data from a single condition. For condition $c=0$, we have $\x=A_0B\s$. The theory of ICA tells us that $\s$ and $A_0B$ are likelihood identifiable, up to permutation and negation.
\begin{theorem}[\citet{comon1994independent}]
\label{thm:comon}
Assume $\x = D\s$ where $s_i$ are mutually independent and non-Gaussian with variance 1. Then $D$ is likelihood identifiable up to the following
\begin{itemize}
     \item we can apply a \emph{permutation} to the columns of $D$ by right multiplication by a permutation matrix $P$,
     \item we can \emph{negate} columns of $D$ by right multiplication by a matrix $N$ with diagonals entries either $1$ or $-1$ and $0$ elsewhere.
\end{itemize}
\end{theorem}
It is clear why these two indeterminacies exist: given $\x = D\s$ we also have $\x = DPP^{-1}\s$ and $\x = DNN^{-1}\s$, and $P^{-1}\s$ and $N^{-1}\s$ satisfy all the conditions of the theorem.

Our existing assumptions are almost enough to guarantee counterfactual identifiability. 
To gain intuition, suppose that $B=I$ and that $A_0,A_1$ are fully identifiable (setting aside the indeterminacies for a moment). Then we would have
\begin{equation}
    \x_{c=1}|(\x,c=0) = A_1A_0^{-1}\x.
\end{equation}
Although $A_0,A_1$ are not fully identifiable, we cannot apply arbitrary permutations and negations to $A_0$ and $A_1$ separately and still have the condition $\z \indep c$ hold. Suppose we make the substitutions
\begin{equation}
    A_0 \mapsto A_0P_0 \qquad A_1 \mapsto A_1P_1
\end{equation}
then the implied distributions for $\z$ must also be changed as
\begin{equation}
    r(\z|c=0) \mapsto P_0^{-1}\circ r(\z|c=0) \qquad r(\z|c=1) \mapsto P_1^{-1}\circ r(\z|c=1).
\end{equation}
If $P_0=P_1$, then the distributions remain equal, but the counterfactual predictions are also unchanged since
\begin{equation}
    (A_1P_1)(A_0P_0)^{-1} = A_1P_1P_0^{-1}A_0^{-1} = A_1A_0^{-1}.
\end{equation}
On the other hand, if $P_0 \ne P_1$ then applying different permutations would lead to two different distributions for $r(\z|c=0)$ and $r(\z|c=1)$, violating $\z \indep c$. There is one more case to be eliminated. That is the case where $P_0$ and $P_1$ are different, but due to a symmetry of the distribution $r(\z)$, the condition $\z \indep c$ still holds. The following assumption excludes this possibility and formalises our assumptions for $r(\z)$.

\begin{assumption}
\label{assume:s}
Assume that $\z \sim r(\z)$ can be expressed as $\z = B\s$ where the components $s_i$ of $\s$ are independent and non-Gaussian of unit variance. Furthermore, assume that for every permutation matrix $P$ and negation matrix $N$ such that $PN$ is not the identity we have
\begin{equation}
    PN\s \overset{d}{\ne} \s.
\end{equation}
\end{assumption}
This assumption is enough to prove the following theorem.

\counterfactualidentifiability*
\begin{proof}
    We have the model $\z \sim r(\z)$ and $\x = A_c\z$.
    The true counterfactual predictions are 
    \begin{equation}
    \x_{c=1}|(\x,c=0) = A_1A_0^{-1}\x,
    \end{equation}
    hence it is enough to show that $A_1A_0^{-1}$ is likelihood identifiable.
    Since $\x|c = A_c\z = A_cB\s = D_c\s$ and 
    \begin{equation}
        D_1D_0^{-1} = (A_1B)(A_0B)^{-1} = A_1A_0^{-1}
    \end{equation}
    it is enough to show that $D_1D_0^{-1}$ is likelihood identifiable.
    By Theorem~\ref{thm:comon}, $D_c$ are identifiable from unpaired data up to permutation and negation.
    Suppose we are able to identify $D_cP_cN_c$. (Note: permutation and negation matrices together form a group, so $P_cN_c$ is a general composition of permutations and negations.)
    Since $\s \indep c$ in the identified model, we must have
    \begin{equation}
        (P_0N_0)^{-1}\s \overset{d}{=} (P_1N_1)^{-1}\s
    \end{equation}
    which further implies
    \begin{equation}
        (P_1N_1)(P_0N_0)^{-1}\s \overset{d}{=} \s.
    \end{equation}
    It is possible to express $(P_1N_1)(P_0N_0)^{-1}$ as a product $PN$ of a permutation and a negation. Applying Assumption~\ref{assume:s}, we must have 
    \begin{equation}
        (P_1N_1)(P_0N_0)^{-1} = I.
    \end{equation}
    Hence
    \begin{equation}
        D_1P_1N_1(D_0P_0N_0)^{-1} = D_1(P_1N_1)(P_0N_0)^{-1}D_0^{-1} = D_1ID_0^{-1} = D_1D_0^{-1}.
    \end{equation}
    Thus $A_1A_0^{-1}$ is identifiable.
\end{proof}

\paragraph{It is unnecessary to compute the independent components}
An important aspect of the Theorem is that it is \emph{not} necessary to compute the independent components $\s$ and the matrices $D_c$. For the problem of counterfactual identifiability, it is sufficient that such matrices exist. For computation, we can rely on $A_0$ and $A_1$. 

\paragraph{Extensions of the proof}
Our proof directly extends to the case for more than two conditions.
It is possible to adapt the proof by considering the existing theory for Noisy ICA \citep{hyvarinen1999fast} and nonlinear ICA \citep{hyvarinen2001independent}. The latter has been explored in the context of VAE models by \citet{khemakhem2020variational}, and presents interesting challenges compared to the linear theory.

Theorem~\ref{thm:counterfactualidentifiability} is a key theorem because it provides circumstances under which a CVAE type model can identify counterfactuals. The requirement $\z \indep c$ can be seen as a significant weakening of the assumption $\z \sim N(0, I)$, whilst Assumption~\ref{assume:s} introduces necessary constraints on $\z$ for identifiability to hold.
 
\subsection{Consistency of CoMP under prior misspecification}
We have seen that CVAE-type models, including CoMP, can identify counterfactuals, but we now assume that $\z$ be of some unknown non-Gaussian distribution rather than an isotropic Gaussian.
This is at odds with the CoMP training objective, which uses a Gaussian prior.
Rather than altering the objective and learning a non-Gaussian prior (which is one valid approach), we instead rely on our experience, which shows that the CoMP training objective does not lead to $q_\phi(\z)$ being a Gaussian.

We will show that, under certain conditions, training with the unaltered CoMP objective when the true data generating distribution has $\z \sim r(\z)$ allows us to find the correct decoder network, and further the marginal distribution of the encoder $q_\phi(\z)$ matches $r(\z)$.

\consistency*
\begin{proof}
We begin by writing out the CoMP training objective with a data batch of size $n$
\begin{align}\begin{split}
    \label{eq:total_training_obj_theory}
    \mathcal{L}^\text{CoMP}_n = \frac{1}{n}\sum_{i=1}^n &\left[ \log\frac{p_\theta(\x_i|\z_i,c_i)p(\z_i)}{q_\phi(\z_i|\x_i,c_i)} -\gamma \log\left(\frac{\frac{1}{|I_{c_i}|}\sum_{j\in I_{c_i}} q_\phi(\z_i|\x_j,c_i)}{\frac{1}{|I_{\neg c_i}|} \sum_{j\in I_{\neg c_i}} q_\phi(\z_i|\x_j, c_j)} \right)\right].
\end{split}\end{align}
Suppose that the true data generating distribution is
\begin{equation}
    \z \sim r(\z) \qquad \x|c \sim p_{\theta^*}(\x|\z,c).
\end{equation}
We define
\begin{equation}
    p_{r,\theta}(\x|c) = \E_{r(\z)}[p_\theta(\x|\z,c)]
\end{equation}
and
\begin{equation}
    r_{\theta}(\z|\x,c) = \frac{r(\z)p_\theta(\x|\z,c)}{p_{r,\theta}(\x|c)}.
\end{equation}
Then we can rewrite the ELBO part of the objective as
\begin{align}
    \frac{1}{n}\sum_{i=1}^n\log\frac{p_\theta(\x_i|\z_i,c_i)p(\z_i)}{q_\phi(\z_i|\x_i,c_i)}     &= \frac{1}{n}\sum_{i=1}^n\log\frac{p_\theta(\x_i|\z_i,c_i)p(\z_i)r(\z_i)p_{r,\theta}(\x_i|c_i)}{q_\phi(\z_i|\x_i,c_i)r(\z_i)p_{r,\theta}(\x_i|c_i)}\\
    &= \frac{1}{n}\sum_{i=1}^n \log p_{r,\theta}(\x_i|c_i) - \log \frac{q_\phi(\z_i|\x_i,c_i)p_{r,\theta}(\x_i,c_i)}{r(\z_i)p_\theta(\x_i|\z_i,c_i)}-\log\frac{r(\z_i)}{p(\z_i)}\\
    \intertext{take the expectation over $\z_i \sim q_\phi(\z|\x_i,c_i)$}
    &= \frac{1}{n}\sum_{i=1}^n \log p_{r,\theta}(\x_i|c_i) - \KL[q_\phi(\z_i|\x_i,c_i)\|r_{\theta}(\z_i|\x_i,c_i)] - \E_{q_\phi(\z_i|\x_i,c_i)}\left[\log \frac{r(\z_i)}{p(\z_i)} \right].
\end{align}
Consider the augmented objective in which we introduce a coefficient $\zeta$ on the final term, and we reintroduce the CoMP misalignment penalty
\begin{align}
\begin{split}
    \E[\mathcal{L}_n^\text{CoMP}(\theta,\phi)] =& \frac{1}{n}\sum_{i=1}^n \log p_{r,\theta}(\x_i|c_i) - \KL[q_\phi(\z_i|\x_i,c_i)\|r_{\theta}(\z_i|\x_i,c_i)] \\
    &- \zeta\frac{1}{n}\sum_{i=1}^n\E_{q_\phi(\z_i|\x_i,c_i)}\left[\log \frac{r(\z_i)}{p(\z_i)} \right]\\
    &-\gamma \frac{1}{n}\sum_{i=1}^n\E_{q_\phi(\z_i|\x_i,c_i)}\left[\log\left(\frac{\frac{1}{|I_{c_i}|}\sum_{j\in I_{c_i}} q_\phi(\z_i|\x_j,c_i)}{\frac{1}{|I_{\neg c_i}|} \sum_{j\in I_{\neg c_i}} q_\phi(\z_i|\x_j, c_j)} \right)\right].
\end{split}
\end{align}
Treating $\zeta,\gamma$ as Lagrange multipliers, this is equivalent to the following constrainted optimisation problem
\begin{align}
    \label{eq:constrained}
    \max_{\theta,\phi}\ &  \frac{1}{n}\sum_{i=1}^n \log p_{r,\theta}(\x_i|c_i) - \KL[q_\phi(\z_i|\x_i,c_i)\|r_{\theta}(\z_i|\x_i,c_i)] \\
    \text{subject to }& \frac{1}{n}\sum_{i=1}^n\E_{q_\phi(\z_i|\x_i,c_i)}\left[\log \frac{r(\z_i)}{p(\z_i)} \right] \le K_\zeta \text{ and}\\
    &\frac{1}{n}\sum_{i=1}^n\E_{q_\phi(\z_i|\x_i,c_i)}\left[\log\left(\frac{\frac{1}{|I_{c_i}|}\sum_{j\in I_{c_i}} q_\phi(\z_i|\x_j,c_i)}{\frac{1}{|I_{\neg c_i}|} \sum_{j\in I_{\neg c_i}} q_\phi(\z_i|\x_j, c_j)} \right)\right] \le L_\gamma.
\end{align}
Neglecting the constraints for a moment, if we maximise the main objective with respect to $\phi$ with a sufficiently expressive encoder, we will recover $q_\phi(\z|\x,c) = r_\theta(\z|\x,c)$ in the limit $n\to\infty$.
We are then left with the maximum likelihood problem
\begin{equation}
\max_{\theta}\  \frac{1}{n}\sum_{i=1}^n \log p_{r,\theta}(\x_i|c_i),
\end{equation}
thus $\theta$ will recover a likelihood-maximising decoder under the \emph{misspecified} prior $r(\z)$.

This solution will be valid if the constraints are satisfied at that solution.
Note that in the limit $n\to\infty$, Theorem~\ref{thm:kl} shows that the CoMP misalignment penalty becomes 
\begin{equation}
    \sum_{c\in\mathcal{C}} p(c) \KL\left[q_\phi(\z|c) || q_\phi(\z|\neg\,c) \right].
\end{equation}
Since $q_\phi(\z|\x,c) = r_\theta(\z|\x,c)$, we have $q_\phi(\z|c) = r(\z)$ for every $c$, hence the misalignment penalty is equal to 0 and the constraint is satisfied.
In this limit, we also have
\begin{equation}
    \frac{1}{n}\sum_{i=1}^n\E_{q_\phi(\z_i|\x_i,c_i)}\left[\log \frac{r(\z_i)}{p(\z_i)} \right] \to \KL[r(\z)\|p(\z)].
\end{equation}
Therefore, returning to the original problem with $\zeta=1$, provided that $\KL[r(\z)\|p(\z)] \le K_1$, then the constraints are satisfied under unconstrained optimisation in \eqref{eq:constrained}, meaning that $\theta$ tends to a maximum likelihood solution of $\E_{\x,c}[\log p_{r,\theta}(\x|c)]$.
\end{proof}

The result of Theorem~\ref{thm:consistency} is directly connected to our previous discussion about ICA, since ICA is a solution to the maximum likelihood problem \citep{hyvarinen2001independent}.
However, the consistency theorem is much more general, since it applies to a much wider range of models, and does not assume a linear decoder.
Theorem~\ref{thm:consistency} also justifies our distinction in Theorem~\ref{thm:counterfactualidentifiability} between $\z$ and $\s$. In trying to satisfy the condition $\KL[r(\z)\|p(\z)] \le K_1$, we are allowed to apply any linear transformation to $\z$. If any linear transformation of $\z$ satisfies $\KL[r(\z)\|p(\z)] \le K_1$, then consistency holds true.

\section{Experimental details}
\label{sec:experiment_details}

The code to reproduce our experiments is available at \url{https://github.com/BenevolentAI/CoMP}.

\subsection{Dataset details and data processing}

\paragraph{\tcgaccle} This dataset, as used in the experiments in \textit{Cellinger} \cite{Celligner}\footnote{\url{www.nature.com/articles/s41467-020-20294-x\#data-availability}}, consists of bulk expression profiles for tumours ($n=12,236$) and cancer cell-lines ($n=1,249$) across 39 different cancer types.
The tumour samples are taken from The Cancer Genome Atlas (TCGA) \cite{tcga}\footnote{\url{www.cancer.gov/about-nci/organization/ccg/research/structural-genomics/tcga}} and Therapeutically Applicable Research To Generate Effective Treatments (TARGET) \cite{target}\footnote{\url{ocg.cancer.gov/programs/target}} and were compiled by the Treehouse Childhood Cancer Initiative at the UC Santa Cruz Genomics Institute \cite{goldman2020visualizing}\footnote{\url{https://treehousegenomics.soe.ucsc.edu/public-data/previous-compendia.html\#tumor_v10_polyA}}.
The cell lines are from the Cancer Cell Line Encyclopedia (CCLE) \cite{ccle}\footnote{\url{portals.broadinstitute.org/ccle}}.
The condition variable is the tumour / cell line label. The expression data is restricted to the intersecting subset of 16,612 protein-coding genes and are TPM $\log_2$-transformed values.

In our experiments, as is common practice in omics data analysis (e.g. \citet{Celligner}), we pre-process the data by filtering out low-variance genes. Here, we select the 8,000 highest variance genes across cell-lines and tumours separately and take the union to give a final feature set of 9,468 genes.

For our calculation of $\mkbet_{k, \alpha}$ and $\tilde{s}_{k,c}$ metrics we only include cancer types with at least 400 samples (i.e. $4\times k$ for our choice of $k=100$) to ensure that the metric retains the ability to evaluate \emph{local} mixing. 15 cancer types pass this threshold, representing 82\% of all samples.

\paragraph{Stimulated / untreated single-cell PBMCs}
This dataset consists of single-cell expression profiles of 14,053 genes for peripheral blood mononuclear cells (PBMCs), various immune cell types pooled from eight lupus patient samples. 7,217 of the cells were stimulated with interferon (IFN)-$\beta$  while 6,359 were left untreated (control) \citep{kang2018multiplexed}. This dataset has been used in \citet{lotfollahi2019scgen} and \citet{lotfollahi2019conditional} previously. We obtained an annotated and pre-filtered dataset from \citet{lotfollahi2019conditional}\footnote{\url{https://github.com/theislab/trVAE\_reproducibility}} \footnote{\url{https://drive.google.com/drive/folders/1n1SLbXha4OH7j7zZ0zZAxrj_-2kczgl8}, filename: \texttt{kang\_count.h5ad}}
, which includes metadata on immune cell type labels along the condition label; stimulated or control.

The file was read into scanpy \cite{scanpy} and pre-processed using \texttt{sc.pp.normalize\_total(data, inplace=True)}, which normalises the data such that each cell has a total count equal to the median total count across all cells. The normalised counts were then $\log (x+1)$ transformed using the  scanpy function, \texttt{sc.pp.log1p(data)}. We selected the top 2,000 most variable genes using the scanpy function, \texttt{sc.pp.highly\_variable\_genes(data, flavor="seurat", n\_top\_genes=2000)}.

We obtained the top 50 differentially expressed (DE) genes between stimulated and control cells for each cell type by subsetting the data for each cell type and using scanpy's function \texttt{sc.tl.rank\_genes\_groups(cell\_type\_data, groupby="stim", n\_genes=50, method="wilcoxon")}, which ranks genes based on a Wilcoxon rank-sum test. For each cell type, we separated the top 50 DE genes into those that were up-regulated and down-regulated by IFN-$\beta$ stimulation.

\paragraph{Single-cell RNA-seq data integration}
This dataset consists of single-cell RNA count measurements for 33,694 genes in 21,463 PBMCs that were processed using different library preparation protocols; 3-prime V1 (4,809 cells), 3-prime V2 (8,380 cells) and 5-prime (7,697 cells). This dataset has been used in \citet{Harmony} previously.

We obtained the data in binary format (RDS) and associated cell sub-type metadata from \citet{Harmony}\footnote{\url{https://github.com/immunogenomics/harmony2019/tree/master/data/figure4}}. We selected the cells that were processed using 3-prime V2 and 5-prime library preparation methods and filtered out mitochondrial genes and those that had zero counts across all cells in the two library preparation methods. This resulted in single-cell expression data for 16,077 cells and 23,338 genes, after filtering.

Using R version 3.6.3 and Seurat 3.0.2, we normalised via \texttt{NormalizeData(x, normalization.method = "LogNormalize", scale.factor = 10000)} and identified the most variable genes using \texttt{FindVariableFeatures(x, selection.method = "vst", nfeatures = 2000)} for each library preparation protocol independently. Then, we applied Seurat's \texttt{SelectIntegrationFeatures} function to select the top 2000 genes that were repeatedly variable across library preparation protocols.

We applied Seurat's canonical correlation analysis (CCA) method \cite{stuart2019comprehensive} to integrate the data from the two library protocols (using Seurat's \texttt{FindIntegrationAnchors} and \texttt{IntegrateData} functions) and extracted the resulting embeddings. Using Harmony version 0.1.0 \cite{Harmony} we integrated the data from the two library protocols using the \texttt{RunHarmony} function and extracted the resulting embeddings.

\paragraph{UCI Adult Income}
This dataset is derived from the 1994 United States census bureau and contains information relating to education, marriage status, ethnicity, self-reported gender of census participants and a binary high / low income label (\$50,000 threshold). Data was downloaded from the UCI Machine Learning Repository \cite{Dua2019UCIIncome}.

\subsection{Evaluation metrics}

Let $\{(\mathbf{x}_i\, c_i, d_i)\}_{i=1}^N$ be the dataset $N$ samples with $c_i \in \{0, 1\}$ the binary condition variable and $d_i \in \{d^{(m)}\}_{m=1}^{M}$ an additional discrete random variable of interest not used in training. 

We start with some housekeeping definitions of sample index subsets of the full dataset $[1, N]$. Let $N_{i,k}$ be the index set of the $k$ nearest-neighbours of $z_i$. Let $I_{c}$ be the index set of samples has $c_j=c$, and $J_{d}$ for samples that has $d_j=d$. 

From \citet{kbet}, $\kbet_{k, \alpha}$ is the proportion of rejected null hypotheses from the set of separate $\chi^2$ independence tests, with significance threshold $\alpha$, on the $k$ nearest-neighbours of every sample. If we let $\kbet_{k, \alpha}^d$ be the metric calculated on the filtered sub-population with index set $J_d$, then we define a mean kBET metric as
\begin{equation}
    \mkbet_{k, \alpha} := \frac{1}{M}\sum_{m=1}^M  \kbet_{k, \alpha}^{d^{(m)}}.
\end{equation}

We also consider local Silhouette Coefficients \cite{rousseeuw1987silhouettes}
\begin{equation}\label{sil}
    s_{k, c} := \frac{1}{|I_c|}\sum_{i\in I_c}\frac{b_{i,k}-a_{i,k }}{\max(a_{i, k}, b_{i,k})}, \quad s_{k} := \frac{1}{|I_c \cup I_{\neg c}|}\sum_{i\in I_c \cup I_{\neg c}}\frac{b_{i,k}-a_{i,k }}{\max(a_{i, k}, b_{i,k})},
\end{equation}
where $a_{i, k}$ and $b_{i, k}$ are the mean Euclidean distances between $\mathbf{z}_i$ and all other sample points in the $k$ nearest-neighbour set that are of the same and different condition variable respectively; i.e.
\begin{equation}
    a_{i,k} \equiv \frac{1}{|N_{i,k}\cap I_{c_i}|} \sum_{j\in N_{i,k}\cap I_{c_i}} \lVert \mathbf{z}_i - \mathbf{z}_j \rVert,\qquad b_{i,k} \equiv \frac{1}{|N_{i,k} \cap I_{\neg c_i}|} \sum_{j\in N_{i,k} \cap I_{\neg c_i}} \lVert \mathbf{z}_i - \mathbf{z}_j \rVert. 
\end{equation}

Similar to the mean kBET metric, we can also define a mean Silhouette Coefficient $\tilde{s}_{k,c}$ as follows. We first define
\begin{equation}
    {s}_{k, c}^{d} := \frac{1}{|I_c \cap J_d|}\sum_{i\in I_c \cap J_d}\frac{b_{i,k,d}-a_{i,k,d}}{\max(a_{i, k,d}, b_{i,k,d})},
\end{equation}
with 
\begin{equation}
\begin{split}
    a_{i,k,d} &\equiv \frac{1}{|N_{i,k}\cap I_{c_i}\cap J_d|} \sum_{j\in N_{i,k}\cap I_{c_i}\cap J_d} \lVert \mathbf{z}_i - \mathbf{z}_j \rVert,\\
    b_{i,k,d} &\equiv \frac{1}{|N_{i,k} \cap I_{\neg c_i}\cap J_d|} \sum_{j\in N_{i,k} \cap I_{\neg c_i}\cap J_d} \lVert \mathbf{z}_i - \mathbf{z}_j \rVert. 
    \end{split}
\end{equation}
Then the mean local Silhouette Coefficient is
\begin{equation}
    \tilde{s}_{k,c} := \frac{1}{M}\sum_{m=1}^M s_{k,c}^{d^{(m)}},
\end{equation}
with $\tilde{s}_k$ defined analogously to $s_k$ in \eqref{sil}.
A well-mixed representation that keeps samples with identical $d_i$ together will have low values of $\mkbet_{k,\alpha}$ and $\tilde{s}_{k, c}$ close to zero. Higher values near 1 would indicate either an undesirable dependency between $\mathbf{z}$ and $c$ in the form of identifiable clusters around values of $c$, a censoring process that fails to preserve the clustering with respect to $d$, or a combination of both.

\subsection{\tcgaccle\  representations for individual cancer types}
As the cancer type labels are not used in training, there is the possibility that cell lines of one cancer type will cluster around tumours of a different type. Here we illustrate this risk by examining the subset of Prostate Cancer latent representations inferred by CoMP and trVAE, where the majority of tumours and cell lines for this cancer type can be found in a single group. As shown in Figure \ref{individualcancers}, trVAE has cell-lines from other cancer types erroneously placed within the prostate cancer cluster; CoMP, on the other hand, maintains a relatively high level of specificity with fewer non-prostate cancer cell lines present. On average across all cancer types, this favourable behaviour of CoMP is reflected in the low $\tilde{s}$ and \mkbet\  scores.

\begin{figure}[t]
  \centering
  \includegraphics[width=1\columnwidth]{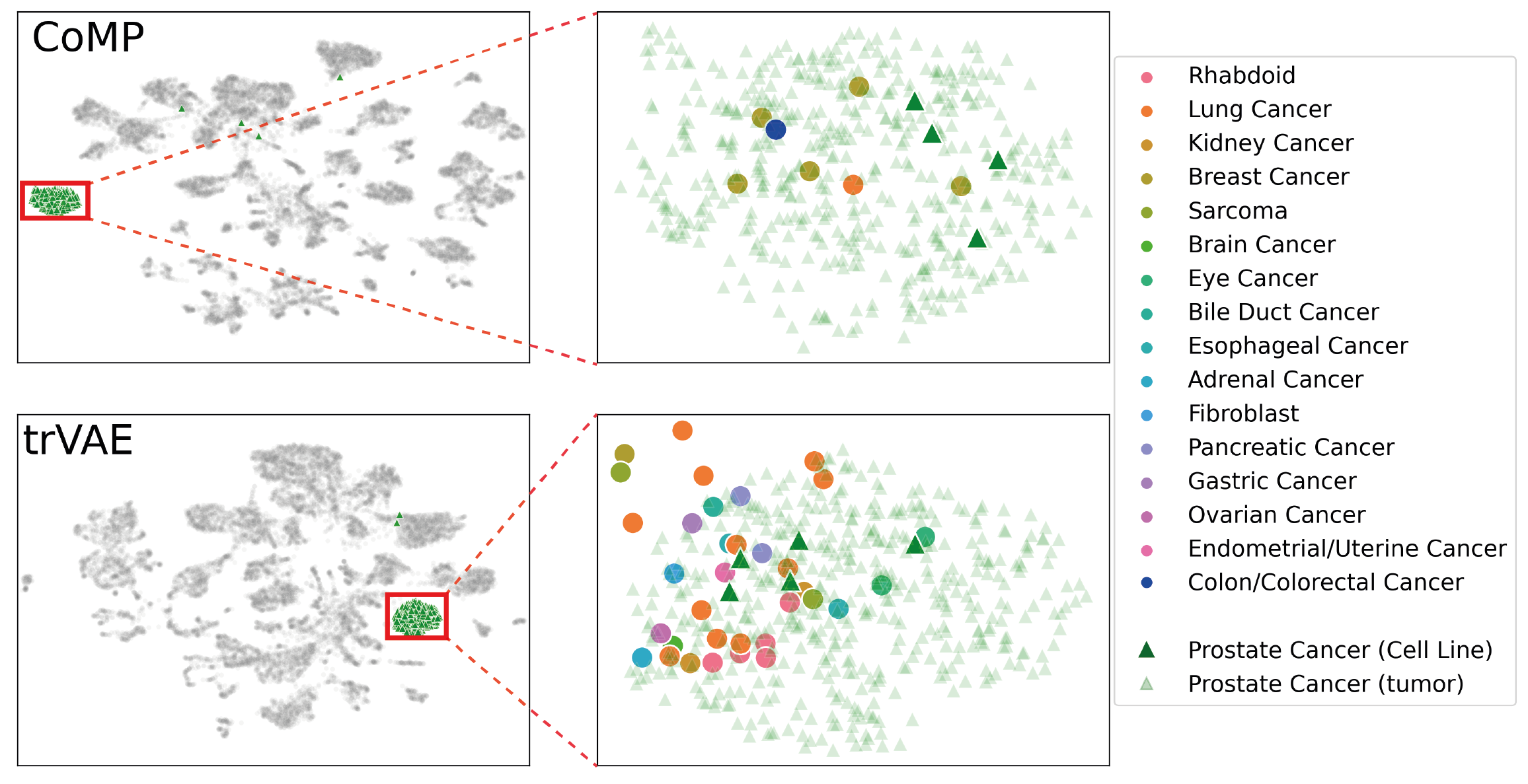}
  \caption{2D UMAP projection of the CoMP and trVAE posterior means of $\mathbf{z}_i$ from \tcgaccle\ data and the detailed Prostate Cancer tumour sample clusters.}
  \label{individualcancers}
\end{figure}

\subsection{Condition mixing metrics for stimulated / untreated single-cell PBMCs}
In this section we present additional results of our experiments on the stimulated / untreated single-cell PBMCs scRNA-seq data evaluating the condition mixing capabilities of CoMP. We focus on the two mixing metrics -- $s_k$ and $\kbet_{k, \alpha}$ -- and report both the mean values and their standard errors over 10 random model initialisations. We have the following three sets of experiments:
\paragraph{Benchmarking} In Tables \ref{tab:kang_kbet_metrics} and \ref{tab:kang_sil_metrics} we benchmark CoMP against the four other VAE models and show that CoMP outperforms the other models by significant margins on both metrics.
\paragraph{Cell type level evaluation} In Table \ref{tab:kang_cell_kbet_sil_metrics} we evaluate the mixing at a cell type level, where the strong mixing capabilities of CoMP is seen consistently across cell types. In particular, we highlight the good mixing of the CD14 Mono cell type by CoMP relative to the other penalised models.
\paragraph{CoMP penalty scale $\gamma$} In Tables \ref{tab:kang_kbet_metrics_gamma} to \ref{tab:kang_cell_kbet_sil_metrics_gamma} we explore the effect of varying the CoMP penalty scale $\gamma$ at both the population and cell type levels. Here we see that the optimum value is $\sim 1$.

\begin{table}
\centering
\caption{kBET metrics for stimulated / untreated single-cell PBMCs expression dataset with $k = 100$ and $\alpha = 0.1$. Here, $\kbet$ and $\mkbet$ refer to the mean $\kbet$ and mean $\mkbet$ across 10 random seeds for each model, respectively. SEM represents the standard error of the mean.}
\begin{tabular}{lcccc}
\toprule
 Model & $\kbet_{k,\alpha}$ & $\kbet_{k,\alpha} \pm \text{SEM}$ & $\mkbet_{k,\alpha}$ & $\mkbet_{k,\alpha} \pm \text{SEM}$ \\
\midrule
   VAE &             0.9788 &                  (0.9754, 0.9821) &              0.9443 &                   (0.9351, 0.9535) \\
  CVAE &             0.9056 &                  (0.8973, 0.9139) &              0.8202 &                   (0.8060, 0.8344) \\
  VFAE &             0.4753 &                  (0.4660, 0.4847) &              0.4067 &                   (0.3942, 0.4192) \\
 trVAE &             0.5082 &                  (0.4946, 0.5218) &              0.3819 &                   (0.3683, 0.3955) \\
  CoMP &             0.1211 &                  (0.0845, 0.1577) &              0.0681 &                   (0.0388, 0.0975) \\
\bottomrule
\end{tabular}
\label{tab:kang_kbet_metrics}
\end{table}

\begin{table}
\centering
\caption{Silhouette Coefficient metrics for stimulated / untreated single-cell PBMCs expression dataset with $k = 100$. Here, $s$ and $\tilde{s}$ refer to the mean $s$ and mean $\tilde{s}$ across 10 random seeds for each model, respectively. SEM represents the standard error of the mean.}
\begin{tabular}{lcccc}
\toprule
 Model &  $s_{k}$ & $s_{k} \pm \text{SEM}$ & $\tilde{s}_{k}$ & $\tilde{s}_{k} \pm \text{SEM}$ \\
\midrule
   VAE &   0.6354 &       (0.6303, 0.6404) &          0.5249 &               (0.5172, 0.5326) \\
  CVAE &   0.4872 &       (0.4805, 0.4939) &          0.3856 &               (0.3802, 0.3910) \\
  VFAE &   0.0501 &       (0.0457, 0.0544) &          0.0793 &               (0.0731, 0.0855) \\
 trVAE &   0.0651 &       (0.0596, 0.0705) &          0.0605 &               (0.0574, 0.0636) \\
  CoMP &  -0.0026 &     (-0.0032, -0.0020) &         -0.0013 &              (-0.0027, 0.0001) \\
\bottomrule
\end{tabular}
\label{tab:kang_sil_metrics}
\end{table}

\begin{table}
\centering
\caption{Cell type specific kBET and Silhouette Coefficient metrics for the stimulated / untreated single-cell PBMCs expression dataset summarised for 10 random seeds for each model. Metrics represent the mean value across the 10 random seeds for each model. Here, $k = 100$ and $\alpha = 0.1$.}
\vspace{4pt}
\begin{tabular}{llcccc}
\toprule
 Cell type &  Model & $\kbet_{k,\alpha}$ & $\kbet_{k,\alpha} \pm \text{SEM}$ &  $s_{k}$ & $s_{k} \pm \text{SEM}$ \\
\midrule
         B &    VAE &             0.9724 &                  (0.9683, 0.9765) &   0.5375 &       (0.5225, 0.5526) \\
         B &   CVAE &             0.9016 &                  (0.8964, 0.9068) &   0.2884 &       (0.2783, 0.2985) \\
         B &   VFAE &             0.3892 &                  (0.3357, 0.4426) &   0.0263 &       (0.0205, 0.0320) \\
         B &  trVAE &             0.2697 &                  (0.2243, 0.3151) &   0.0102 &       (0.0083, 0.0121) \\
         B &   CoMP &             0.0110 &                  (0.0002, 0.0217) &  -0.0040 &     (-0.0050, -0.0030) \\
 CD14 Mono &    VAE &             1.0000 &                  (1.0000, 1.0000) &   0.9388 &       (0.9356, 0.9420) \\
 CD14 Mono &   CVAE &             1.0000 &                  (1.0000, 1.0000) &   0.9373 &       (0.9317, 0.9428) \\
 CD14 Mono &   VFAE &             0.8192 &                  (0.8084, 0.8300) &   0.0548 &       (0.0486, 0.0610) \\
 CD14 Mono &  trVAE &             0.9360 &                  (0.9213, 0.9508) &   0.1579 &       (0.1414, 0.1743) \\
 CD14 Mono &   CoMP &             0.1709 &                  (0.0817, 0.2601) &   0.0003 &      (-0.0015, 0.0022) \\
 CD16 Mono &    VAE &             1.0000 &                  (1.0000, 1.0000) &   0.7462 &       (0.7168, 0.7756) \\
 CD16 Mono &   CVAE &             1.0000 &                  (1.0000, 1.0000) &   0.7455 &       (0.7204, 0.7706) \\
 CD16 Mono &   VFAE &             0.8796 &                  (0.8572, 0.9020) &   0.2328 &       (0.1860, 0.2796) \\
 CD16 Mono &  trVAE &             0.8059 &                  (0.7856, 0.8263) &   0.0535 &       (0.0493, 0.0576) \\
 CD16 Mono &   CoMP &             0.0947 &                  (0.0184, 0.1711) &   0.0005 &      (-0.0011, 0.0021) \\
     CD4 T &    VAE &             0.9964 &                  (0.9960, 0.9968) &   0.5069 &       (0.4932, 0.5206) \\
     CD4 T &   CVAE &             0.9104 &                  (0.9028, 0.9179) &   0.2103 &       (0.1956, 0.2250) \\
     CD4 T &   VFAE &             0.1460 &                  (0.1358, 0.1562) &   0.0035 &       (0.0030, 0.0039) \\
     CD4 T &  trVAE &             0.2236 &                  (0.1985, 0.2487) &   0.0042 &       (0.0036, 0.0047) \\
     CD4 T &   CoMP &             0.0538 &                  (0.0442, 0.0634) &  -0.0022 &     (-0.0028, -0.0016) \\
     CD8 T &    VAE &             0.9041 &                  (0.8633, 0.9448) &   0.2828 &       (0.2714, 0.2942) \\
     CD8 T &   CVAE &             0.4805 &                  (0.4080, 0.5529) &   0.0653 &       (0.0579, 0.0728) \\
     CD8 T &   VFAE &             0.0397 &                  (0.0307, 0.0486) &   0.0094 &       (0.0083, 0.0106) \\
     CD8 T &  trVAE &             0.0317 &                  (0.0224, 0.0409) &   0.0071 &       (0.0052, 0.0090) \\
     CD8 T &   CoMP &             0.0634 &                  (0.0437, 0.0830) &  -0.0000 &      (-0.0008, 0.0008) \\
        DC &    VAE &             1.0000 &                  (1.0000, 1.0000) &   0.6834 &       (0.6678, 0.6991) \\
        DC &   CVAE &             1.0000 &                  (1.0000, 1.0000) &   0.6723 &       (0.6598, 0.6847) \\
        DC &   VFAE &             0.7095 &                  (0.6715, 0.7476) &   0.2901 &       (0.2733, 0.3069) \\
        DC &  trVAE &             0.6286 &                  (0.5972, 0.6600) &   0.2339 &       (0.2225, 0.2453) \\
        DC &   CoMP &             0.0784 &                  (0.0329, 0.1239) &   0.0034 &      (-0.0027, 0.0096) \\
        NK &    VAE &             0.9645 &                  (0.9525, 0.9764) &   0.2609 &       (0.2358, 0.2860) \\
        NK &   CVAE &             0.8798 &                  (0.8682, 0.8914) &   0.1095 &       (0.0975, 0.1215) \\
        NK &   VFAE &             0.1548 &                  (0.1258, 0.1838) &   0.0093 &       (0.0072, 0.0114) \\
        NK &  trVAE &             0.1013 &                  (0.0514, 0.1512) &   0.0113 &       (0.0067, 0.0159) \\
        NK &   CoMP &             0.0721 &                  (0.0488, 0.0953) &  -0.0025 &     (-0.0035, -0.0015) \\
         T &    VAE &             0.7172 &                  (0.6377, 0.7967) &   0.2423 &       (0.2135, 0.2711) \\
         T &   CVAE &             0.3891 &                  (0.3315, 0.4466) &   0.0561 &       (0.0459, 0.0663) \\
         T &   VFAE &             0.1155 &                  (0.0934, 0.1376) &   0.0082 &       (0.0060, 0.0104) \\
         T &  trVAE &             0.0585 &                  (0.0354, 0.0815) &   0.0061 &       (0.0047, 0.0075) \\
         T &   CoMP &             0.0009 &                  (0.0003, 0.0016) &  -0.0062 &     (-0.0074, -0.0050) \\
\bottomrule
\end{tabular}
\label{tab:kang_cell_kbet_sil_metrics}
\end{table}

\begin{table}
\centering
\caption{Effect of varying $\gamma$ for the CoMP model on the kBET metrics for the stimulated / untreated single-cell PBMCs expression dataset. Here, $\kbet$ and $\mkbet$ refer to the mean $\kbet$ and mean $\mkbet$ across 10 random seeds for each value of $\gamma$, respectively. SEM represents the standard error of the mean. Here, $k = 100$ and $\alpha = 0.1$.}
\begin{tabular}{llcccc}
\toprule
Model &  $\gamma$ & $\kbet_{k,\alpha}$ & $\kbet_{k,\alpha} \pm \text{SEM}$ & $\mkbet_{k,\alpha}$ & $\mkbet_{k,\alpha} \pm \text{SEM}$ \\
\midrule
 CoMP &      0.25 &             0.2703 &                  (0.2482, 0.2924) &              0.1276 &                   (0.1085, 0.1467) \\
 CoMP &      0.50 &             0.1741 &                  (0.1438, 0.2045) &              0.0763 &                   (0.0543, 0.0983) \\
 CoMP &      1.00 &             0.1211 &                  (0.0845, 0.1577) &              0.0681 &                   (0.0388, 0.0975) \\
 CoMP &      5.00 &             0.4426 &                  (0.3889, 0.4963) &              0.4419 &                   (0.3827, 0.5011) \\
 CoMP &     10.00 &             0.5311 &                  (0.4456, 0.6167) &              0.5118 &                   (0.4135, 0.6100) \\
 CoMP &     15.00 &             0.4288 &                  (0.3511, 0.5065) &              0.4614 &                   (0.3738, 0.5489) \\
 CoMP &     20.00 &             0.5880 &                  (0.5101, 0.6660) &              0.6383 &                   (0.5547, 0.7219) \\
\bottomrule
\end{tabular}
\label{tab:kang_kbet_metrics_gamma}
\end{table}

\begin{table}
\centering
\caption{Effect of varying $\gamma$ for the CoMP model on the Silhouette Coefficient metrics for the stimulated / untreated single-cell PBMCs expression dataset. Here, $s$ and $\tilde{s}$ refer to the mean $s$ and mean $\tilde{s}$ across 10 random seeds for each value of $\gamma$, respectively. SEM represents the standard error of the mean. Here, $k = 100$.}
\begin{tabular}{llcccc}
\toprule
Model &  $\gamma$ &  $s_{k}$ & $s_{k} \pm \text{SEM}$ & $\tilde{s}_{k}$ & $\tilde{s}_{k} \pm \text{SEM}$ \\
\midrule
 CoMP &      0.25 &  -0.0024 &     (-0.0028, -0.0021) &         -0.0006 &              (-0.0018, 0.0006) \\
 CoMP &      0.50 &  -0.0029 &     (-0.0030, -0.0028) &         -0.0023 &             (-0.0031, -0.0015) \\
 CoMP &      1.00 &  -0.0026 &     (-0.0032, -0.0020) &         -0.0013 &              (-0.0027, 0.0001) \\
 CoMP &      5.00 &   0.0043 &       (0.0026, 0.0059) &          0.0209 &               (0.0145, 0.0274) \\
 CoMP &     10.00 &   0.0028 &       (0.0016, 0.0039) &          0.0523 &               (0.0300, 0.0746) \\
 CoMP &     15.00 &   0.0046 &       (0.0025, 0.0067) &          0.0750 &               (0.0484, 0.1016) \\
 CoMP &     20.00 &   0.0061 &       (0.0038, 0.0083) &          0.1319 &               (0.1053, 0.1584) \\
\bottomrule
\end{tabular}
\label{tab:kang_sil_metrics_gamma}
\end{table}

\begin{table}
\centering
\caption{Effect of varying $\gamma$ for cell type specific kBET and $s$ metrics for the stimulated / untreated single-cell PBMCs expression dataset. Metrics represent the mean value across the 10 random seeds. Here, $k = 100$ and $\alpha = 0.1$.}
\resizebox{0.7\textwidth}{!}{
\begin{tabular}{lllcccc}
\toprule
 Cell type & Model &  $\gamma$ & $\kbet_{k,\alpha}$ & $\kbet_{k,\alpha} \pm \text{SEM}$ &  $s_{k}$ & $s_{k} \pm \text{SEM}$ \\
\midrule
         B &  CoMP &      0.25 &             0.0061 &                  (0.0037, 0.0086) &  -0.0046 &     (-0.0053, -0.0038) \\
         B &  CoMP &      0.50 &             0.0030 &                  (0.0004, 0.0056) &  -0.0033 &     (-0.0039, -0.0028) \\
         B &  CoMP &      1.00 &             0.0110 &                  (0.0002, 0.0217) &  -0.0040 &     (-0.0050, -0.0030) \\
         B &  CoMP &      5.00 &             0.4860 &                  (0.3881, 0.5840) &   0.0094 &       (0.0073, 0.0115) \\
         B &  CoMP &     10.00 &             0.5321 &                  (0.4226, 0.6415) &   0.0195 &       (0.0118, 0.0273) \\
         B &  CoMP &     15.00 &             0.5229 &                  (0.4331, 0.6127) &   0.0167 &       (0.0109, 0.0224) \\
         B &  CoMP &     20.00 &             0.6135 &                  (0.5308, 0.6963) &   0.0729 &       (0.0384, 0.1075) \\
 CD14 Mono &  CoMP &      0.25 &             0.6868 &                  (0.6336, 0.7400) &   0.0060 &      (-0.0001, 0.0121) \\
 CD14 Mono &  CoMP &      0.50 &             0.3840 &                  (0.3001, 0.4680) &  -0.0007 &      (-0.0017, 0.0003) \\
 CD14 Mono &  CoMP &      1.00 &             0.1709 &                  (0.0817, 0.2601) &   0.0003 &      (-0.0015, 0.0022) \\
 CD14 Mono &  CoMP &      5.00 &             0.4530 &                  (0.3295, 0.5765) &   0.0254 &       (0.0092, 0.0416) \\
 CD14 Mono &  CoMP &     10.00 &             0.7073 &                  (0.6114, 0.8031) &   0.0745 &       (0.0419, 0.1072) \\
 CD14 Mono &  CoMP &     15.00 &             0.6991 &                  (0.5946, 0.8036) &   0.1240 &       (0.0795, 0.1685) \\
 CD14 Mono &  CoMP &     20.00 &             0.7889 &                  (0.6943, 0.8834) &   0.1429 &       (0.1012, 0.1846) \\
 CD16 Mono &  CoMP &      0.25 &             0.0963 &                  (0.0382, 0.1543) &   0.0031 &       (0.0018, 0.0045) \\
 CD16 Mono &  CoMP &      0.50 &             0.0589 &                  (0.0216, 0.0962) &   0.0003 &      (-0.0008, 0.0013) \\
 CD16 Mono &  CoMP &      1.00 &             0.0947 &                  (0.0184, 0.1711) &   0.0005 &      (-0.0011, 0.0021) \\
 CD16 Mono &  CoMP &      5.00 &             0.5897 &                  (0.4934, 0.6859) &   0.0421 &       (0.0248, 0.0593) \\
 CD16 Mono &  CoMP &     10.00 &             0.6113 &                  (0.4660, 0.7566) &   0.1741 &       (0.0946, 0.2537) \\
 CD16 Mono &  CoMP &     15.00 &             0.5950 &                  (0.4552, 0.7348) &   0.2778 &       (0.1705, 0.3850) \\
 CD16 Mono &  CoMP &     20.00 &             0.8420 &                  (0.7381, 0.9460) &   0.4705 &       (0.3782, 0.5628) \\
     CD4 T &  CoMP &      0.25 &             0.0642 &                  (0.0543, 0.0742) &  -0.0027 &     (-0.0032, -0.0022) \\
     CD4 T &  CoMP &      0.50 &             0.0401 &                  (0.0344, 0.0458) &  -0.0027 &     (-0.0031, -0.0023) \\
     CD4 T &  CoMP &      1.00 &             0.0538 &                  (0.0442, 0.0634) &  -0.0022 &     (-0.0028, -0.0016) \\
     CD4 T &  CoMP &      5.00 &             0.3631 &                  (0.2966, 0.4296) &   0.0036 &       (0.0024, 0.0048) \\
     CD4 T &  CoMP &     10.00 &             0.4058 &                  (0.3020, 0.5096) &   0.0075 &       (0.0039, 0.0111) \\
     CD4 T &  CoMP &     15.00 &             0.3287 &                  (0.2466, 0.4108) &   0.0080 &       (0.0043, 0.0117) \\
     CD4 T &  CoMP &     20.00 &             0.5127 &                  (0.4140, 0.6115) &   0.0561 &       (0.0192, 0.0929) \\
     CD8 T &  CoMP &      0.25 &             0.0081 &                  (0.0051, 0.0111) &  -0.0015 &     (-0.0026, -0.0004) \\
     CD8 T &  CoMP &      0.50 &             0.0216 &                  (0.0129, 0.0304) &  -0.0021 &     (-0.0032, -0.0010) \\
     CD8 T &  CoMP &      1.00 &             0.0634 &                  (0.0437, 0.0830) &  -0.0000 &      (-0.0008, 0.0008) \\
     CD8 T &  CoMP &      5.00 &             0.4287 &                  (0.3508, 0.5067) &   0.0266 &       (0.0167, 0.0365) \\
     CD8 T &  CoMP &     10.00 &             0.4289 &                  (0.3251, 0.5326) &   0.0115 &       (0.0073, 0.0156) \\
     CD8 T &  CoMP &     15.00 &             0.2733 &                  (0.1927, 0.3540) &   0.0071 &       (0.0051, 0.0090) \\
     CD8 T &  CoMP &     20.00 &             0.4913 &                  (0.3731, 0.6095) &   0.0495 &       (0.0127, 0.0863) \\
        DC &  CoMP &      0.25 &             0.1379 &                  (0.0972, 0.1786) &   0.0056 &       (0.0026, 0.0085) \\
        DC &  CoMP &      0.50 &             0.0739 &                  (0.0394, 0.1085) &   0.0009 &      (-0.0029, 0.0047) \\
        DC &  CoMP &      1.00 &             0.0784 &                  (0.0329, 0.1239) &   0.0034 &      (-0.0027, 0.0096) \\
        DC &  CoMP &      5.00 &             0.2339 &                  (0.1546, 0.3132) &   0.0461 &       (0.0230, 0.0693) \\
        DC &  CoMP &     10.00 &             0.4642 &                  (0.3143, 0.6141) &   0.1064 &       (0.0433, 0.1695) \\
        DC &  CoMP &     15.00 &             0.6008 &                  (0.4650, 0.7367) &   0.1502 &       (0.0832, 0.2172) \\
        DC &  CoMP &     20.00 &             0.7962 &                  (0.6907, 0.9017) &   0.1718 &       (0.1117, 0.2319) \\
        NK &  CoMP &      0.25 &             0.0158 &                  (0.0075, 0.0241) &  -0.0043 &     (-0.0049, -0.0036) \\
        NK &  CoMP &      0.50 &             0.0233 &                  (0.0045, 0.0420) &  -0.0048 &     (-0.0056, -0.0041) \\
        NK &  CoMP &      1.00 &             0.0721 &                  (0.0488, 0.0953) &  -0.0025 &     (-0.0035, -0.0015) \\
        NK &  CoMP &      5.00 &             0.6378 &                  (0.5259, 0.7497) &   0.0058 &       (0.0018, 0.0097) \\
        NK &  CoMP &     10.00 &             0.5422 &                  (0.4403, 0.6440) &   0.0135 &       (0.0094, 0.0177) \\
        NK &  CoMP &     15.00 &             0.4105 &                  (0.3194, 0.5016) &   0.0134 &       (0.0083, 0.0185) \\
        NK &  CoMP &     20.00 &             0.5637 &                  (0.4606, 0.6667) &   0.0372 &       (0.0152, 0.0592) \\
         T &  CoMP &      0.25 &             0.0052 &                  (0.0025, 0.0079) &  -0.0066 &     (-0.0076, -0.0057) \\
         T &  CoMP &      0.50 &             0.0054 &                  (0.0021, 0.0086) &  -0.0059 &     (-0.0066, -0.0052) \\
         T &  CoMP &      1.00 &             0.0009 &                  (0.0003, 0.0016) &  -0.0062 &     (-0.0074, -0.0050) \\
         T &  CoMP &      5.00 &             0.3430 &                  (0.2440, 0.4420) &   0.0087 &       (0.0041, 0.0134) \\
         T &  CoMP &     10.00 &             0.4024 &                  (0.3032, 0.5015) &   0.0113 &       (0.0067, 0.0159) \\
         T &  CoMP &     15.00 &             0.2605 &                  (0.1846, 0.3364) &   0.0030 &       (0.0005, 0.0055) \\
         T &  CoMP &     20.00 &             0.4979 &                  (0.3883, 0.6076) &   0.0542 &       (0.0224, 0.0860) \\
\bottomrule
\end{tabular}
}
\label{tab:kang_cell_kbet_sil_metrics_gamma}
\end{table}

\subsection{Counterfactual prediction of stimulated / untreated single-cell PBMCs expression dataset (IFN-$\beta$ stimulation)}
In this section we present the full results on the counterfactual prediction of single-cell PBMC expression data under IFN-$\beta$ stimulation. In Tables \ref{tab:kang_cf_pearson} and \ref{tab:kang_cf_mse}, we present the mean and standard error of the Pearson correlation coefficient and MSE metrics respectively for CoMP and the other four VAE models. We present our results for each cell type separately. As is consistent with the summary presented in Figures \ref{fig:kang_cf} and \ref{fig:kang_cf_all}, we see that CoMP produces highly accurate counterfactual reconstructions. Indeed, this can be seen in the scatter plots showing the mean expression of (actual) stimulated cells against the mean of counterfactually stimulated control cells (Figure \ref{fig:kang_cf_scatter}). Here, we see that the other VAE models tend to underestimate the expression of genes that are up-regulated by IFN-$\beta$ stimulation and overestimate the expression of genes that are down-regulated. However, this is not as evident with CoMP.

Similar to the mixing metrics, we evaluate the effect of varying the penalty scale $\gamma$. As we see in Tables \ref{tab:kang_cf_pearson_gamma} and \ref{tab:kang_cf_mse_gamma}, the optimal value is $\approx 1$.

\begin{table}
\centering
\caption{Counterfactual reconstruction by cell type: Pearson correlation coefficient metrics for all genes ($r_{all}$) and the top 50 DE genes ($r_{DE}$). Metrics represent the mean across 10 random seeds for each model. SEM represents standard error of the mean.}
\begin{tabular}{llcccc}
\toprule
 Cell type &  Model & $r_{\text{all}}$ & $r_{\text{all}} \pm \text{SEM}$ & $r_{\text{DE}}$ & $r_{\text{DE}} \pm \text{SEM}$ \\
\midrule
         B &    VAE &           0.8854 &                (0.8850, 0.8857) &          0.8170 &               (0.8165, 0.8175) \\
         B &   CVAE &           0.9499 &                (0.9481, 0.9516) &          0.9153 &               (0.9125, 0.9181) \\
         B &   VFAE &           0.9908 &                (0.9901, 0.9915) &          0.9880 &               (0.9866, 0.9893) \\
         B &  trVAE &           0.9877 &                (0.9868, 0.9886) &          0.9833 &               (0.9817, 0.9849) \\
         B &   CoMP &           0.9986 &                (0.9984, 0.9988) &          0.9985 &               (0.9982, 0.9987) \\
 CD14 Mono &    VAE &           0.7488 &                (0.7485, 0.7491) &          0.4896 &               (0.4891, 0.4900) \\
 CD14 Mono &   CVAE &           0.7529 &                (0.7520, 0.7538) &          0.4958 &               (0.4938, 0.4977) \\
 CD14 Mono &   VFAE &           0.9954 &                (0.9951, 0.9958) &          0.9928 &               (0.9921, 0.9935) \\
 CD14 Mono &  trVAE &           0.9830 &                (0.9804, 0.9856) &          0.9650 &               (0.9586, 0.9714) \\
 CD14 Mono &   CoMP &           0.9954 &                (0.9915, 0.9992) &          0.9920 &               (0.9848, 0.9993) \\
 CD16 Mono &    VAE &           0.8304 &                (0.8301, 0.8307) &          0.7135 &               (0.7131, 0.7140) \\
 CD16 Mono &   CVAE &           0.8351 &                (0.8341, 0.8360) &          0.7223 &               (0.7203, 0.7243) \\
 CD16 Mono &   VFAE &           0.9912 &                (0.9909, 0.9915) &          0.9910 &               (0.9904, 0.9916) \\
 CD16 Mono &  trVAE &           0.9881 &                (0.9873, 0.9889) &          0.9821 &               (0.9802, 0.9839) \\
 CD16 Mono &   CoMP &           0.9990 &                (0.9985, 0.9994) &          0.9989 &               (0.9986, 0.9993) \\
     CD4 T &    VAE &           0.8975 &                (0.8971, 0.8978) &          0.8366 &               (0.8360, 0.8372) \\
     CD4 T &   CVAE &           0.9697 &                (0.9682, 0.9712) &          0.9514 &               (0.9492, 0.9537) \\
     CD4 T &   VFAE &           0.9977 &                (0.9975, 0.9979) &          0.9983 &               (0.9982, 0.9985) \\
     CD4 T &  trVAE &           0.9915 &                (0.9908, 0.9922) &          0.9905 &               (0.9893, 0.9918) \\
     CD4 T &   CoMP &           0.9990 &                (0.9990, 0.9991) &          0.9988 &               (0.9987, 0.9989) \\
     CD8 T &    VAE &           0.9108 &                (0.9104, 0.9112) &          0.8719 &               (0.8713, 0.8724) \\
     CD8 T &   CVAE &           0.9726 &                (0.9715, 0.9736) &          0.9613 &               (0.9598, 0.9628) \\
     CD8 T &   VFAE &           0.9923 &                (0.9920, 0.9927) &          0.9935 &               (0.9931, 0.9939) \\
     CD8 T &  trVAE &           0.9828 &                (0.9810, 0.9846) &          0.9808 &               (0.9781, 0.9836) \\
     CD8 T &   CoMP &           0.9927 &                (0.9917, 0.9937) &          0.9950 &               (0.9945, 0.9955) \\
        DC &    VAE &           0.8156 &                (0.8153, 0.8159) &          0.5809 &               (0.5802, 0.5816) \\
        DC &   CVAE &           0.8213 &                (0.8205, 0.8221) &          0.5943 &               (0.5925, 0.5961) \\
        DC &   VFAE &           0.9885 &                (0.9879, 0.9892) &          0.9894 &               (0.9887, 0.9901) \\
        DC &  trVAE &           0.9743 &                (0.9702, 0.9783) &          0.9502 &               (0.9396, 0.9608) \\
        DC &   CoMP &           0.9946 &                (0.9925, 0.9966) &          0.9931 &               (0.9899, 0.9962) \\
        NK &    VAE &           0.8918 &                (0.8910, 0.8926) &          0.8304 &               (0.8292, 0.8316) \\
        NK &   CVAE &           0.9539 &                (0.9520, 0.9558) &          0.9269 &               (0.9237, 0.9301) \\
        NK &   VFAE &           0.9870 &                (0.9865, 0.9874) &          0.9864 &               (0.9855, 0.9873) \\
        NK &  trVAE &           0.9393 &                (0.9259, 0.9526) &          0.9290 &               (0.9113, 0.9466) \\
        NK &   CoMP &           0.9917 &                (0.9904, 0.9929) &          0.9899 &               (0.9881, 0.9917) \\
         T &    VAE &           0.8848 &                (0.8843, 0.8853) &          0.7469 &               (0.7457, 0.7480) \\
         T &   CVAE &           0.9516 &                (0.9500, 0.9533) &          0.8960 &               (0.8926, 0.8994) \\
         T &   VFAE &           0.9849 &                (0.9841, 0.9856) &          0.9763 &               (0.9750, 0.9777) \\
         T &  trVAE &           0.9567 &                (0.9498, 0.9637) &          0.9368 &               (0.9294, 0.9443) \\
         T &   CoMP &           0.9941 &                (0.9936, 0.9946) &          0.9934 &               (0.9928, 0.9940) \\
\bottomrule
\end{tabular}
\label{tab:kang_cf_pearson}
\end{table}

\begin{table}
\centering
\caption{Counterfactual reconstruction by cell type: Mean squared error metrics for all genes ($\text{MSE}_{\text{all}}$) and the top 50 DE genes ($\text{MSE}_{\text{DE}}$). Metrics represent the mean across 10 random seeds for each model. SEM represents standard error of the mean.}
\begin{tabular}{llcccc}
\toprule
 Cell type &  Model & $\text{MSE}_{\text{all}}$ & $\text{MSE}_{\text{all}} \pm \text{SEM}$ & $\text{MSE}_{\text{DE}}$ & $\text{MSE}_{\text{DE}} \pm \text{SEM}$ \\
\midrule
         B &    VAE &                    0.0085 &                         (0.0085, 0.0085) &                   0.3230 &                        (0.3221, 0.3239) \\
         B &   CVAE &                    0.0038 &                         (0.0037, 0.0039) &                   0.1398 &                        (0.1347, 0.1448) \\
         B &   VFAE &                    0.0008 &                         (0.0007, 0.0008) &                   0.0199 &                        (0.0178, 0.0220) \\
         B &  trVAE &                    0.0010 &                         (0.0009, 0.0010) &                   0.0276 &                        (0.0250, 0.0301) \\
         B &   CoMP &                    0.0001 &                         (0.0001, 0.0001) &                   0.0024 &                        (0.0020, 0.0028) \\
 CD14 Mono &    VAE &                    0.0483 &                         (0.0483, 0.0484) &                   1.7942 &                        (1.7923, 1.7962) \\
 CD14 Mono &   CVAE &                    0.0476 &                         (0.0475, 0.0478) &                   1.7624 &                        (1.7563, 1.7684) \\
 CD14 Mono &   VFAE &                    0.0014 &                         (0.0013, 0.0015) &                   0.0343 &                        (0.0314, 0.0371) \\
 CD14 Mono &  trVAE &                    0.0044 &                         (0.0038, 0.0051) &                   0.1422 &                        (0.1190, 0.1654) \\
 CD14 Mono &   CoMP &                    0.0011 &                         (0.0002, 0.0019) &                   0.0245 &                        (0.0023, 0.0468) \\
 CD16 Mono &    VAE &                    0.0301 &                         (0.0301, 0.0302) &                   1.1255 &                        (1.1234, 1.1276) \\
 CD16 Mono &   CVAE &                    0.0294 &                         (0.0293, 0.0295) &                   1.0933 &                        (1.0878, 1.0989) \\
 CD16 Mono &   VFAE &                    0.0017 &                         (0.0017, 0.0018) &                   0.0223 &                        (0.0204, 0.0242) \\
 CD16 Mono &  trVAE &                    0.0029 &                         (0.0027, 0.0031) &                   0.0861 &                        (0.0790, 0.0932) \\
 CD16 Mono &   CoMP &                    0.0002 &                         (0.0001, 0.0003) &                   0.0031 &                        (0.0017, 0.0044) \\
     CD4 T &    VAE &                    0.0060 &                         (0.0059, 0.0060) &                   0.2274 &                        (0.2266, 0.2282) \\
     CD4 T &   CVAE &                    0.0018 &                         (0.0017, 0.0019) &                   0.0677 &                        (0.0644, 0.0709) \\
     CD4 T &   VFAE &                    0.0001 &                         (0.0001, 0.0002) &                   0.0021 &                        (0.0018, 0.0023) \\
     CD4 T &  trVAE &                    0.0005 &                         (0.0005, 0.0006) &                   0.0126 &                        (0.0107, 0.0144) \\
     CD4 T &   CoMP &                    0.0001 &                         (0.0001, 0.0001) &                   0.0015 &                        (0.0014, 0.0016) \\
     CD8 T &    VAE &                    0.0058 &                         (0.0058, 0.0058) &                   0.2187 &                        (0.2178, 0.2196) \\
     CD8 T &   CVAE &                    0.0019 &                         (0.0019, 0.0020) &                   0.0684 &                        (0.0659, 0.0709) \\
     CD8 T &   VFAE &                    0.0005 &                         (0.0005, 0.0006) &                   0.0098 &                        (0.0093, 0.0103) \\
     CD8 T &  trVAE &                    0.0012 &                         (0.0011, 0.0013) &                   0.0332 &                        (0.0284, 0.0381) \\
     CD8 T &   CoMP &                    0.0005 &                         (0.0004, 0.0006) &                   0.0074 &                        (0.0066, 0.0082) \\
        DC &    VAE &                    0.0332 &                         (0.0331, 0.0332) &                   1.2308 &                        (1.2292, 1.2324) \\
        DC &   CVAE &                    0.0322 &                         (0.0321, 0.0324) &                   1.1887 &                        (1.1834, 1.1939) \\
        DC &   VFAE &                    0.0024 &                         (0.0023, 0.0025) &                   0.0303 &                        (0.0287, 0.0318) \\
        DC &  trVAE &                    0.0056 &                         (0.0048, 0.0064) &                   0.1758 &                        (0.1436, 0.2081) \\
        DC &   CoMP &                    0.0011 &                         (0.0007, 0.0016) &                   0.0161 &                        (0.0083, 0.0239) \\
        NK &    VAE &                    0.0091 &                         (0.0091, 0.0092) &                   0.3395 &                        (0.3370, 0.3420) \\
        NK &   CVAE &                    0.0043 &                         (0.0041, 0.0044) &                   0.1535 &                        (0.1477, 0.1593) \\
        NK &   VFAE &                    0.0014 &                         (0.0013, 0.0014) &                   0.0345 &                        (0.0327, 0.0362) \\
        NK &  trVAE &                    0.0053 &                         (0.0042, 0.0064) &                   0.1455 &                        (0.1100, 0.1810) \\
        NK &   CoMP &                    0.0008 &                         (0.0007, 0.0010) &                   0.0204 &                        (0.0169, 0.0238) \\
         T &    VAE &                    0.0077 &                         (0.0077, 0.0077) &                   0.2799 &                        (0.2786, 0.2811) \\
         T &   CVAE &                    0.0033 &                         (0.0032, 0.0034) &                   0.1126 &                        (0.1088, 0.1164) \\
         T &   VFAE &                    0.0011 &                         (0.0010, 0.0011) &                   0.0237 &                        (0.0223, 0.0252) \\
         T &  trVAE &                    0.0030 &                         (0.0025, 0.0034) &                   0.0674 &                        (0.0583, 0.0764) \\
         T &   CoMP &                    0.0004 &                         (0.0004, 0.0005) &                   0.0066 &                        (0.0060, 0.0073) \\
\bottomrule
\end{tabular}
\label{tab:kang_cf_mse}
\end{table}

\begin{table}
\centering
\caption{Effect of $\gamma$ on counterfactual data reconstruction: Mean Pearson correlation coefficient for all and DE genes across 10 random seeds.}
\resizebox{0.7\textwidth}{!}{
\begin{tabular}{lllcccc}
\toprule
 Cell type & Model &  $\gamma$ & $r_{\text{all}}$ & $r_{\text{all}} \pm \text{SEM}$ & $r_{\text{DE}}$ & $r_{\text{DE}} \pm \text{SEM}$ \\
\midrule
         B &  CoMP &      0.25 &           0.9987 &                (0.9986, 0.9988) &          0.9986 &               (0.9984, 0.9987) \\
         B &  CoMP &      0.50 &           0.9987 &                (0.9985, 0.9988) &          0.9985 &               (0.9983, 0.9988) \\
         B &  CoMP &      1.00 &           0.9986 &                (0.9984, 0.9988) &          0.9985 &               (0.9982, 0.9987) \\
         B &  CoMP &      5.00 &           0.9577 &                (0.9432, 0.9722) &          0.9623 &               (0.9510, 0.9736) \\
         B &  CoMP &     10.00 &           0.9233 &                (0.9023, 0.9443) &          0.9397 &               (0.9239, 0.9555) \\
         B &  CoMP &     15.00 &           0.9106 &                (0.8925, 0.9287) &          0.9299 &               (0.9163, 0.9435) \\
         B &  CoMP &     20.00 &           0.8974 &                (0.8841, 0.9107) &          0.9080 &               (0.8966, 0.9195) \\
 CD14 Mono &  CoMP &      0.25 &           0.9906 &                (0.9866, 0.9945) &          0.9828 &               (0.9746, 0.9909) \\
 CD14 Mono &  CoMP &      0.50 &           0.9948 &                (0.9917, 0.9980) &          0.9910 &               (0.9844, 0.9977) \\
 CD14 Mono &  CoMP &      1.00 &           0.9954 &                (0.9915, 0.9992) &          0.9920 &               (0.9848, 0.9993) \\
 CD14 Mono &  CoMP &      5.00 &           0.9892 &                (0.9831, 0.9954) &          0.9823 &               (0.9723, 0.9922) \\
 CD14 Mono &  CoMP &     10.00 &           0.9536 &                (0.9316, 0.9757) &          0.9439 &               (0.9217, 0.9662) \\
 CD14 Mono &  CoMP &     15.00 &           0.9224 &                (0.8933, 0.9515) &          0.9174 &               (0.8886, 0.9463) \\
 CD14 Mono &  CoMP &     20.00 &           0.9034 &                (0.8737, 0.9332) &          0.8966 &               (0.8673, 0.9258) \\
 CD16 Mono &  CoMP &      0.25 &           0.9983 &                (0.9977, 0.9989) &          0.9982 &               (0.9976, 0.9988) \\
 CD16 Mono &  CoMP &      0.50 &           0.9989 &                (0.9985, 0.9992) &          0.9988 &               (0.9985, 0.9991) \\
 CD16 Mono &  CoMP &      1.00 &           0.9990 &                (0.9985, 0.9994) &          0.9989 &               (0.9986, 0.9993) \\
 CD16 Mono &  CoMP &      5.00 &           0.9847 &                (0.9805, 0.9890) &          0.9801 &               (0.9753, 0.9850) \\
 CD16 Mono &  CoMP &     10.00 &           0.9420 &                (0.9168, 0.9672) &          0.9503 &               (0.9313, 0.9693) \\
 CD16 Mono &  CoMP &     15.00 &           0.9017 &                (0.8702, 0.9332) &          0.9222 &               (0.8999, 0.9444) \\
 CD16 Mono &  CoMP &     20.00 &           0.8563 &                (0.8289, 0.8838) &          0.8850 &               (0.8668, 0.9031) \\
     CD4 T &  CoMP &      0.25 &           0.9989 &                (0.9989, 0.9990) &          0.9987 &               (0.9986, 0.9988) \\
     CD4 T &  CoMP &      0.50 &           0.9991 &                (0.9990, 0.9991) &          0.9989 &               (0.9988, 0.9990) \\
     CD4 T &  CoMP &      1.00 &           0.9990 &                (0.9990, 0.9991) &          0.9988 &               (0.9987, 0.9989) \\
     CD4 T &  CoMP &      5.00 &           0.9925 &                (0.9899, 0.9951) &          0.9901 &               (0.9863, 0.9939) \\
     CD4 T &  CoMP &     10.00 &           0.9948 &                (0.9933, 0.9962) &          0.9970 &               (0.9954, 0.9985) \\
     CD4 T &  CoMP &     15.00 &           0.9944 &                (0.9931, 0.9958) &          0.9979 &               (0.9975, 0.9983) \\
     CD4 T &  CoMP &     20.00 &           0.9782 &                (0.9689, 0.9875) &          0.9738 &               (0.9584, 0.9892) \\
     CD8 T &  CoMP &      0.25 &           0.9963 &                (0.9961, 0.9964) &          0.9965 &               (0.9964, 0.9966) \\
     CD8 T &  CoMP &      0.50 &           0.9955 &                (0.9951, 0.9960) &          0.9962 &               (0.9959, 0.9965) \\
     CD8 T &  CoMP &      1.00 &           0.9927 &                (0.9917, 0.9937) &          0.9950 &               (0.9945, 0.9955) \\
     CD8 T &  CoMP &      5.00 &           0.9666 &                (0.9626, 0.9705) &          0.9790 &               (0.9765, 0.9814) \\
     CD8 T &  CoMP &     10.00 &           0.9559 &                (0.9499, 0.9620) &          0.9757 &               (0.9727, 0.9787) \\
     CD8 T &  CoMP &     15.00 &           0.9528 &                (0.9468, 0.9589) &          0.9745 &               (0.9715, 0.9774) \\
     CD8 T &  CoMP &     20.00 &           0.9455 &                (0.9397, 0.9512) &          0.9605 &               (0.9516, 0.9694) \\
        DC &  CoMP &      0.25 &           0.9959 &                (0.9955, 0.9962) &          0.9945 &               (0.9942, 0.9949) \\
        DC &  CoMP &      0.50 &           0.9966 &                (0.9963, 0.9970) &          0.9956 &               (0.9949, 0.9962) \\
        DC &  CoMP &      1.00 &           0.9946 &                (0.9925, 0.9966) &          0.9931 &               (0.9899, 0.9962) \\
        DC &  CoMP &      5.00 &           0.9694 &                (0.9576, 0.9811) &          0.9671 &               (0.9528, 0.9814) \\
        DC &  CoMP &     10.00 &           0.9219 &                (0.8966, 0.9472) &          0.9265 &               (0.9031, 0.9499) \\
        DC &  CoMP &     15.00 &           0.8867 &                (0.8549, 0.9184) &          0.8955 &               (0.8686, 0.9224) \\
        DC &  CoMP &     20.00 &           0.8547 &                (0.8288, 0.8806) &          0.8676 &               (0.8428, 0.8924) \\
        NK &  CoMP &      0.25 &           0.9962 &                (0.9959, 0.9964) &          0.9955 &               (0.9951, 0.9959) \\
        NK &  CoMP &      0.50 &           0.9949 &                (0.9942, 0.9957) &          0.9938 &               (0.9927, 0.9950) \\
        NK &  CoMP &      1.00 &           0.9917 &                (0.9904, 0.9929) &          0.9899 &               (0.9881, 0.9917) \\
        NK &  CoMP &      5.00 &           0.9567 &                (0.9491, 0.9643) &          0.9399 &               (0.9296, 0.9501) \\
        NK &  CoMP &     10.00 &           0.8916 &                (0.8654, 0.9178) &          0.8670 &               (0.8361, 0.8979) \\
        NK &  CoMP &     15.00 &           0.8780 &                (0.8501, 0.9060) &          0.8518 &               (0.8187, 0.8850) \\
        NK &  CoMP &     20.00 &           0.8767 &                (0.8517, 0.9018) &          0.8477 &               (0.8189, 0.8765) \\
         T &  CoMP &      0.25 &           0.9951 &                (0.9947, 0.9954) &          0.9945 &               (0.9940, 0.9950) \\
         T &  CoMP &      0.50 &           0.9950 &                (0.9947, 0.9952) &          0.9945 &               (0.9940, 0.9949) \\
         T &  CoMP &      1.00 &           0.9941 &                (0.9936, 0.9946) &          0.9934 &               (0.9928, 0.9940) \\
         T &  CoMP &      5.00 &           0.9682 &                (0.9584, 0.9779) &          0.9690 &               (0.9627, 0.9753) \\
         T &  CoMP &     10.00 &           0.9462 &                (0.9336, 0.9588) &          0.9595 &               (0.9519, 0.9672) \\
         T &  CoMP &     15.00 &           0.9402 &                (0.9277, 0.9527) &          0.9569 &               (0.9495, 0.9642) \\
         T &  CoMP &     20.00 &           0.9239 &                (0.9136, 0.9342) &          0.9254 &               (0.9092, 0.9416) \\
\bottomrule
\end{tabular}
}
\label{tab:kang_cf_pearson_gamma}
\end{table}

\begin{table}
\centering
\caption{Effect of $\gamma$ on counterfactual data reconstruction: Mean of the mean squared error (MSE) for all and DE genes across 10 random seeds.}
\resizebox{0.7\textwidth}{!}{
\begin{tabular}{lllcccc}
\toprule
 Cell type & Model &  $\gamma$ & $\text{MSE}_{\text{all}}$ & $\text{MSE}_{\text{all}} \pm \text{SEM}$ & $\text{MSE}_{\text{DE}}$ & $\text{MSE}_{\text{DE}} \pm \text{SEM}$ \\
\midrule
         B &  CoMP &      0.25 &                    0.0001 &                         (0.0001, 0.0001) &                   0.0023 &                        (0.0021, 0.0025) \\
         B &  CoMP &      0.50 &                    0.0001 &                         (0.0001, 0.0001) &                   0.0023 &                        (0.0019, 0.0026) \\
         B &  CoMP &      1.00 &                    0.0001 &                         (0.0001, 0.0001) &                   0.0024 &                        (0.0020, 0.0028) \\
         B &  CoMP &      5.00 &                    0.0032 &                         (0.0022, 0.0043) &                   0.0672 &                        (0.0479, 0.0864) \\
         B &  CoMP &     10.00 &                    0.0056 &                         (0.0041, 0.0071) &                   0.1046 &                        (0.0768, 0.1325) \\
         B &  CoMP &     15.00 &                    0.0065 &                         (0.0053, 0.0078) &                   0.1201 &                        (0.0960, 0.1442) \\
         B &  CoMP &     20.00 &                    0.0077 &                         (0.0067, 0.0086) &                   0.1666 &                        (0.1442, 0.1891) \\
 CD14 Mono &  CoMP &      0.25 &                    0.0023 &                         (0.0014, 0.0032) &                   0.0576 &                        (0.0309, 0.0843) \\
 CD14 Mono &  CoMP &      0.50 &                    0.0012 &                         (0.0005, 0.0020) &                   0.0276 &                        (0.0068, 0.0483) \\
 CD14 Mono &  CoMP &      1.00 &                    0.0011 &                         (0.0002, 0.0019) &                   0.0245 &                        (0.0023, 0.0468) \\
 CD14 Mono &  CoMP &      5.00 &                    0.0034 &                         (0.0013, 0.0056) &                   0.0935 &                        (0.0330, 0.1541) \\
 CD14 Mono &  CoMP &     10.00 &                    0.0123 &                         (0.0066, 0.0180) &                   0.3157 &                        (0.1724, 0.4591) \\
 CD14 Mono &  CoMP &     15.00 &                    0.0196 &                         (0.0124, 0.0268) &                   0.4825 &                        (0.3056, 0.6594) \\
 CD14 Mono &  CoMP &     20.00 &                    0.0245 &                         (0.0173, 0.0316) &                   0.6022 &                        (0.4302, 0.7742) \\
 CD16 Mono &  CoMP &      0.25 &                    0.0003 &                         (0.0002, 0.0005) &                   0.0054 &                        (0.0034, 0.0074) \\
 CD16 Mono &  CoMP &      0.50 &                    0.0002 &                         (0.0001, 0.0003) &                   0.0036 &                        (0.0024, 0.0048) \\
 CD16 Mono &  CoMP &      1.00 &                    0.0002 &                         (0.0001, 0.0003) &                   0.0031 &                        (0.0017, 0.0044) \\
 CD16 Mono &  CoMP &      5.00 &                    0.0040 &                         (0.0024, 0.0056) &                   0.0876 &                        (0.0451, 0.1300) \\
 CD16 Mono &  CoMP &     10.00 &                    0.0120 &                         (0.0070, 0.0170) &                   0.2751 &                        (0.1522, 0.3979) \\
 CD16 Mono &  CoMP &     15.00 &                    0.0194 &                         (0.0134, 0.0254) &                   0.4529 &                        (0.3058, 0.6001) \\
 CD16 Mono &  CoMP &     20.00 &                    0.0284 &                         (0.0233, 0.0335) &                   0.6688 &                        (0.5421, 0.7955) \\
     CD4 T &  CoMP &      0.25 &                    0.0001 &                         (0.0001, 0.0001) &                   0.0015 &                        (0.0014, 0.0017) \\
     CD4 T &  CoMP &      0.50 &                    0.0001 &                         (0.0001, 0.0001) &                   0.0013 &                        (0.0013, 0.0014) \\
     CD4 T &  CoMP &      1.00 &                    0.0001 &                         (0.0001, 0.0001) &                   0.0015 &                        (0.0014, 0.0016) \\
     CD4 T &  CoMP &      5.00 &                    0.0005 &                         (0.0003, 0.0006) &                   0.0133 &                        (0.0081, 0.0185) \\
     CD4 T &  CoMP &     10.00 &                    0.0003 &                         (0.0002, 0.0004) &                   0.0040 &                        (0.0019, 0.0062) \\
     CD4 T &  CoMP &     15.00 &                    0.0003 &                         (0.0003, 0.0004) &                   0.0027 &                        (0.0022, 0.0032) \\
     CD4 T &  CoMP &     20.00 &                    0.0013 &                         (0.0008, 0.0019) &                   0.0381 &                        (0.0155, 0.0607) \\
     CD8 T &  CoMP &      0.25 &                    0.0003 &                         (0.0003, 0.0003) &                   0.0049 &                        (0.0046, 0.0051) \\
     CD8 T &  CoMP &      0.50 &                    0.0003 &                         (0.0003, 0.0003) &                   0.0052 &                        (0.0048, 0.0056) \\
     CD8 T &  CoMP &      1.00 &                    0.0005 &                         (0.0004, 0.0006) &                   0.0074 &                        (0.0066, 0.0082) \\
     CD8 T &  CoMP &      5.00 &                    0.0023 &                         (0.0020, 0.0025) &                   0.0329 &                        (0.0290, 0.0369) \\
     CD8 T &  CoMP &     10.00 &                    0.0030 &                         (0.0026, 0.0033) &                   0.0366 &                        (0.0322, 0.0411) \\
     CD8 T &  CoMP &     15.00 &                    0.0032 &                         (0.0028, 0.0035) &                   0.0387 &                        (0.0344, 0.0431) \\
     CD8 T &  CoMP &     20.00 &                    0.0038 &                         (0.0034, 0.0042) &                   0.0717 &                        (0.0512, 0.0921) \\
        DC &  CoMP &      0.25 &                    0.0009 &                         (0.0008, 0.0009) &                   0.0135 &                        (0.0122, 0.0149) \\
        DC &  CoMP &      0.50 &                    0.0007 &                         (0.0006, 0.0008) &                   0.0102 &                        (0.0082, 0.0123) \\
        DC &  CoMP &      1.00 &                    0.0011 &                         (0.0007, 0.0016) &                   0.0161 &                        (0.0083, 0.0239) \\
        DC &  CoMP &      5.00 &                    0.0075 &                         (0.0044, 0.0106) &                   0.1189 &                        (0.0580, 0.1798) \\
        DC &  CoMP &     10.00 &                    0.0165 &                         (0.0113, 0.0217) &                   0.2549 &                        (0.1607, 0.3490) \\
        DC &  CoMP &     15.00 &                    0.0228 &                         (0.0169, 0.0288) &                   0.3663 &                        (0.2603, 0.4723) \\
        DC &  CoMP &     20.00 &                    0.0299 &                         (0.0250, 0.0347) &                   0.4809 &                        (0.3844, 0.5773) \\
        NK &  CoMP &      0.25 &                    0.0004 &                         (0.0004, 0.0004) &                   0.0089 &                        (0.0080, 0.0097) \\
        NK &  CoMP &      0.50 &                    0.0005 &                         (0.0004, 0.0006) &                   0.0122 &                        (0.0099, 0.0146) \\
        NK &  CoMP &      1.00 &                    0.0008 &                         (0.0007, 0.0010) &                   0.0204 &                        (0.0169, 0.0238) \\
        NK &  CoMP &      5.00 &                    0.0041 &                         (0.0035, 0.0048) &                   0.1169 &                        (0.0988, 0.1350) \\
        NK &  CoMP &     10.00 &                    0.0090 &                         (0.0069, 0.0110) &                   0.2419 &                        (0.1869, 0.2969) \\
        NK &  CoMP &     15.00 &                    0.0100 &                         (0.0078, 0.0122) &                   0.2687 &                        (0.2101, 0.3273) \\
        NK &  CoMP &     20.00 &                    0.0103 &                         (0.0084, 0.0122) &                   0.2874 &                        (0.2376, 0.3371) \\
         T &  CoMP &      0.25 &                    0.0004 &                         (0.0003, 0.0004) &                   0.0053 &                        (0.0048, 0.0058) \\
         T &  CoMP &      0.50 &                    0.0004 &                         (0.0003, 0.0004) &                   0.0053 &                        (0.0049, 0.0058) \\
         T &  CoMP &      1.00 &                    0.0004 &                         (0.0004, 0.0005) &                   0.0066 &                        (0.0060, 0.0073) \\
         T &  CoMP &      5.00 &                    0.0022 &                         (0.0016, 0.0029) &                   0.0409 &                        (0.0311, 0.0507) \\
         T &  CoMP &     10.00 &                    0.0037 &                         (0.0029, 0.0046) &                   0.0579 &                        (0.0454, 0.0703) \\
         T &  CoMP &     15.00 &                    0.0041 &                         (0.0033, 0.0050) &                   0.0624 &                        (0.0501, 0.0747) \\
         T &  CoMP &     20.00 &                    0.0053 &                         (0.0046, 0.0060) &                   0.1005 &                        (0.0829, 0.1181) \\
\bottomrule
\end{tabular}
}
\label{tab:kang_cf_mse_gamma}
\end{table}

\begin{figure}[t]
  \centering
  \includegraphics[width=1\columnwidth]{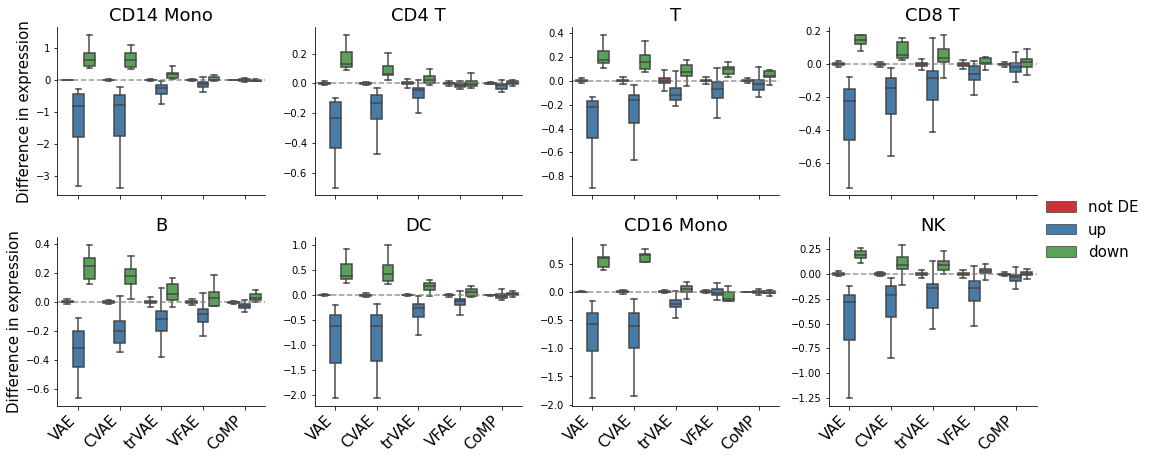}
  \caption{The difference in gene expression values for 1950 non-differentially expressed genes (red) and the top 50 differentially expressed genes (up-regulated: blue, down-regulated: green) between IFN-$\beta$ stimulated cells and counterfactually stimulated control cells for each cell type. The difference in expression for a gene is the gene mean expression across stimulated cells of a cell type minus the mean reconstructed gene expression for counterfactually stimulated control cells of the same cell type.}
  \label{fig:kang_cf_all}
\end{figure}

\begin{figure}[t]
  \centering
  \includegraphics[width=0.85\columnwidth]{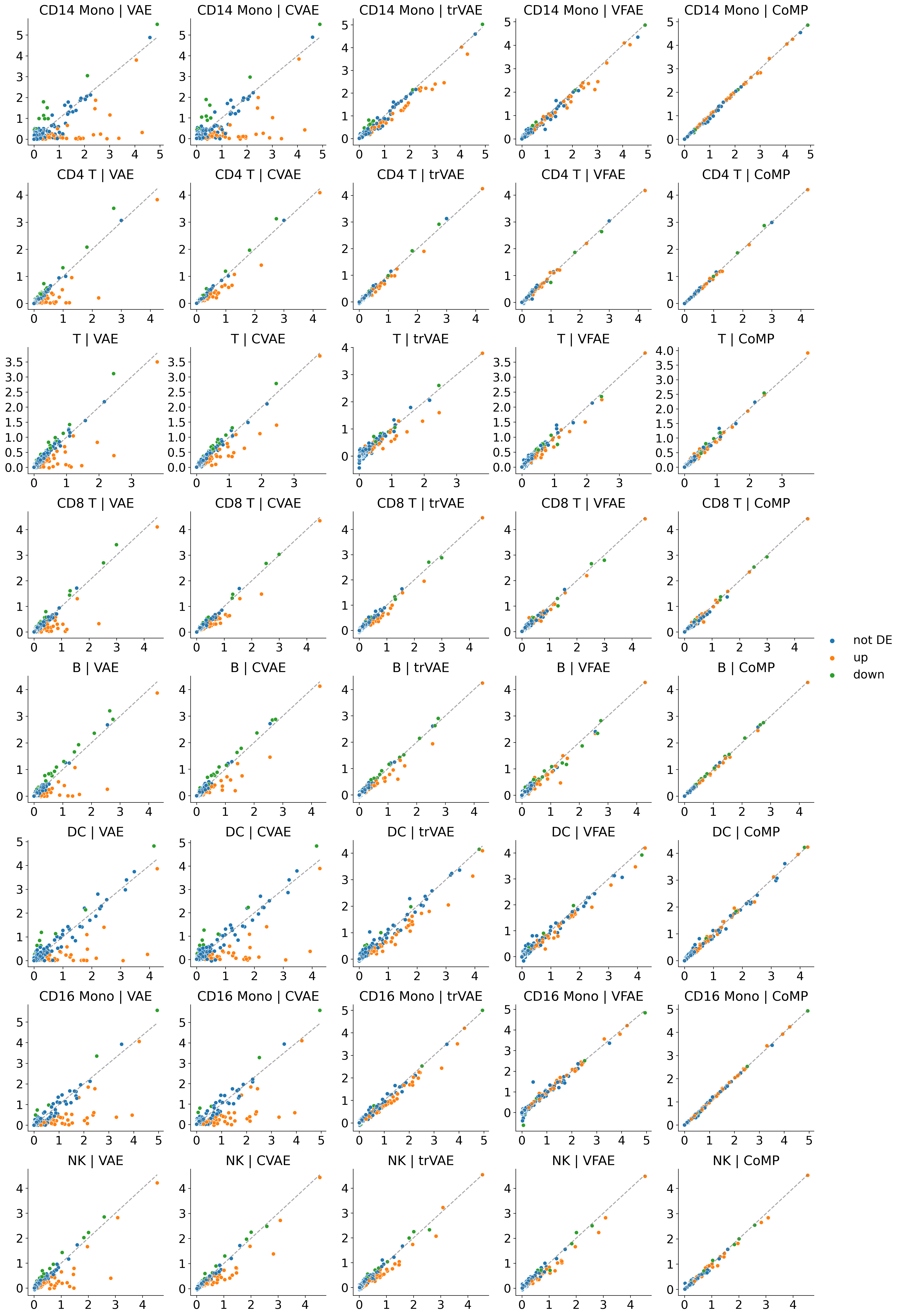}
  \caption{Mean gene expression of actual stimulated cells against the mean gene expression of counterfactually stimulated control cells for each cell type and model.}
  \label{fig:kang_cf_scatter}
\end{figure}

\clearpage
\subsection{Batch mixing metrics for scRNA-seq data integration}
\label{sec:batch_effect_details}
In this section we present additional details of our experiments for the scRNA-seq data integration experiment, evaluating the mixing of cells from two batches (different library preparation protocols) in the latent space for CoMP over 19 random seeds, compared with Seurat (CCA) \cite{stuart2019comprehensive} and Harmony \cite{Harmony}.  The following cell sub-types were considered in this analysis: naive B cells, memory B cells, CD14 monocytes, CD16 monocytes, naive CD4+ T cells, effector CD8+ T cells, naive CD8+ T cells, memory CD8+ T cells, regulatory T cells, activated dendritic cells, plasmacytoid dendritic cells, and natural killer cells. Megakaryocytes and Hematopoietic stem cells were excluded from the analysis due to the very low numbers of cells within the dataset, which made computing the metrics for these cell sub-types infeasible. Note that cell sub-types were predefined and provided in the associated metadata from \citet{Harmony}.
In Table~\ref{tab:batch_correction}, we present the results of CoMP across a number of training seeds, showing that it has reliable performance against baselines with minimal standard error.

\begin{table}[t]
\small
\centering
\caption{scRNA-seq data integration experiment results, with $k=100$, $c=\text{Protocol}$, and $\alpha=0.05$. $s_{k,c}$ and $\tilde{s}_{k,c}$ are the two Silhouette Coefficient variants (see Section~\ref{sec:experiments}). The top scores are in \textbf{bold}. For CoMP, results represent the mean across 19 random initialisations $\pm$ standard error of the mean. $\tilde{s}$ and $\mkbet$ represent mean metrics computed over considered cell sub-types (see \ref{sec:batch_effect_details} for details of considered cell sub-types).}
\label{tab:batch_correction}
\begin{tabular}{cccccc}
& $s$ & ${\kbet}$ & $\tilde{s}$ & $\mkbet$
\\ \toprule
Seurat CCA & 0.0176 & 0.436 & 0.022 & 0.356 \\ 
Harmony & 0.0158 & 0.318 & 0.013 & 0.245 \\ 
\textit{CoMP} & $\bm{0.0007 \pm 0.002}$ & $\bm{0.171 \pm 0.003}$ & $\bm{0.0029 \pm 0.0004}$ & $\bm{0.132 \pm 0.002}$ \\
\bottomrule
\end{tabular}
\end{table}

\subsection{Implementation details and hyperparameters}
The encoder and decoders are parameterised by multi-layer fully-connected networks. Following the trVAE implementation \cite{lotfollahi2019conditional}, we implement a multi-scale Gaussian kernel for both trVAE and VFAE benchmark models, except on the Adult Income dataset where a single scale kernel was used to match the original implementation. The details of the model architectures and hyperparameters used in CoMP, VFAE and trVAE across three sets of experiments are given in Tables \ref{hparam1}--\ref{hparam2}. The networks in all experiments were trained with a 90/10 train/validation split, metrics were calculated on the entire dataset.

We use the means of the posteriors as the encoded representation $\z$ for each sample. For the  Adult Income experiments this differs from 
\citet{louizos2015variational}, where the representations are sampled from the posterior before classification. We found the noise from sampling would mask the inclusion of predictive information about gender in the encoded means from the VFAE, as can be seen in the difference in accuracy between [VFAE-s] and [VFAE-m] in Table~\ref{tab:income-results}.

\subsection{Model training resources}
Experiments were performed on NVIDIA Tesla V100 GPUs. Each training run of CoMP for a single hyperparameter configuration on the \tcgaccle\  dataset (our largest dataset) on a single GPU takes 2--3 hours. Running times for the other models are broadly similar.

\subsection{CO$_2$ emissions}
Experiments were conducted using private infrastructure, which has an estimated carbon efficiency of 0.188 kgCO$_2$eq/kWh.
An estimated cumulative 1900 hours of computation were performed on hardware of type Tesla V100.
Total emissions are estimated to be 107 kgCO$_2$eq. 
Estimations were conducted using the \href{https://mlco2.github.io/impact#compute}{Machine Learning Impact calculator} presented in \citet{lacoste2019quantifying}.

\begin{table}[b]
\centering
\caption{CoMP architecture and hyperparameters for the \tcgaccle\  dataset.}
\label{hparam1}
\begin{tabular}{lrll}
\toprule
               Layer & Output Dim &                       Inputs &                   Notes \\
\midrule
      \textbf{Input} &       9468 &                              &                         \\
 \textbf{Conditions} &          2 &                              &                         \\
    \textbf{Encoder} &            &                              &                         \\
               FC\_1 &        512 &          [Input, Conditions] &  BatchNorm1D, LeakyReLU \\
               FC\_2 &        512 &                        FC\_1 &  BatchNorm1D, LeakyReLU \\
               FC\_3 &        512 &                        FC\_2 &  BatchNorm1D, LeakyReLU \\
             Z\_mean &         16 &                        FC\_3 &                         \\
                   Z &         16 &               [Z\_mean, 0.1] &                Normal() \\
    \textbf{Decoder} &            &                              &                         \\
               FC\_1 &        512 &                            Z &  BatchNorm1D, LeakyReLU \\
                FC-2 &        512 &                        FC\_1 &  BatchNorm1D, LeakyReLU \\
                FC-3 &        512 &                         FC-2 &  BatchNorm1D, LeakyReLU \\
         \^{X}\_mean &       9468 &                         FC-3 &                         \\
        \^{X}\_scale &          1 &                         FC-3 &                         \\
               \^{X} &       9468 &  [\^{X}\_mean, \^{X}\_scale] &                Normal() \\
    \textbf{Penalty} &            &                              &                         \\
        CoMP penalty &            &              [Z, Conditions] &                         \\
\midrule
           Optimiser &       Adam &                              &                         \\
       Learning rate &      1e-4 &                              &                         \\
          Batch size &       5500 &                              &                         \\
              Epochs &       4000 &                              &                         \\
             $\beta$ &      1e-7 &                              &                         \\
            $\gamma$ &        0.5 &                              &                         \\
     LeakyReLU slope &       0.01 &                              &                         \\
\bottomrule
\end{tabular}
\end{table}

\begin{table}
\centering
\caption{VFAE architecture and hyperparameters for the \tcgaccle\  dataset.}
\begin{tabular}{lrll}
\toprule
               Layer & Output Dim &                       Inputs &                   Notes \\
\midrule
      \textbf{Input} &       9468 &                              &                         \\
 \textbf{Conditions} &          2 &                              &                         \\
    \textbf{Encoder} &            &                              &                         \\
               FC\_1 &        512 &          [Input, Conditions] &  BatchNorm1D, LeakyReLU \\
               FC\_2 &        512 &                        FC\_1 &  BatchNorm1D, LeakyReLU \\
               FC\_3 &        512 &                        FC\_2 &  BatchNorm1D, LeakyReLU \\
             Z\_mean &         16 &                        FC\_3 &                         \\
            Z\_scale &         16 &                        FC\_3 &                         \\
                   Z &         16 &          [Z\_mean, Z\_scale] &                Normal() \\
    \textbf{Decoder} &            &                              &                         \\
               FC\_1 &        512 &                            Z &  BatchNorm1D, LeakyReLU \\
                FC-2 &        512 &                        FC\_1 &  BatchNorm1D, LeakyReLU \\
                FC-3 &        512 &                         FC-2 &  BatchNorm1D, LeakyReLU \\
         \^{X}\_mean &       9468 &                         FC-3 &                         \\
        \^{X}\_scale &          1 &                         FC-3 &                         \\
               \^{X} &       9468 &  [\^{X}\_mean, \^{X}\_scale] &                Normal() \\
    \textbf{Penalty} &            &                              &                         \\
                 MMD &            &            [FC1, Conditions] &  Multi-scale RBF kernel \\
\midrule
           Optimiser &       Adam &                              &                         \\
       Learning rate &      1e-03 &                              &                         \\
          Batch size &       5550 &                              &                         \\
              Epochs &       4000 &                              &                         \\
             $\beta$ &      1e-7 &                              &                         \\
            $\gamma$ &          4 &                              &                         \\
     LeakyReLU slope &       0.01 &                              &                         \\
\bottomrule
\end{tabular}
\end{table}

\begin{table}
\centering
\caption{trVAE architecture and hyperparameters for the \tcgaccle\  dataset.}
\begin{tabular}{lrll}
\toprule
               Layer & Output Dim &                       Inputs &                   Notes \\
\midrule
      \textbf{Input} &       9468 &                              &                         \\
 \textbf{Conditions} &          2 &                              &                         \\
    \textbf{Encoder} &            &                              &                         \\
               FC\_1 &        512 &          [Input, Conditions] &  BatchNorm1D, LeakyReLU \\
               FC\_2 &        512 &                        FC\_1 &  BatchNorm1D, LeakyReLU \\
               FC\_3 &        512 &                        FC\_2 &  BatchNorm1D, LeakyReLU \\
             Z\_mean &         16 &                        FC\_3 &                         \\
            Z\_scale &         16 &                        FC\_3 &                         \\
                   Z &         16 &          [Z\_mean, Z\_scale] &                Normal() \\
    \textbf{Decoder} &            &                              &                         \\
               FC\_1 &        512 &                            Z &  BatchNorm1D, LeakyReLU \\
                FC-2 &        512 &                        FC\_1 &  BatchNorm1D, LeakyReLU \\
                FC-3 &        512 &                         FC-2 &  BatchNorm1D, LeakyReLU \\
         \^{X}\_mean &       9468 &                         FC-3 &                         \\
        \^{X}\_scale &          1 &                         FC-3 &                         \\
               \^{X} &       9468 &  [\^{X}\_mean, \^{X}\_scale] &                Normal() \\
    \textbf{Penalty} &            &                              &                         \\
                 MMD &            &            [FC1, Conditions] &  Multi-scale RBF kernel \\
\midrule
           Optimiser &       Adam &                              &                         \\
       Learning rate &      3e-4 &                              &                         \\
          Batch size &       5550 &                              &                         \\
              Epochs &       4000 &                              &                         \\
             $\beta$ &      1e-7 &                              &                         \\
            $\gamma$ &         10 &                              &                         \\
     LeakyReLU slope &       0.01 &                              &                         \\
\bottomrule
\end{tabular}
\end{table}

\begin{table}
\centering
\caption{CoMP architecture and hyperparameters for the stimulated / untreated single-cell PBMCs dataset.}
\begin{tabular}{lrll}
\toprule
               Layer & Output Dim &                       Inputs &                   Notes \\
\midrule
      \textbf{Input} &       2000 &                              &                         \\
 \textbf{Conditions} &          2 &                              &                         \\
    \textbf{Encoder} &            &                              &                         \\
               FC\_1 &        512 &          [Input, Conditions] &  BatchNorm1D, LeakyReLU \\
               FC\_2 &        512 &                        FC\_1 &  BatchNorm1D, LeakyReLU \\
               FC\_3 &        512 &                        FC\_2 &  BatchNorm1D, LeakyReLU \\
             Z\_mean &         40 &                        FC\_3 &                         \\
                   Z &         40 &               [Z\_mean, 0.1] &                Normal() \\
    \textbf{Decoder} &            &                              &                         \\
               FC\_1 &        512 &                            Z &  BatchNorm1D, LeakyReLU \\
                FC-2 &        512 &                        FC\_1 &  BatchNorm1D, LeakyReLU \\
                FC-3 &        512 &                         FC-2 &  BatchNorm1D, LeakyReLU \\
         \^{X}\_mean &       2000 &                         FC-3 &                         \\
        \^{X}\_scale &          1 &                         FC-3 &                         \\
               \^{X} &       2000 &  [\^{X}\_mean, \^{X}\_scale] &                Normal() \\
    \textbf{Penalty} &            &                              &                         \\
        CoMP penalty &            &              [Z, Conditions] &                         \\
\midrule
          Optimiser &       Adam &                              &                         \\
       Learning rate &      1e-06 &                              &                         \\
          Batch size &        512 &                              &                         \\
              Epochs &      10000 &                              &                         \\
             $\beta$ &      1e-7 &                              &                         \\
            $\gamma$ &          1 &                              &                         \\
     LeakyReLU slope &       0.01 &                              &                         \\
\bottomrule
\end{tabular}
\end{table}

\begin{table}
\centering
\caption{VFAE architecture and hyperparameters for the stimulated / untreated single-cell PBMCs dataset.}
\begin{tabular}{lrll}
\toprule
               Layer & Output Dim &                       Inputs &                   Notes \\
\midrule
      \textbf{Input} &       2000 &                              &                         \\
 \textbf{Conditions} &          2 &                              &                         \\
    \textbf{Encoder} &            &                              &                         \\
               FC\_1 &        512 &          [Input, Conditions] &  BatchNorm1D, LeakyReLU \\
               FC\_2 &        512 &                        FC\_1 &  BatchNorm1D, LeakyReLU \\
               FC\_3 &        512 &                        FC\_2 &  BatchNorm1D, LeakyReLU \\
             Z\_mean &         40 &                        FC\_3 &                         \\
                   Z &         40 &               [Z\_mean, 0.1] &                Normal() \\
    \textbf{Decoder} &            &                              &                         \\
               FC\_1 &        512 &                            Z &  BatchNorm1D, LeakyReLU \\
                FC-2 &        512 &                        FC\_1 &  BatchNorm1D, LeakyReLU \\
                FC-3 &        512 &                         FC-2 &  BatchNorm1D, LeakyReLU \\
         \^{X}\_mean &       2000 &                         FC-3 &                         \\
        \^{X}\_scale &          1 &                         FC-3 &                         \\
               \^{X} &       2000 &  [\^{X}\_mean, \^{X}\_scale] &                Normal() \\
    \textbf{Penalty} &            &                              &                         \\
                 MMD &            &            [FC1, Conditions] &  Multi-scale RBF kernel \\
\midrule
           Optimiser &       Adam &                              &                         \\
       Learning rate &      1e-4 &                              &                         \\
          Batch size &        512 &                              &                         \\
              Epochs &      10000 &                              &                         \\
             $\beta$ &      1e-7 &                              &                         \\
            $\gamma$ &          1 &                              &                         \\
     LeakyReLU slope &       0.01 &                              &                         \\
\bottomrule
\end{tabular}
\end{table}

\begin{table}
\centering
\caption{trVAE architecture and hyperparameters for the single-cell PBMC dataset.}
\begin{tabular}{lrll}
\toprule
               Layer & Output Dim &                       Inputs &                   Notes \\
\midrule
      \textbf{Input} &       2000 &                              &                         \\
 \textbf{Conditions} &          2 &                              &                         \\
    \textbf{Encoder} &            &                              &                         \\
               FC\_1 &        512 &          [Input, Conditions] &  BatchNorm1D, LeakyReLU \\
               FC\_2 &        512 &                        FC\_1 &  BatchNorm1D, LeakyReLU \\
               FC\_3 &        512 &                        FC\_2 &  BatchNorm1D, LeakyReLU \\
             Z\_mean &         40 &                        FC\_3 &                         \\
                   Z &         40 &               [Z\_mean, 0.1] &                Normal() \\
    \textbf{Decoder} &            &                              &                         \\
               FC\_1 &        512 &                            Z &  BatchNorm1D, LeakyReLU \\
                FC-2 &        512 &                        FC\_1 &  BatchNorm1D, LeakyReLU \\
                FC-3 &        512 &                         FC-2 &  BatchNorm1D, LeakyReLU \\
         \^{X}\_mean &       2000 &                         FC-3 &                         \\
        \^{X}\_scale &          1 &                         FC-3 &                         \\
               \^{X} &       2000 &  [\^{X}\_mean, \^{X}\_scale] &                Normal() \\
    \textbf{Penalty} &            &                              &                         \\
                 MMD &            &            [FC1, Conditions] &  Multi-scale RBF kernel \\
\midrule
           Optimiser &       Adam &                              &                         \\
       Learning rate &      5e-4 &                              &                         \\
          Batch size &        512 &                              &                         \\
              Epochs &       6000 &                              &                         \\
             $\beta$ &      1e-7 &                              &                         \\
            $\gamma$ &         10 &                              &                         \\
     LeakyReLU slope &       0.01 &                              &                         \\
\bottomrule
\end{tabular}
\end{table}

\begin{table}
\centering
\caption{CoMP architecture and hyperparameters for the scRNA-seq data integration dataset.}
\begin{tabular}{lrll}
\toprule
               Layer & Output Dim &                       Inputs &                   Notes \\
\midrule
      \textbf{Input} &       2000 &                              &                         \\
 \textbf{Conditions} &          2 &                              &                         \\
    \textbf{Encoder} &            &                              &                         \\
               FC\_1 &        512 &          [Input, Conditions] &  BatchNorm1D, LeakyReLU \\
               FC\_2 &        512 &                        FC\_1 &  BatchNorm1D, LeakyReLU \\
               FC\_3 &        512 &                        FC\_2 &  BatchNorm1D, LeakyReLU \\
             Z\_mean &         40 &                        FC\_3 &                         \\
                   Z &         40 &               [Z\_mean, 0.1] &                Normal() \\
    \textbf{Decoder} &            &                              &                         \\
               FC\_1 &        512 &                            Z &  BatchNorm1D, LeakyReLU \\
                FC-2 &        512 &                        FC\_1 &  BatchNorm1D, LeakyReLU \\
                FC-3 &        512 &                         FC-2 &  BatchNorm1D, LeakyReLU \\
         \^{X}\_mean &       2000 &                         FC-3 &                         \\
        \^{X}\_scale &          1 &                         FC-3 &                         \\
               \^{X} &       2000 &  [\^{X}\_mean, \^{X}\_scale] &                Normal() \\
    \textbf{Penalty} &            &                              &                         \\
        CoMP penalty &            &              [Z, Conditions] &                         \\
\midrule
           Optimiser &       Adam &                              &                         \\
       Learning rate &      5e-06 &                              &                         \\
          Batch size &        512 &                              &                         \\
              Epochs &        350 &                              &                         \\
             $\beta$ &          1 &                              &                         \\
            $\gamma$ &       0.75 &                              &                         \\
     LeakyReLU slope &       0.01 &                              &                         \\
\bottomrule
\end{tabular}
\end{table}
\begin{table}
\centering
\caption{CoMP architecture and hyperparameters for the UCI Adult Income dataset.}
\begin{tabular}{lrll}
\toprule
               Layer & Output Dim &                       Inputs &                   Notes \\
\midrule
      \textbf{Input} &         82 &                              &                         \\
 \textbf{Conditions} &          2 &                              &                         \\
    \textbf{Encoder} &            &                              &                         \\
               FC\_1 &         64 &          [Input, Conditions] &  BatchNorm1D, LeakyReLU \\
               FC\_2 &         64 &                        FC\_1 &  BatchNorm1D, LeakyReLU \\
             Z\_mean &         16 &                        FC\_2 &                         \\
                   Z &         16 &               [Z\_mean, 0.1] &                Normal() \\
    \textbf{Decoder} &            &                              &                         \\
               FC\_1 &         64 &                            Z &  BatchNorm1D, LeakyReLU \\
                FC-2 &         64 &                        FC\_1 &  BatchNorm1D, LeakyReLU \\
         \^{X}\_mean &         82 &                         FC-2 &                         \\
        \^{X}\_scale &          1 &                         FC-2 &                         \\
               \^{X} &         82 &  [\^{X}\_mean, \^{X}\_scale] &                Normal() \\
    \textbf{Penalty} &            &                              &                         \\
        CoMP penalty &            &              [Z, Conditions] &                         \\
\midrule
           Optimiser &       Adam &                              &                         \\
       Learning rate &      1e-04 &                              &                         \\
          Batch size &       4096 &                              &                         \\
              Epochs &      10000 &                              &                         \\
             $\beta$ &          1 &                              &                         \\
            $\gamma$ &        0.5 &                              &                         \\
     LeakyReLU slope &       0.01 &                              &                         \\
\bottomrule
\end{tabular}
\end{table}

\begin{table}
\centering
\caption{VFAE architecture and hyperparameters for the UCI Adult Income dataset.}
\begin{tabular}{lrll}
\toprule
               Layer & Output Dim &                       Inputs &                   Notes \\
\midrule
      \textbf{Input} &         82 &                              &                         \\
 \textbf{Conditions} &          2 &                              &                         \\
    \textbf{Encoder} &            &                              &                         \\
               FC\_1 &         64 &          [Input, Conditions] &  BatchNorm1D, LeakyReLU \\
             Z\_mean &         16 &                        FC\_1 &                         \\
            Z\_scale &         16 &                        FC\_1 &                         \\
                   Z &         16 &          [Z\_mean, Z\_scale] &                Normal() \\
    \textbf{Decoder} &            &                              &                         \\
               FC\_1 &         64 &                            Z &  BatchNorm1D, LeakyReLU \\
         \^{X}\_mean &         82 &                        FC\_1 &                         \\
        \^{X}\_scale &          1 &                        FC\_1 &                         \\
               \^{X} &         82 &  [\^{X}\_mean, \^{X}\_scale] &                Normal() \\
    \textbf{Penalty} &            &                              &                         \\
                 MMD &            &              [Z, Conditions] &                         \\
\midrule
           Optimiser &       Adam &                              &                         \\
       Learning rate &      1e-04 &                              &                         \\
          Batch size &        512 &                              &                         \\
              Epochs &      10000 &                              &                         \\
             $\beta$ &          1 &                              &                         \\
            $\gamma$ &       1000 &                              &                         \\
           RBF scale &          2 &                              &                         \\
     LeakyReLU slope &       0.01 &                              &                         \\
\bottomrule
\end{tabular}
\end{table}

\begin{table}
\centering
\caption{trVAE architecture and hyperparameters for the UCI Adult Income dataset.}
\label{hparam2}
\begin{tabular}{lrll}
\toprule
               Layer & Output Dim &                       Inputs &                   Notes \\
\midrule
      \textbf{Input} &         82 &                              &                         \\
 \textbf{Conditions} &          2 &                              &                         \\
    \textbf{Encoder} &            &                              &                         \\
               FC\_1 &         32 &          [Input, Conditions] &  BatchNorm1D, LeakyReLU \\
               FC\_2 &         32 &                        FC\_1 &  BatchNorm1D, LeakyReLU \\
             Z\_mean &          8 &                        FC\_2 &                         \\
                   Z &          8 &               [Z\_mean, 0.1] &                Normal() \\
    \textbf{Decoder} &            &                              &                         \\
               FC\_1 &         32 &                            Z &  BatchNorm1D, LeakyReLU \\
                FC-2 &         32 &                        FC\_1 &  BatchNorm1D, LeakyReLU \\
         \^{X}\_mean &         82 &                         FC-2 &                         \\
        \^{X}\_scale &          1 &                         FC-2 &                         \\
               \^{X} &         82 &  [\^{X}\_mean, \^{X}\_scale] &                Normal() \\
    \textbf{Penalty} &            &                              &                         \\
                 MMD &            &            [FC1, Conditions] &  Multi-scale RBF kernel \\
\midrule
           Optimiser &       Adam &                              &                         \\
       Learning rate &      1e-04 &                              &                         \\
          Batch size &       4096 &                              &                         \\
              Epochs &      10000 &                              &                         \\
             $\beta$ &      0.001 &                              &                         \\
            $\gamma$ &         10 &                              &                         \\
     LeakyReLU slope &       0.01 &                              &                         \\
\bottomrule
\end{tabular}
\end{table}

\end{document}